\documentclass[final]{article}

\usepackage[nonatbib]{neurips_2024_ar}
\usepackage[utf8]{inputenc} % allow utf-8 input
\usepackage[T1]{fontenc}    % use 8-bit T1 fonts
\usepackage{hyperref}       % hyperlinks
\usepackage{url}            % simple URL typesetting
\usepackage{booktabs}       % professional-quality tables
\usepackage{amsfonts}       % blackboard math symbols
\usepackage[nice]{nicefrac}       % compact symbols for 1/2, etc.
\usepackage{microtype}      % microtypography
\usepackage{xcolor}         % colors

\usepackage[backend=biber,style=numeric,sorting=nty]{biblatex}
\bibliography{biblio.bib}
\usepackage{graphicx}
\usepackage{subfigure}
\usepackage{xspace}
\usepackage{physics}
\usepackage{amsmath}
\usepackage{amssymb}
\usepackage{mathtools}
\usepackage{algorithm}
\usepackage{algorithmic}
\usepackage{amsthm}
\usepackage{placeins}
\usepackage{todonotes}
\usepackage{multirow}
\usepackage{wrapfig}

\usepackage[capitalize,noabbrev]{cleveref}
\usepackage{xstring}
\theoremstyle{plain}
\newtheorem{theorem}{Theorem}[section]
\newtheorem{proposition}[theorem]{Proposition}

\newtheorem{corollary}[theorem]{Corollary}
\theoremstyle{definition}

\newtheorem{assumption}[theorem]{Assumption}
\theoremstyle{remark}
\newtheorem{remark}[theorem]{Remark}

\DeclareMathOperator{\softmax}{SoftMax}
\def\se{\sigma_e}

\makeatletter
\graphicspath{{img/}}

\def\negmathsp{\ensuremath{\!}}

\def\LCFM{\ensuremath{L_{\text{CFM}}}\xspace}
\def\Lour{\ensuremath{L_{\text{ExFM}}}\xspace}

\def\deftxt#1{\emph{#1}}
\def\ND#1#2#3{\ensuremath{%
\mathcal N(x\:|\:#2,#3I)
}}
\def\ie{\emph{i.\,e.}\xspace}
\def\eg{\emph{e.\,g.}\xspace}
\def\model(#1,#2){\ensuremath{%
v_\theta(#1,\,#2)
}}
\def\phicexpow^#1#2#3{
    \phicexusualcommon{\phi^{#1}}{#2}{#3}}
\def\phicexusual#1#2{
    \phicexusualcommon\phi{#1}{#2}}
\def\phicexusualcommon#1#2#3{
    #1_{#2,#3}\@ifnextchar(\negmathsp{}}
\def\phicex{\@ifnextchar^\phicexpow\phicexusual}
\def\phic{\phicex t{x_1}}

\def\ox{\ensuremath{%
\overline x_1
}}
\def\rhom{\rho_m}

\def\comm#1{\textit{// #1}}

\def\nosqrt{\gdef\is@sqrt{\relax}}
\gdef\is@sqrt{^2}
\def\normaldx#1#2#3{\mathcal N\left(#3\middle\vert\, #1, #2\is@sqrt\right)\gdef\is@sqrt{^2}}
\def\normald#1#2{\normaldx{#1}{#2}\cdot}

\DeclareMathOperator{\divtrue}{div}

\def\myfrac#1#2{\left(#1\middle)\middle/#2\right.}
\def\myfracinv#1#2{#1\left/\middle(#2\right)}

\makeatother
\title{Explicit Flow Matching: \\
On The Theory of Flow Matching Algorithms with Applications}

\author{%
  Gleb Ryzhakov \\
  CAIT,
  Skolkovo Institute of Science and Technology  \\
  Bolshoy Boulevard, 30, p.1, Moscow 121205, Russia \\
  \texttt{G.Ryzhakov@skol.tech} \\
   \And
  Svetlana Pavlova \\
  CAIT, 
  Skolkovo Institute of Science and Technology  \\
  Bolshoy Boulevard, 30, p.1, Moscow 121205, Russia \\
  \texttt{Svetlana.Pavlova@skol.tech} \\
     \And
  Egor Sevriugov \\
  CAIT,
  Skolkovo Institute of Science and Technology  \\
  Bolshoy Boulevard, 30, p.1, Moscow 121205, Russia \\
  \texttt{Egor.Sevriugov@skol.tech} \\
    \And
  Ivan Oseledets \\
  AIRI\\
  p. 19, Nizhny Susalny per. 5, Moscow, 105064, Russia\\
  and \\ CAIT, 
  Skolkovo Institute of Science and Technology  \\
  Bolshoy Boulevard, 30, p.1, Moscow 121205, Russia \\
  \texttt{I.Oseledets@skol.tech} \\
}

\begin{document}

\pdfstringdefDisableCommands{%
  \def\\{}%
  \def\texttt#1{<#1>}%
}

\maketitle

\begin{abstract}
This paper proposes a novel method, Explicit Flow Matching (ExFM), for training and analyzing flow-based generative models. ExFM leverages a theoretically grounded loss function, ExFM loss (a tractable form of Flow Matching (FM) loss), to demonstrably reduce variance during training, leading to faster convergence and more stable learning. Based on theoretical analysis of these formulas, we derived exact expressions for the vector field (and score in stochastic cases) for model examples (in particular, for separating multiple exponents), and in some simple cases, exact solutions for trajectories. In addition, we also investigated simple cases of diffusion generative models by adding a stochastic term and obtained an explicit form of the expression for score. While the paper emphasizes the theoretical underpinnings of ExFM, it also showcases its effectiveness through numerical experiments on various datasets, including high-dimensional ones. Compared to traditional FM methods, ExFM achieves superior performance in terms of both learning speed and final outcomes.
\end{abstract}

\section{Introduction}
In recent years, there has been a remarkable surge in Deep Learning, wherein the advancements have transitioned from purely neural networks to tackling differential equations. Notably, Diffusion Models~\cite{sohldickstein2015deep}
%\todo{check the cite} 
have emerged as key players in this field. 
%These models leverage the mathematics of diffusions, which are continuous-time stochastic processes, to model probability distributions. One key 
These models 
%aspect of diffusion models is the ability to 
transform a simple initial distribution, usually a standard Gaussian distribution, into a target distribution via a solution of Stochastic Differentiable Equation (SDE) \cite{albergo_stochastic_2023} or Ordinary Differentiable Equation (ODE)\cite{albergo_building_2023} with right-hand side 
representing a trained neural network. 
The Conditional Flow Matching (CFM)~\cite{lipman2023flow} technique, which we focus on in our research,
is a promising approach for constructing probability distributions using conditional probability paths, 
which  is notably a robust and stable alternative for training Diffusion Models.
%We contribute to the theoretical part for training Flow Matching models by introducing new objectives. These objectives aim to improve the training process and enhance the quality of the generated samples.
The development of the CFM-based approach includes various techniques and heuristics \cite{chen2023riemannian,
jolicoeurmartineau2023generating, pooladian2023multisample} aimed at improving convergence or quality of learning or inference.
%Further works have attempted to improve the convergence, quality or speed of CFM by using various heuristics. 
For example, in the works 
\cite{tong2024improving,
Tong2024,
liu2022flow}
it was proposed to straighten the trajectories between points by different methods,
which led to serious modifications of the learning process.
We refer the reader for, example, to the paper~\cite{Tong2024} where different FM-based approaches are summarised,
and to the paper~\cite{lipman2023flow} for the connection between Diffusion Models and CFM. 

In our work, 
we introduced an approach 
which we called 
Explicit Flow Matching (ExFM),
to consider the Flow Matching framework theoretically 
by modifying the loss and writing the explicit value of the vector field.
%We call our mathod 
Strictly speaking, the presented loss is a tractable form of the FM loss, see Eq.~(5) of~\cite{lipman2023flow}.
Base on this methods
we can improve the convergence of the method in practical examples reducing the variance of the loss, 
but the main focus of our paper is on theoretical derivations.

Our method allows us to write an expression for the vector field in closed form for quite simple cases (Gaussian distributions), however, we note that
Diffusion Models framework 
in the case of a Gaussian Mixture of two Gaussian as a target distribution
is still under investigation, see recent publications~\cite{shah2023learning,NEURIPS2023_06abed94}.

%Through our work, we also emphasize the advancements made through Flow Matching models in deep learning. These models offer a powerful tool for generative modelling, allowing us to capture the underlying distribution of complex data. Additionally, we highlight the potential of new Explicit Flow Matching method in building probability paths, showcasing their ability to generate high-quality samples that resemble the original data.

%Therefore, 
Our main contributions are:
\begin{enumerate}
\item A tractable form of the FM loss is presented, 
which reaches a minimum on the same function as the loss used in Conditional Flow Matching, but has a smaller variance;
\item 
%The presented loss allows us to obtain an explicit expression for the vector field delivering the minimum to it.
The explicit expression in integral form for the vector field delivering the minimum to this loss (therefore for Flow Matching loss) is presented.
%This expression depends on the initial and final densities, as well as on the specific type of conditional mapping used;
\item 
As a consequence, we derive expressions for the flow matching vector field and score in several particular cases (when linear conditional mapping is used, normal distribution, etc.);
\item Analytical analysis of SGD convergence showed that our formula have better training variance on several cases;
\item 
Numerical experiments show that we can achieve better learning results in fewer steps.
\end{enumerate}

\subsection{Preliminaries}
Flow matching is well known method for finding a flow to connect samples
from two
distribution with densities~$\rho_0$
and~$\rho_1$.
It is done by solving continuity equation with respect to the
time dependent vector field~$\overline v(x,t)$ and time-dependent density $\rho(x,t)$
with boundary conditions:
\begin{equation}
\left\{
\begin{aligned}
    %\frac{\partial \rho(x,t)}{\partial t} 
    \pdv{\rho(x,t)}{t}
    &= -\divtrue(\rho(x,t) \overline v(x,t)), \\
    \rho(x,0) &= \rho_0(x), \quad
    \rho(x,1) = \rho_1(x).
\end{aligned}
\right.
\label{eq:FP}
\end{equation}
Function $\rho(x,t)$ is called \deftxt{probability density path}.
Typically,
the distribution~$\rho_0$ is known and
it is chosen for convenience reasons,
for example, as standard normal distribution~$\rho(x)=\ND x0{}$.
The distribution~$\rho_1$ is unknown and we only know the set of samples from it,
so the problem is to approximate the vector field $v(x,t)\approx\overline v(x,t)$ using these samples.
To make problem~\eqref{eq:FP} well defined, 
one usually imposes additional regularity conditions on the densities,
such as smoothness.
The rigorous justification of the obtained results we put in the Appendix, 
leaving the general formulations of theorems and ideas in the main text.

From a given vector field, we can construct a \deftxt{flow}~$\phi_t$, \ie, a time-dependent map,
satisfying the ODE $
\pdv{\phi_t(x)}{t} =  v(\phi_t(x),t)
$
with initial condition 
$
\phi_0(x) = x
$.
%\begin{equation*}
%\left\{
%\begin{aligned}
%$    %\frac{d\phi_t(x)}{dt} 
%\pdv{\phi_t(x)}{t} &=  v(\phi_t(x),t), \\
%$    \phi_0(x) &= x \\
%\end{aligned}
%\right.
%$.
%\end{equation*}
%Once the flow is found,
Thus,
one can sample a point~$x_0$ from the distribution~$\rho_0$
and then using this ODE obtain a point $x_1=\phi_1(x_0)$
which  have a distribution approximately equal to~$\rho_1$.
For given boundary $\rho_0$ and $\rho_1$,
the vector field or path solutions are not the only solutions,
but if we have found any solution, it will already allow us to sample from the unknown density $rho_1$. 
However, if the problem is more narrowly defined, \eg,
one needs to have a map that is close to the Optimal Transport (OT) map, 
we have to impose additional constraints. 

The problem of finding any vector field~$v$ 
is solved in conditional manner in the paper~\cite{lipman2023flow},
where so-called Conditional Flow Matching (CFM) is present.
Namely, the following loss function was introduced 
for the training a model~$v_\theta$
which depends on parameters~$\theta$
%\begin{multline}
\begin{equation}
    \label{eq:LCFM}
    \LCFM(\theta)=
    \mathbb E_t
    \mathbb E_{x_1,x_0}
    \norm{
    \model(\phic(x_0),t)-
    \phic'(x_0)
    }^2,
\end{equation}
%\end{multline}
where $\phic(x_0)$ is some flow, conditioned on $x_1$ % connecting~$x_0$ and~$x_1$
(one can take $\phic(x_0)=(1-t)x_0+tx_1+\sigma_stx_0$ in the simplest case, where $\sigma_s>0$ is a small parameter need for this map to be invertable at any $0\leq t\leq1$).
Hereinafter the dash indicates the time derivative.
%for any time-dependent function 
%($f':=\pdv tf$).
Time variable $t$ is uniformly distributed: $t\sim\mathcal U[0,1]$
and 
random variables $x_0$ and $x_1$ are distributed according to the initial and final distributions, respectively:
$x_0\sim\rho_0$, $x_1\sim\rho_1$.
Below we  omit specifying of the symbol~$\mathbb E$ the distribution by which the expectation is taken
where it does not lead to ambiguity.
\begin{figure}[!tb]
\centering
% \subfigure[Usual CFM]{\label{fig:diff:a}\includegraphics[width=0.48\linewidth]{img/a.jpg}}\hfill
% \subfigure[Our method]{\label{fig:diff:b}\includegraphics[width=0.48\linewidth]{img/b.jpg}}
\label{fig:main_scheme}
%\hss
%\framebox{
\subfigure{\label{fig:diff}\includegraphics[clip,width=0.5\linewidth]{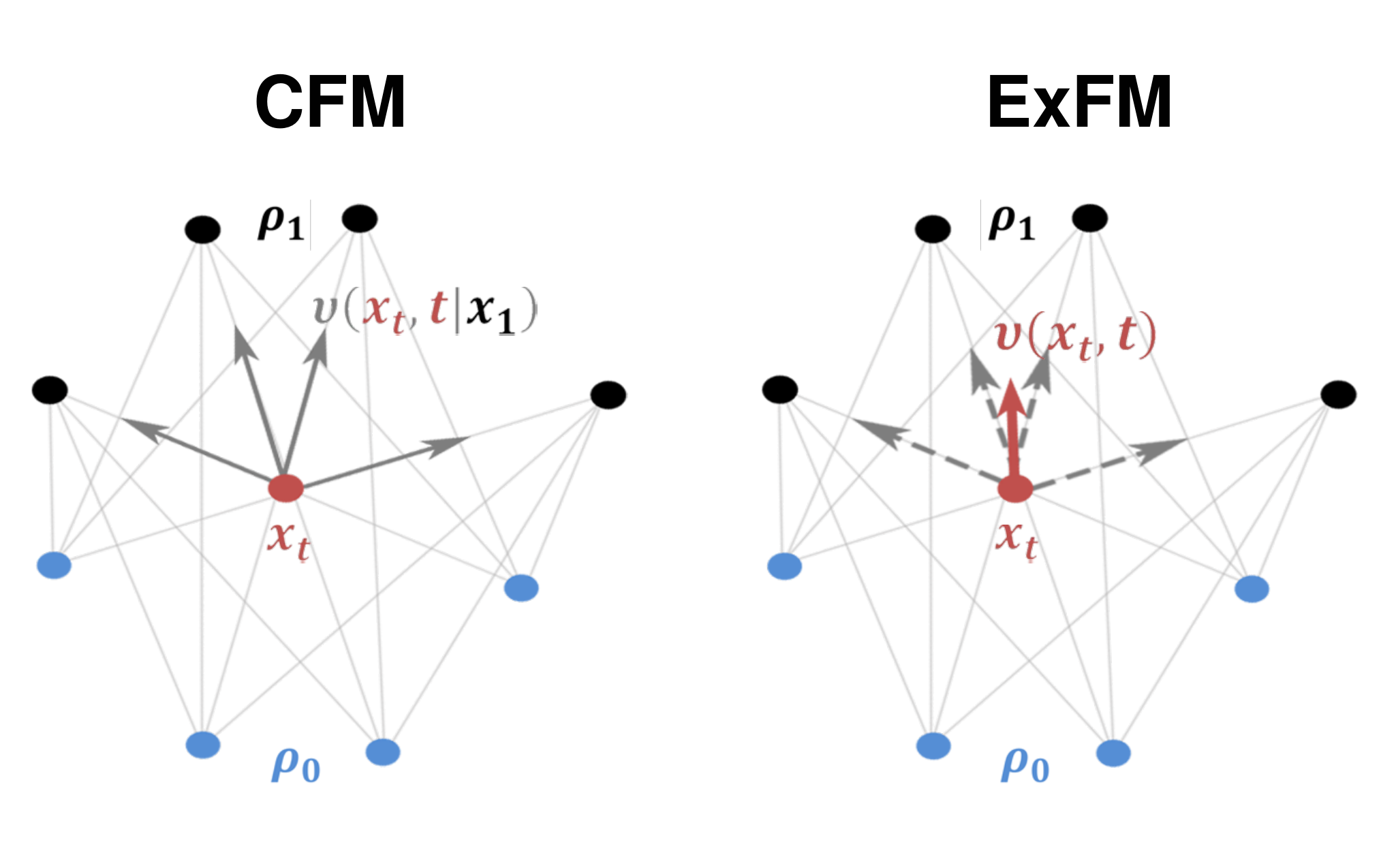}}
\subfigure{\label{fig:CFM_ExFM_disps}\includegraphics[clip,width=0.4\linewidth]{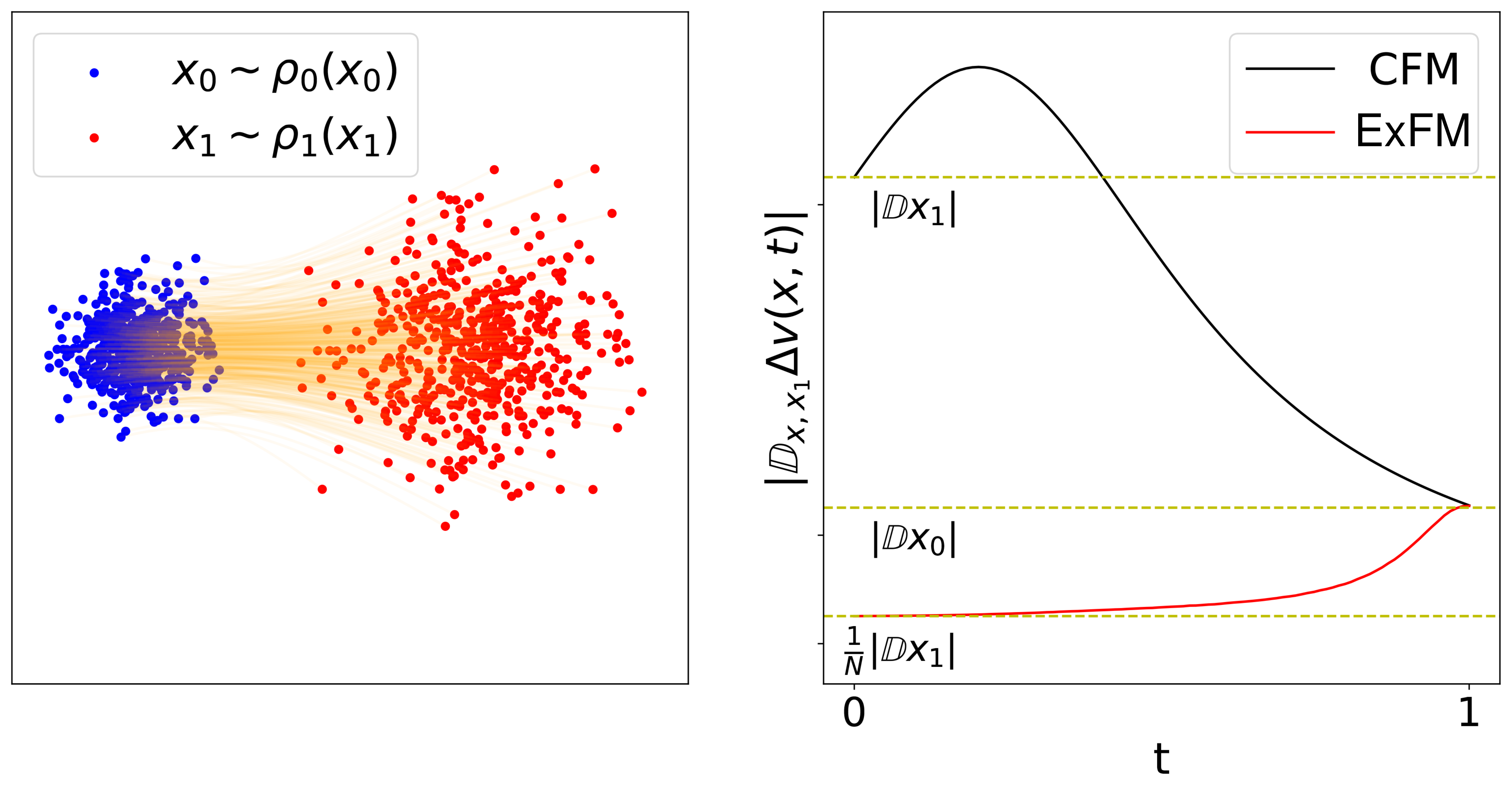}}
%}
%\hss
\vspace*{-2ex}
\caption{%
(Left) 
The key novelty of our approach is that in classical CFM, highly divergent directions can appear
    in a small spatial area at similar times (left part).
    In our approach (right part) we average over these vectors, 
    training the model on a smoothed unnoised vector field.
(Right)
The comparison evaluated dispersion norm 
over time parameter~$t$
for CFM and ExFM in matching standard Gaussian $\rho_0 = \mathcal N(0,I)$ to general Gaussian $\rho_1 =  \mathcal  N(\mu,\sigma^2I)$ distributions. The y-axis represents the sum of dispersion vector components, denoted as $|\mathbb D_{x,x_1}\Delta v(x,t)|$.
The left panel illustrates samples drawn from the $\rho_0$ and $\rho_1$ distributions, 
as well as the corresponding flows. The right panel depicts the dispersion trend over time for both CFM (black line) and ExFM (red line) objectives. The dotted lines correspond to the dispersion levels (in top-down order $|\mathbb D x_1|$, $|\mathbb D x_0|$, $|\mathbb D x_1| / N$.
}
%\label{fig:diff}
\end{figure}

%\begin{figure}[!tb]
%\centering
%\label{fig:CFM_ExFM_disps}
%\includegraphics[clip,width=0.4\linewidth]{img/CFM_ExFM_comp.pdf}
%\vspace*{-2ex}
%\caption{%
%}
% \label{fig:diff}
%\end{figure}

\subsection{Why new method?}
Model training using loss~\eqref{eq:LCFM}
have the following disadvantage:
during training, due to the randomness of~$x_0$ and~$x_1$,
significantly different values can be presented for model as output value at close model argument values~$(x_t,t)$.
Indeed, a fixed point~$x_t=\phic(x_0)$
can be obtained by an infinite set of~$x_0$ and~$x_1$ pairs,
some of which are directly opposite,
and at least for small times~$t$
the probability of these different directions may not be significantly different.
At the same time, data~$\phic'(x_0)$
on which the model learns
 %to be trained to the model
significantly different for such different positions of pairs~$x_0$ and~$x_1$.
Thus, the model is forced to do two functions during training: generalize and take the mathematical expectation (clean the data from noise).

In our approach, see Fig.~\ref{fig:diff},
we feed the model input with cleaned data with small variance.
Thus, the model only needs to generalize the data,
which happens much faster (in fewer training steps).

Moreover,
in the process of constructing the modified loss,
we have developed the exact formula for the vector field,
see Eq.~\eqref{eq:vtrue}, \eqref{eq:vtrue:pi}.
The existence of an explicit formula for the vector field 
is of great importance not only from a theoretical but also from a practical point of view.
%which may be interesting from the theoretical point of view.

%Instead of $\rho(x,t)$ and $v(x,t)$ we are working with conditional probability paths $\rho(x,t|x_0,x_1)$ and velocity $v(x,t|x_0,x_1)$, where condition are starting and ending points $x_0, x_1$ of probability path.

\section{Main idea}
\subsection{Modified objective}
Lets expand the last two mathematical expectations in the loss~\eqref{eq:LCFM}
and substitute variables using map~$\phic$,
passing from the point~$x_0$ to its position~$x_t=\phic(x_0)$ at time~$t$:
%\begin{equation}
    \begin{multline}
        \label{eq:LCMour_1}
        %\LCFM(\theta)=
        \!\!
        \mathbb E_{x_1,x_0}
        \norm{
        \model(\phic(x_0),t)-
        \phic'(x_0)
    }^2\!=
        \!\!\iint
        \norm{
        \model(\phic(x_0),t)-
        \phic'(x_0)
    }^2\!\!
    \rho_0(x_0)
    \rho_1(x_1)
    \dd x_0\dd x_1
    \\
        \!=\!\!
        \iint
        \norm{
        \model(x_t,t)-
        \phic'\bigl(\phic^{-1}(x_t)\bigr)
    }^2
    %\times
    \underbrace{\det[\eval{\pdv*{\phic^{-1}(x)}{x}}_{x=x_t}]
    \rho_0\bigl(\phic^{-1}(x_t)\bigr)}_{\rho_{x_1}(x_t,t)}
    \rho_1(x_1)
    \dd {x_t}\dd x_1
    \\
    =
    %    \iint
    %    \norm{
    %    \model(x_t,t)-
    %    \phic'\bigl(\phic^{-1}(x_t)\bigr)
    %}^2
    %\times\\
    %\rho_{x_1}(x_t,t)
    %\rho_1(x_1)
    %\dd {x_t}\dd x_1=
        \mathbb E_{x_1,x_t\sim\rho_{x_1}(\cdot,t)}
        \norm{
        \model(x_t,t)-
        \phic'\bigl(\phic^{-1}(x_t)\bigr)
    }^2
    .
    \end{multline}
%\end{equation}
%Let us consider some interpolation function between starting and ending points:
%$$x = f(x_0,x_1,t)$$
We assume, that the map~$\phic$ is invertable at each~$0<t<1$,
\ie that~$\phic^{-1}(x_t)$ exits on this time interval and for all~$x_t=\{\phi_t(x_0)\mid\forall x_0:\rho(x_0)>0\}$.
Eq.~\eqref{eq:LCMour_1}
can be seen as a transition from expectation on the variable~$x_0\sim\rho_0$ to expectation on the variable
$x_t\sim\rho_{x_1}(\cdot,t)$,
where 
$
\rho_{x_1}(x,t)=[\phic]_*\rho_0(x):=\rho_0\bigl(\phic^{-1}(x)\bigr)\det[\pdv*{\phic^{-1}(x)}{x}]
$.
See paper \cite{Chen2018NODE} for details about
%change of variables in probability densities using flows and
the push-forward operator ``*''.
Our representation~\eqref{eq:LCMour_1} is very similar to expression~(9) of the cited paper~\cite{lipman2023flow},
only we write it in terms of the conditional flow rather than the conditional vector field.

\def\xt{x_t}
To obtain the modified loss,
we return to end of the standard CFM loss representation in~\eqref{eq:LCMour_1}.
%At the end,
It is written as the expectation over two random variables~$x_1$ and~$\xt$
having a common distribution density
\begin{equation}
    \label{eq:rho_jount}
    \{x_1,\xt\}\sim\rho_j(x_1,\xt,t)=
    \rho_{x_1}(\xt,t)
    \rho_1(x_1),
\end{equation}
which, generally speaking, is not factorizable.
Let us rewrite this expectations in terms of two independent random variables,
each of which have its marginal distribution.
%Recall, that marginal distribution of~$x$ given in~\eqref{eq:def_rhom}
The marginal distribution~$\rhom$ of~$\xt$
%given in~\eqref{eq:def_rhom}
can be obtained via integration:
\begin{equation}
    \label{eq:def_rhom}
    \rhom(\xt,t)=
    %\mathbb E_{x_1}
%\rho_{x_1}(\xt,t)
%=
    %\!
    \int
        \rho_{j}(x_1,\xt,t)
    \dd {x_1}
    =
    %\!\!
    \int
        \rho_{x_1}(\xt,t)
    \rho_1(x_1)
    \dd {x_1},
\end{equation}
while the marginal distribution of~$x_1$ is just (unknown) function~$\rho_1$. %
%and write it in integral form
%in the case of the linear map~\eqref{eq:linear_map}.
%we return to the loss~\eqref{eq:LCFM}
%and make there changing of variables
%of the form
%~$x_t\gets\phico(x_0)$,
%performing calculations similar
%to those made during changing of variables in~\eqref{eq:LCMour_1}.
Let for convenience $w(t,x_1,x)=\phic'\bigl(\phic^{-1}(x)\bigr)$\footnote{
Note, that $w(t,x_1,x)$ is the conditional velocity at the given point~$x$.}.
We have
\begin{multline}
    %\!\!\!\!\!
    \LCFM(\theta)=
    \mathbb E_{t,x_1,\xt\sim\rho_{x_1}(\cdot,t)}
    \norm{
        \model(\xt,t)-
        w(t,x_1,\xt)
        %\phic'\bigl(\phic^{-1}(\xt)\bigr)
    }^2
    =\\
    \int_0^1
        \iint
        \norm{
        \model(\xt,t)-
        w(t,x_1,\xt)
        %\phic'\bigl(\phic^{-1}(x)\bigr)
    }^2
    %\times
    \rho_{x_1}(x,t)
    \rho_1(x_1)
    \dd {\xt}\dd x_1\dd t
    =\\
    \int_0^1
        \iint
        \norm{
        \model(\xt,t)-
        w(t,x_1,\xt)
        %\phic'\bigl(\phic^{-1}(x)\bigr)
    }^2
    \,
    %\times
    \left(\nicefrac{\rho_{x_1}(\xt,t)}{\rho_m(\xt,t)}\right)
    \rho_m(\xt,t)
    \rho_1(x_1)
    \dd {\xt}\dd x_1\dd t
    =\\
    \def\xt{x}
    \mathbb E_{t,x_1,\xt\sim\rho_m(\cdot,t)}
    \norm{
        \model(\xt,t)-
        w(t,x_1,\xt)
        %\phic'\bigl(\phic^{-1}(\xt)\bigr)
    }^2
    %\frac{\rho_{x_1}(x,t)}{\rho_m(x,t)}
    \,
    %\times
    \rho_{c}(x|x_1,t)/\rho_1(x_1)
    ,
    \label{eq:LCM_ind}
\end{multline}
where we introduce a  conditional distribution
\def\xt{x}
\begin{equation}
    \rho_{c}(x|x_1,t)
    := 
    \nicefrac[\displaystyle]{\rho_{x_1}(x,t)\rho_1(x_1)}{\rho_m(x,t)}
    :=
    %\nicefrac[\displaystyle]
    {\rho_{x_1}(x,t)\rho_1(x_1)}
    \bigg/
    {
        \int
        \rho_{x_1}(\xt,t)
        \rho_1(x_1)
        \dd {x_1}
    }
    .
    \label{eq:cond_distr}
\end{equation}
The key feature
of the representation~\eqref{eq:LCM_ind}
is that 
the integration variables~$x_1$ and $\xt$ are independent.
Thus, we can evaluate them using Monte Carlo-like schemes in different ways.
%In particular,
%we can take a different number of samples for each of these variables.
%This is fundamentally different from the usual CFM approach,
%where the number of~$x_0$ samples equaled the number of~$x_1$ samples.
However,
we go further and make a modification to this loss
to reduce the variance of Monte Carlo methods.

%\subsection{Exact expression for vector field}
\subsection{New loss and exact expression for vector field}
Note that so far the expression for $\LCFM$ have not changed,
it has just been rewritten in different forms.
Now we change this expression
so that its numerical value,
generally speaking,
may be different,
but the derivative of the model parameters will be the same.
We introduce the following loss
\begin{multline}
    \label{eq:Lour}
    \Lour(\theta)=
    \mathbb E_t
    \mathbb E_{\xt\sim\rhom}
    \Bigl\|
    \model(\xt,t)-
    {}
        %\frac{\mathbb E_{x_1\sim\rho_1}
        %w(t,x_1,\xt)\rho_{x_1}(\xt,t)}{
%\mathbb E_{x_1\sim\rho_1}
%\rho_{x_1}(\xt,t)
    %}
    \mathbb E_{x_1\sim\rho_1}
    w(t,x_1,\xt)
    \rho_{c}(\xt|x_1,t)/\rho_1(x_1)
        \Bigl\|^2
        =
        \\
%%%%%%%%%%%%%%%%%%%%%%%%%%%%%%%%%%%%%%%%%%%%%%%%%%%%%%%%%%%%%%%%%%%%%%%%%%%%%%%%%%%%%%%%%%%%%%%%%%%%%%%%
    \int_0^1
        \int
        %\norm{
        \Bigl\|
        \model(\xt,t)-
        %{}\\
        %\frac1{\rhom(\xt,t)}
        \int
        w(t,x_1,\xt)
    \times
        %\rho_{x_1}(\xt,t)
    \rho_{c}(\xt|x_1,t)
    %\rho_1(x_1)
    \dd {x_1}
    \Bigr\|^2
        \rhom(\xt,t)
    \dd {\xt}\dd t
    .
\end{multline}
%where
%We change the notation~$x_t$ to~$x$ here to emphasize
%that~$x$ do not depend on time~$t$ but is a dummy variable
%(although its distribution is time-dependent).

\begin{theorem}
    \label{th:grads_eq}
    Losses
    $\LCFM$ in Eq.~\eqref{eq:LCFM}
    and
    $\Lour$ in Eq.~\eqref{eq:Lour}
    have the same derivative with respect to model parameters:
\begin{equation}
    \label{eq:grad}
    \dv*{\LCFM(\theta)}{\theta}
    =
    \dv*{\Lour(\theta)}{\theta}.
\end{equation}
\end{theorem}
Proof is in the Appendix~\ref{seq:App:diff}.

In the presented loss~\Lour,
%the spatial argument of the model function is the variable~$\xt$,
%over which the external integration is performed.
%Moreover,
%the other terms under the norm do not contain
%external integration variables other than~$\xt$ (and~$t$).
the integration (outside the norm operator) 
proceeds on those variables on which the model depends, 
while inside this operator there are no other free variables.
Thus, using this kind of loss,
it is possible to find an exact analytical expression for the vector field
for which the minimum of this loss is zero (unlike the loss \LCFM).
Namely, we have
\begin{equation}
v(x,t)
=\int w(t,x_1,x)
\rho_{c}(x|x_1,t)
% \rho_1(x_1)
\dd{x_1}.
\label{eq:vtru_comm}
\end{equation}
We can obtain the exact form of this vector field
given the particular map~$\phic$.
For example, the following statement holds:
\begin{corollary}
Consider the linear conditioned flow
%\begin{equation}
%    \label{eq:linear_map}
$
    \phic(x_0)=(1-t)x_0+tx_1
$
%\end{equation}
which is inevitable as $0\leq t<1$.
Then
$
w(t,x_1,x)=\frac{x_1-x}{1-t}
$,
$
\rho_{x_1}(x,t)
=
\rho_0\left(
    \frac{x-x_1 t}{1-t}
\right)
\frac1{(1-t)^d}
$
and the loss
$\Lour$ in Eq.~\eqref{eq:Lour}
reaches zero value when the model of the vector field
have the following analytical form
%\begin{multline}
\begin{equation}
v(x,t)
=
%=\int w(t,x_1,x)
%\rho_{c}(x|x_1,t)
%\rho_1(x_1)
%\dd{x_1}=\\
\myfracinv{
\int
(x_1-x)
\rho_0\left(
\frac{x-x_1 t}{1-t}
\right)
%\frac1{1-t}
\rho_1(x_1)
\dd{x_1}}{
( 1-t )
\int
\rho_0\left(
\frac{x-x_1 t}{1-t}
\right)
\rho_1(x_1)
\dd{x_1}
}{
}
.
\label{eq:vtrue}
%\end{multline}
\end{equation}
This is the exact value of the vector field
whose flow translates the given distribution~$\rho_0$ 
to~$\rho_1$.
\end{corollary}
Complete proofs are in the Appendix~\ref{sec:App:expl}.
Note that the result~\eqref{eq:vtrue} is not totally new, 
for example, a similar result (though in the form of a general expression rather than an explicit formula), 
was given in~\cite{tong2024improving}, Eq.~(9).
However, our contribution consists of both the general form~\eqref{eq:vtru_comm}
and practical and theoretical conclusions from it (see below).

\begin{remark}
In the case of the initial and final times
$t=0,\,1$, 
Eq.~\eqref{eq:vtrue}
is noticeably simpler
\begin{equation}
v(x,0)=
\mathbb E_{x_1}x_1 - x
=
\int x_1\rho_1(x_1)\dd{x_1}-x.
\quad
v(x,1)=
x-\int x_0\rho_0(x_0)\dd{x_0}.
\label{eq:v:t0}
\end{equation}
This expression for the initial velocity means that
each point first tends
to the center of mass of the unknown distribution~$\rho_1$
regardless of its initial position.
%Replacing the variables in~\eqref{eq:vtrue} $y\gets\frac{x-x_1 t}{1-t}$
%and taking the limit $t\to0$ (given that~$\rho_0$ is non-negative and integrable at infinity,
%and assuming that~$\rho_1$ is bounded, see Appendix for strict derivation)
%we obtain a similar formula for the final time~\hbox{$t=1$}:
%\begin{equation}
%v(x,1)=
%x-\int x_0\rho_0(x_0)\dd{x_0}.
%\label{eq:v:t1}
%\end{equation}
\end{remark}

%\begin{remark}
%Generalizing the previous remark, 
%consider independent $x_0$ and $x_1$,
%and
%arbitrary mapping~$\phi$.
%Then, $w(t=0,x_1,x)=\phicex0{x_1}'(x)$,
%$\det[\pdv*{\phicex^{-1}0{x_1}(x)}{x}]=1$
%and
%vector field
%$v$ at the time~$t=0$ equal to:
%$
%v(x,0)=
%\int
%\phicex0{x_1}'(x)\rho(x_1)\dd x_1
%$.
%%where the function~$F(x,x_1)$
%For any map this expression do not depend on the density~$\rho_0$
%thus, in general, there is no universal map that 
%delivers some property which depends on both distributions. 
%In particular, there do not exist an universal map
%that would perform optimal transport (OT) between 
%all the given densities~$\rho_0$ and~$\rho_1$.
%\end{remark}

\paragraph{Extensions to SDE}
Now let the conditional map be stochastic:
%\begin{equation*}
$
\phic =
(1-t)x_0 + tx_1 + \sigma_e(t)\epsilon
$,
%\end{equation*}
where $\epsilon\sim\mathcal N(0,1)$. % is the standard normal random variable.
Typically, $\sigma_e(0)=\sigma_e(1)=0$, for example, $\sigma_e(t)=t(1-t)\sigma_e$.

Note that this formulation covers (with appropriate selection of the $\sigma_e(t)$ parameter)
the case of diffusion models~\cite{Tong2024}.
%if we make a substitution of variables\footnote{such that $t'=0$ corresponds to the initial sample, and $t'=\infty$ to the full noise.}
%for time $t=\exp(-t')$.

Then, we can write the exact solution 
for a so-called \emph{score and flow
matching} objective 
(see~\cite{Tong2024} for details)
\begin{equation*}
%\mathcal{L}_{\mathrm{U[SF]}^2\text{M}}(\theta)=\\
%\mathbb{E}\big[\underbrace{\|v_\theta(x,t)-u_t^\circ(x)\|^2}_{\text{flow matching loss}}+
%\lambda(t)^2\underbrace{\|s_\theta(x,t)-\nabla\log p_t(x)\|^2}_{\text{score matching loss}}\big].
\mathcal{L}_{\mathrm{[SF]}^2\text{M}}(\theta)=
\mathbb{E}\big[\underbrace{\|v_\theta(x,t)-u_t^\circ(x)\|^2}_{\text{flow matching loss}}+
\lambda(t)^2\underbrace{\|s_\theta(x,t)-\nabla\log p_t(x)\|^2}_{\text{score matching loss}}\big].
\end{equation*}
that corresponds to this map.
In the last expression,
the following explicit conditional expressions 
are considered 
in the cited paper
for the case $\sigma_e(t)=\sqrt{t(1-t)}\sigma_e$
\begin{equation*}
%\begin{aligned}
u_t^\circ(x)=\frac{1-2t}{t(1-t)}(x-(tx_1+(1-t)x_0))+(x_1-x_0),\;\;
\nabla\log p_t(x)=\frac{tx_1+(1-t)x_0-x}{\sigma_e^2t(1-t)}.
%\end{aligned}
\end{equation*}

The exact solution (our result, explicit analog of the Eq.~(10) from \cite{Tong2024}) 
under consideration has the form~\eqref{eq:vtrue:SDE:1} and~\eqref{eq:strue:SDE:1}
and, for example for the for the Gaussian~$\rho_0$
this expressions reduced to the Eq.~\eqref{eq:vtrue:SDE:G} and~\eqref{eq:strue:SDE:G}, correspondingly.
See Appendix~\ref{sec:SDE} for the details on this case.

\paragraph{Simple examples}
Consider the case of Standard Normal Distribution as $\rho_0$ and Gaussian Mixture of two Gaussians as $\rho_1$.
Vector field have a closed form~\eqref{eq:vtrue:G-GG} in this case, and we can fast numerically solve ODE for trajectories.
Random generated trajectories and plot of the vector field are shown on Fig.~\ref{fig:G-GM:main} (a)--(b).
Detailed explanation of this case is in the Sec.~\ref{sec:vtrue:G-GG}.
Another example is related to the case of a stochastic map in the form of Brownian Bridge, which briefly described in the last paragraph and considered in Sec.~\ref{seq:SDE:samples} in details,
see Fig.~\ref{fig:G-GM:main} (c)--(f).
Note that at some $\sigma_e$ values the trajectories are a little bit straightened in this case compared to the usual linear map,
if we compare cases on the Fig.~\ref{fig:G-GM:SDE}.
\begin{figure}[tb]%
%for NIPS
\def\lw{0.15}
\def\rw{0.1}
%for arXiv
\def\lw{0.25}
\def\rw{0.2}
\def\pc#1{
\subfigure[\footnotesize BB Trajectories, $\sigma_e=#1$]{
\includegraphics[clip,width=\lw\linewidth]{traj_se=#1.pdf}
}
\subfigure[\footnotesize BB VF, $\sigma_e=#1$]{
\includegraphics[clip,width=\rw\linewidth]{img/field_se=#1.png}
}
}
\centering
\subfigure[\footnotesize GM trajectories]{
\includegraphics[clip,width=\lw\linewidth]{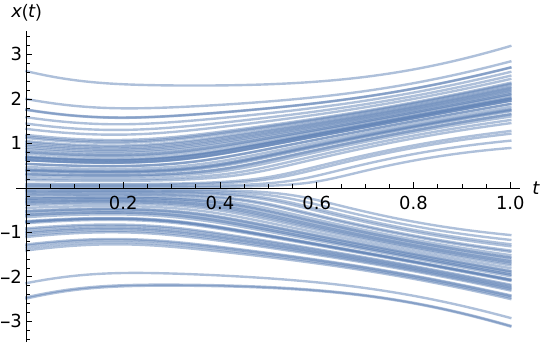}
}
\subfigure[\footnotesize GM VF]{
\includegraphics[clip,width=\rw\linewidth]{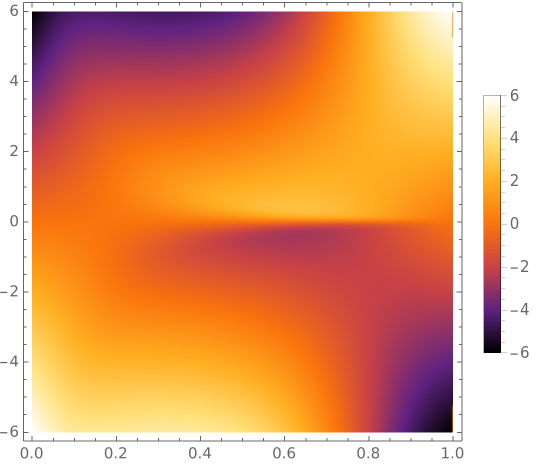}
}
\pc{3}
\pc{10}
\caption{Trajectories and vector field obtained in simple cases:
(a) $N=80$ random trajectories from $\normald01$ to GM;  %  of $\normald{-2}{1/2}$ and $\normald2{1/2}$;
(b) 2D plot of the vector field in this case
(c)--(f) $N=40$ random trajectories from $\normald01$ to $\normald23$
and
2D plot of the vector field% in this case
for different~$\sigma_e$
for the Brownian Bridge map
}
\label{fig:G-GM:main}
\end{figure}

\subsection{Training scheme based on the modified loss}
Let us consider the difference between our new scheme
based on loss \Lour and the classical CFM learning scheme.
As a basis for the implementation of the learning scheme,
we take the open-source code\footnote{\url{https://github.com/atong01/conditional-flow-matching}}
from the works~\cite{Tong2024,tong2024improving}.

Consider a general framework of numerical schemes in classical CFM.
%based on loss~\Lour.
We first sample~$m$ random time variables~$t\sim\mathcal U[0,1]$.
Then we sample several values of~$x$.
To do this,
%as in the classical CFM approach,
we sample a certain number~$n$ samples~$\{x_0^i\}_{i=1}^n$ from the  ``noisy'' distribution~$\rho_0$,
and the same number~$n$ of samples~$\{x_1^i\}_{i=1}^n$ from the unknown distribution~$\rho_1$.
Then we pair them (according to some scheme),
and get~$n$ samples as
$x^{j,i}=\phicex{t^j}{x_1^i}(x_0^i)$
(\eg a linear combination in the simple case of linear map:
$x^{j,i}=(1-t^j)x_0^i+t^jx_1^i$),
$\forall i=1,2,\dots, n$;
$\forall j=1,2,\dots, m$.
Note, than one of the variable $n$ or $m$ (or both) can be equal to~$1$.

At the step 2, the following discrete loss is constructed from the obtained samples
\begin{equation}
    \label{eq:discrete_loss_CFM}
    \LCFM^d(\theta)=
    \sum_{j=1}^m
    \sum_{i=1}^n
        %\norm{
        \Bigl\|
        \model(x^{j,i},t^j)-
        \phicex{t^j}{x_1^i}'(x_0^i)
    \Bigr\|^2
    .
\end{equation}
Finally, we do a standard gradient descent step to update model parameters~$\theta$ using this loss.

The first and last step in our algorithm is the same as in the standard algorithm, but the second step is significantly different. Namely,
%The first and the last stpe
%The second step is different from standard CFM in our approach.
we additionally generate
a sufficiently large number~$N\gg n\cdot m$ of samples~$\ox$
from the unknown distribution~$\rho_1$,
sampling $(N-n)$ new samples and
adding to it the samples~$\{x_1^{i}\}_1^{n}$
that are already obtained on the previous step.

Then we form the following discrete loss
which replaces the integral on~$x_1$ in~$\Lour$ 
by its evaluation~$v^d$
by self-normalized importance sampling
or rejection sampling (see Appendix~\ref{sec:Int-est} for details)
\begin{equation}    
    \label{eq:discrete_loss_our}
    \Lour^d(\theta)=
%    \sum_{j=1}^m
%    \sum_{i=1}^n
%        %\norm{
%        \Bigl\|
%        \model(x^{j,i},t^j)
%        -
%    \frac{\sum_{k=1}^N
%        w(t^j,\ox^k,x^{j,i})\rho_{\ox^k}(x^{j,i},\,t^j)}%
%    {\sum_{k=1}^N\rho_{\ox^k}(x^{j,i},\,t^j)}
%    \Bigr\|^2
%    =\\
    \sum_{j=1}^m
    \sum_{i=1}^n
        %\norm{
        \Bigl\|
        \model(x^{j,i},t^j)
        -
v^d(x^{j,i},\,t^j)
    \Bigr\|^2
    .
\end{equation}
For example,
if we use 
self-normalized importance sampling
and
assume that the Jacobian $\det[\pdv*{\phic^{-1}(x)}{x}]$ do not depend on~$x_1$,
we can write
\begin{equation}
\def\updown#1{\rho_0\bigl(\phicex^{-1}t{\ox^k}(#1)\bigr)}%
v^d(x,\,t)=
    \myfrac{\sum_{k=1}^N
        w(t,\ox^k,x)\updown{x}}%
    {\sum_{k=1}^N\updown{x}}
    .
\label{eq:vd}
\end{equation}

\begin{theorem}\label{th:var_lower}
Under mild conditions,
the error variance of the integral gradient~\eqref{eq:grad}
using the Monte Carlo method~\eqref{eq:discrete_loss_our}
is lower than using formula~\eqref{eq:discrete_loss_CFM}
with the same number~$n\cdot m$ 
of samples for~$\{x\}$.
\end{theorem}
Sketch of the proof is in the Appendix~\ref{seq:App:skech}.
The steps of our scheme are formally summarized in Algorithm~\ref{alg:main}.

\paragraph{Particular case of linear map and Gaussian noise}
Let $\phic$ be the linear flow: $\phic(x_0)=(1-t)x_0+tx_1$. %~\eqref{eq:linear_map}.
%Then
%\begin{multline}
%    \Lour(\theta)=\\
%    \int_0^1
%        \int
%        %\norm{
%        \biggl\|
%        \model(x,t)-
%        %\phic'\bigl(\phic^{-1}(x)\bigr)
%\dfrac{
%\int
%(x_1-x)
%\rho_0\left(
%\frac{x-x_1 t}{1-t}
%\right)
%%\frac1{1-t}
%\rho_1(x_1)
%\dd{x_1}}{
%( 1-t )
%\int
%\rho_0\left(
%\frac{x-x_1 t}{1-t}
%\right)
%\rho_1(x_1)
%\dd{x_1}}{
%}
%\biggr\|^2
%\times\\
%        \rhom(\xt,t)
%\dd{x}\dd t
%    \label{}
%,
%\end{multline}
%$$
%w(t,x_1,x)=\frac{x_1-x}{1-t}.
%$$
%Additionally,
and
consider the case of standard normal distribution for the initial density~$\rho_0$: $\rho_0(x)\sim\mathcal N(x\mid 0,I)$.
Then in the case of using self-normalized importance sampling, we have
\begin{equation}
%\begin{aligned}
    \label{eq:main_numeric}
    v^d(x,\,t)
    =
    %\myfrac{\sum_{k=1}^N
    %   \frac{\ox^k - x}{1-t} 
    %    \exp(Y^k)
    %}{\sum_{k=1}^N\exp(Y^k)},
    \sum_{k=1}^N
    \frac{\ox^k - x}{1-t}\bigl(\softmax(Y^1,\,\ldots,\,Y^N)\bigr)_k
    ,
    \qq{where}
    Y^k = -\frac12
    \frac{\norm{x- t \cdot \ox^k}^2_{\mathbb R^d}}{1-t}.
%\end{aligned}
\end{equation}
Here, the lower index~$k$ in $\softmax$ stands for the $k$-th component, and the $\softmax$ operation itself 
came about due to exponents in the Gaussian density as a more stable substitute for computing than directly through exponents.
%In the last expression, $\norm{X}^2_{\mathbb R^d}$ means sum on squares of components of~$d$-dimensional vector~$X\in\mathbb R^d$.

%For the linear map case with Gaussian noise,

\paragraph{Extension of other maps and initial densities~$\rho_0$}
Common expression~\eqref{eq:vtru_comm} can be reduced to closed form for the particular choices of density~$\rho_0$
and map~$\phi$ (consequently, expression for~$w$).
We summarise several known approaches for which FM-based techniques can be applied in Table~\ref{tab:FM_summary}%
\footnote{The idea and common structure of the Table is taken from~\cite{tong2024improving}}.
See Appendix~\ref{sec:ex:A} and~\ref{sec:anal} for derivations of formulas and for more extensions.
\begin{table}[tbh]
    \centering
    \caption{
    Correspondence between some methods which can reduced to FM framework 
    and our theoretical descriptions of them.
    %Probability path definitions for existing methods which fit in the generalized conditional flow matching framework (top) and our newly defined paths (bottom). We define two new probability path objectives that can handle general source distributions and optimal transport flows.}
    }
    %\vspace*{-1em}
    \label{tab:FM_summary}
     \resizebox{0.85\linewidth}{!}{%
    \begin{tabular}{@{}lrrr|p{12em}}
    \toprule
    Probability Path & $q(z)$ & $\mu_t(z)$ & $\sigma_t$ & \parbox{12em}{Explicit expressions: \\ vector field (VF) and score (S)}\\
    \midrule
    Var.\ Exploding~\cite{song_generative_2019} & $\rho_1(x_1)$ & $x_1$ & $\sigma_{1-t}$ & VF:  \eqref{eq:vtrue:VE}\\ 
    Var.\ Preserving~\cite{ho2020denoising} & $\rho_1(x_1)$ & $\alpha_{1-t} x_1$ & $\sqrt{1- \alpha_{1-t}^2}$ & VF:  \eqref{eq:vtrue:VP} \\
    Flow Matching~\cite{lipman2023flow}      & $\rho_1(x_1)$ & $t x_1$ & $t \sigma_s - t + 1$ & VF: \eqref{eq:vtrue} if $\sigma=0$; and \eqref{eq:vtrue:sigma} \\ 
    %{Rectified Flow}~\cite{liu2022flow}  & $\rho_0(x_0) \rho_1(x_1)$ & $t x_1 + (1-t) x_0$ & 0 & \yes & \no & \yes\\
    % {Var.\ Pres.\ Stochastic Interpolant}~\cite{albergo_building_2023} & $\rho_0(x_0) \rho_1(x_1)$ & $\cos(\frac{1}{2} \pi t) x_0 + \sin (\frac{1}{2} \pi t) x_1$ & 0 &\\
    Independent CFM & $\rho_0(x_0) \rho_1(x_1)$ & $t x_1 + (1-t) x_0$ & $\sigma$ & VF: \eqref{eq:vtru_comm} %, S: \eqref{eq:strue:SDE:1}
    \\
    %{Optimal Transport CFM}~\cite{tong2024improving} & $\pi(x_0, x_1)$ & $t x_1 + (1-t) x_0$ & $\sigma$ & \yes & \yes & \yes \\
    {Schr\"odinger Bridge CFM}~\cite{Tong2024} & 
    % \pi_{2 \sigma^2}(x_0, x_1)$
    $\rho_0(x_0) \rho_1(x_1)$
    & $t x_1 + (1-t) x_0$ & $\sigma \sqrt{t (1-t)}$ &
    Can be obtained by SDE using VF: \eqref{eq:vtrue:SDE:G}, S:\eqref{eq:strue:SDE:G}
    \\
    %\midrule
    %
    % {(OURS)} ExFM & $hello$ & $t x_1 + (1-t) x_0$ & $hi$ & \yes & \yes & \yes \\
    \bottomrule
    \end{tabular}
    }
\end{table}

\paragraph{Complexity}
We assume that the main running time of the algorithm is spent on training the model, especially if it is quite complex. Thus, the running time of one training step depends crucially on the number~$n\cdot m$ of samples~$\{x\}$ and it is approximately the same for both algorithms: the addition of points~$\overline x_1$ entails only an additional calculation using formula~\eqref{eq:main_numeric}, which can be done quickly and, moreover, can be simple parallelized.

%The main advantage of our algorithm, as shown by experiments, is significantly reducing the number of steps to achieve the same learning accuracy.

%%%%%%%%%%%%%%%%%%%%%%%%%%%%%%%%%%%%%%%%%%%%%%%%%%%%%%%%%%%%%%%%%%%%%%
%%%%%%%%%%%%%%%%%%%%%%%%%%%%%%%%%%%%%%%%%%%%%%%%%%%%%%%%%%%%%%%%%%%%%%
%%%%%%%%%%%%%%%%%%%%%%%%%%%%%%%%%%%%%%%%%%%%%%%%%%%%%%%%%%%%%%%%%%%%%%

% \newpage

\subsection{Irreducible dispersion of gradient for CFM optimization}
Ensuring the stability of optimization is vital. Let $\Delta\theta$ be changes in parameters, obtained by SGD with step size $\gamma / 2$ applied to the functional from Eq.~\eqref{eq:discrete_loss_CFM}:
\begin{equation}
\Delta v(x^{j,i},t^j) = -\gamma \cdot\bigl( v(x^{j,i},t^j) -  v^d(x^{j,i},\,t^j) \bigr)
.
\label{eq:var:up:ours}
\end{equation}
For simplification, we consider a function, $v_\theta(x,t)$, capable of perfectly fitting the CFM problem and providing an optimal solution for any point~$x$ and time~$t$. For a linear conditional flow at a specific point $x^{j,i}\sim\rho_{x_1^i}(\cdot,t^j)$ at time $t^j\sim U(0,1)$, the update $\Delta v(x^{j,i},t^j)$ can be represented as follows:
\begin{equation}
\Delta v(x^{j,i},t^j) = \gamma \left(x_1^i - \hat{x}_0^i - v(x^{j,i},t^j)\right),
\label{eq:inter_CFM_general}
\end{equation}
where $\hat{x}_0^i = \frac{x^{j,i} - t^jx_1^i}{1-t^j}$. We define the dispersion $\mathbb{D}_{x,x_1}f(x,x_1)$ for $x\sim \rho_{x_1}(\cdot,t)$ and $x_1\sim \rho_1$ as:
\begin{equation}
\mathbb D_{x,x_1}f(x,x_1) = \mathbb E_{x,x_1}f^2(x,x_1) - (\mathbb E_{x,x_1}f(x,x_1))^2
.
\label{eq:def_disp}
\end{equation}
\begin{proposition}
At the time $t = 0$, the dispersion of update in the form~\eqref{eq:inter_CFM_general} 
have the following element-wise lower bound:
\begin{equation*}
\mathbb{D}_{x^{j,i},x_1^i} \Delta v(x^{j,i},0) = \gamma^2\mathbb{D}_{x_1^i}x_1^i + \gamma^2\mathbb{D}_{x^{j,i},x_1^i}(x^{j,i}+ v(x^{j,i},0)) \geq \gamma^2\mathbb{D}_{x_1^i}x_1^i.
\end{equation*}
Equality is reached when the model $v(x^{j,i},0)$ has exact values equal to~\eqref{eq:v:t0}.
\end{proposition}
Given that the dispersion cannot be reduced with an increase in batch size, the only available option is to decrease the step size of the optimization method, \ie, reduce the learning rate slowing down the convergence. The situation is much better for the proposed loss in \eqref{eq:discrete_loss_our}. We can express the update $\Delta v(x^{j,i},t^j)$ in the case of ExFM objective as:
\begin{equation}
\Delta v(x^{j},t^j) = \gamma^2 \Bigl(\sum\limits_{k=1}^Nx_1^k\Tilde{\rho}\left(x^{j,i}|x_1^k,t^j\right) - x^{j,i} - v(x^{j,i},t^j)\Bigr),
\label{eq:inter_ExFM_general}
\end{equation}
where $x^{j,i} \sim \rho_{x_1^i}(\cdot, t^j)$, $x_1^k \sim \rho_1$ 
and $\Tilde{\rho}\left(x^{j,i}|x_1^k,t^j\right) = \rho_0\left(\frac{x^{j,i}-t^jx_1^k}{1-t^j}\right) /
\sum\limits_{k=1}^N\rho_0\left(\frac{x^{j,i}-t^jx_1^k}{1-t^j}\right)$. 
Similar to the derivations in the previous part, we can found simplified form for the dispersion of update at $t=0$.
% For $t=0$ coefficients 
% $\Tilde{\rho}\left(x^{j,i}|x_1^k,t^j\right) = \frac{1}{N}$,
% $\mathbb D_{x_1^k}\sum\limits_{k=1}^Nx_1^k = N\mathbb D_{x_1^k}x_1^k$,
% $x^{j,i}=x_0^{j,i}$ and $v(x^{j,i},t^j)$ 
% are independent of $x_1^k$.
\begin{proposition}
At the time $t = 0$, the dispersion of update from \eqref{eq:inter_ExFM_general} 
have the following element-wise lower bound:
\begin{equation*}
\mathbb D_{x^{j,i},x_1^k}\Delta v(x^{j,i},0) = \frac{\gamma^2}{N}\mathbb D_{x_1^k}x_1^k + \gamma^2\mathbb D_{x^{j,i},x_1^k}(x^{j,i}+ v(x^{j,i},0)) \geq \frac{\gamma^2}{N}\mathbb D_{x_1^k}x_1^k
.
\end{equation*}
Equality is reached when the model $v(x^{j,i},0)$ has exact values equal to~\eqref{eq:v:t0}.
\end{proposition}
In comparison to CFM, the dispersion of the update is $N$ times smaller than the dispersion of the target distribution and could be controlled without impeding convergence by adjusting the number of samples $N$. In Figure~\ref{fig:CFM_ExFM_disps}, we visually compare the dispersions of CFM and ExFM. The illustration aligns a standard normal distribution $\mathcal  N(0,I)$ with a shifted and scaled variant $\mathcal N(\mu,I \sigma^2)$. ExFM yields lower dispersion throughout the range $t\in[0,1]$. Detailed analytical calculations of the optimal velocity $v(x,t)$ 
and dispersion are provided in the Appendix~\ref{sec:G-G:A}.

\begin{figure}[!ht]
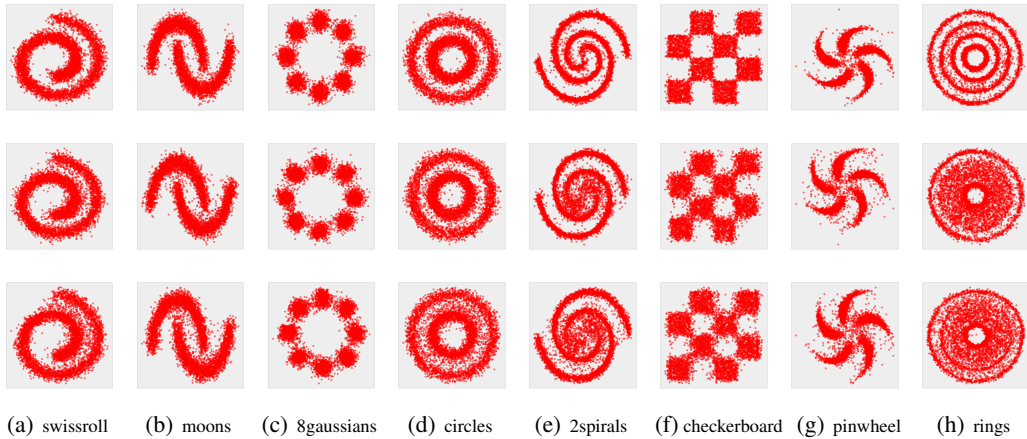
%
\def\sbsub#1#2{\includegraphics[trim=0cm 0cm 0cm 0cm,width=0.13\linewidth]{#1_#2_15000}}%
\def\sb#1#2{\subfigure[\scriptsize#2]{%
\hbox to 4.7em{\hss\vbox{
\sbsub{ExFM}{#1}\\
\sbsub{CFM}{#1}\\
\sbsub{OT-CFM}{#1}
}\hss}}%
}
\centering
\sb{swissroll}{swissroll}
\sb{moons}{moons}
\sb{8gaussians}{8gaussians}
\sb{circles}{circles}
\sb{2spirals}{2spirals}
\sb{checkerboard}{\hbox to 4.4em{\kern-0.5ex checkerboard}}
\sb{pinwheel}{\hbox to 3.0em{pinwheel}}
\sb{rings}{rings}
\caption{Visual comparison of methods on toy 2D data. First row sampled by ExFM, second row sampled by CFM, third row sampled by OT-CFM.
%Epoch size for "swissroll",  "moons", "8gaussians" is 400, for "circles" is 1000, for  "checkerboard", "pinwheel", "2spirals" is 2000, for "rings" is 5000. 
}%
\label{fig:CFM_ExFM_toy}%
\end{figure}
\begin{table}[!htb]
\centering
\def\g{\bf}
%\small
\caption{Wasserstein distance comparison for ExFM, CFM and OT-CFM methods for 2D-toy datasets for $15\,000$ learning steps ($30\,000$ learning steps for \texttt{rings} dataset) mean and std taken from 10 sampling iterations.}
\begin{sc}
\begin{tabular}{lrrr}
\toprule
   Data & ExFM & CFM & OT-CFM \\
\midrule
         swissroll    & \g 5.95e-02 $\pm$ 4.3e-03 & 8.68e-02 $\pm$ 7.3e-03 & 6.98e-02 $\pm$ 6.1e-03 \\
 moons        & \g 4.87e-02 $\pm$ 4.7e-03 & 6.80e-02 $\pm$ 8.2e-03 & 5.94e-02 $\pm$ 6.3e-03 \\
 8gaussians   & \g 8.83e-02 $\pm$ 1.41e-02 & 1.12e-01 $\pm$ 1.4e-02 & 1.00e-01 $\pm$ 1.5e-02 \\
 circles      & \g 6.70e-02 $\pm$ 3.3e-03 & 8.51e-02 $\pm$ 3.4e-03 & 8.47e-02 $\pm$ 6.9e-03 \\
 2spirals     & \g 6.94e-02 $\pm$ 9.5e-03 & 1.01e-01 $\pm$ 6e-03 & 1.08e-01 $\pm$ 2e-02 \\
 checkerboard & \g 1.14e-01 $\pm$ 1.1e-02 & 1.59e-01 $\pm$ 1.4e-02 & 1.22e-01 $\pm$ 1.5e-02 \\
 pinwheel     & \g 6.52e-02 $\pm$ 5.9e-03 & 1.13e-01 $\pm$ 1.1e-02 & 8.08e-02 $\pm$ 5.8e-03 \\
 rings        & \g 6.35e-02 $\pm$ 4.4e-03 & 1.16e-01 $\pm$ 4e-03 & 1.08e-01 $\pm$ 3e-03 \\
        
\bottomrule
\end{tabular}
\end{sc}
\label{tab:W_toy}
\end{table}

\begin{table}[!htb]
\centering
\def\g{\bf}
%\small
\caption{NLL comparison for ExFM, CFM and OT-CFM methods for tabular datasets for 10 000 learning steps, mean and std taken from 10 sampling iterations.}
\begin{sc}
%\small
%\resizebox{0.9\linewidth}{!}{%
\begin{tabular}{lrrr}
\toprule
   Data & ExFM & CFM & OT-CFM \\
\midrule
        power & \g -8.51e-02 $\pm$ 4.85e-02 & 1.64e-01 $\pm$ 4.2e-02 & 5.22e-02 $\pm$ 3.92e-02 \\ 
        gas & \g -5.53e+00 $\pm$ 4e-02 & -5.00e+00 $\pm$ 3e-02 & -5.48e+00 $\pm$ 3e-02 \\ 
        hepmass & \g 2.16e+01 $\pm$ 6e-02 & 2.21e+01 $\pm$ 6e-02 & \g 2.16e+01 $\pm$ 4e-02 \\ 
        bsds300 & -1.29e+02 $\pm$ 8e-01 & -1.29e+02 $\pm$ 9e-01 & \g -1.32e+02 $\pm$ 6e-01 \\ 
        miniboone & \g 1.34e+01 $\pm$ 2e-04 & 1.42e+01 $\pm$ 1e-04 & 1.43e+01 $\pm$ 9e-05 \\ 
        
\bottomrule
\end{tabular}
%}
\end{sc}
\label{tab:CFM_ExFM_tabular}
\end{table}

\section{Numerical Experiments}

\paragraph{Toy 2D data}
We conducted unconditional density estimation among eight distributions. 
Additional details of the experiments see in the Appendix~\ref{seq:AppAdd:numeric}.
% Ordinary Differential Equation (ODE) solver.
We commence the exposition of our findings by showcasing a series of classical 2-dimensional examples, as depicted in Fig.~\ref{fig:CFM_ExFM_toy} and Table~\ref{tab:W_toy}. Our observations indicate that ExFM adeptly handles complex distribution shapes is particularly noteworthy, especially considering its ability to do so within a small number of learning steps. Additionally, the visual comparison underscores the evident superiority of ExFM over the CFM and OT-CFM approaches. % This highlights the robustness and effectiveness of ExFM in addressing the challenges posed by complex distributions.

\paragraph{Tabular data}

We conducted unconditional density estimation on five tabular datasets, namely \texttt{power}, \texttt{gas}, \texttt{hepmass}, \texttt{minibone}, and \texttt{BSDS300}. Additional details of the experiments see in the Appendix~\ref{seq:AppAdd:numeric}.
The empirical findings obtained from the numerical experiments from Table~\ref{tab:CFM_ExFM_tabular} indicate a statistically significant improvement in the performance of our proposed method. Notably, ExFM demonstrates a notable acceleration in convergence rate.

\paragraph{High-dimensional data and additional experiments}

We conducted experiments on high-dimensional data, among them experiments on CIFAR10 and MNIST dataset. FID results on CIFAR10 shows slightly better score among sampled images.

Additional details of the experiments and sampled images see in the Appendix~\ref{seq:AppAdd:numeric}.

\paragraph{Stochastic ExFM (ExFM-S) on toy 2D data} 
We evaluated the performance of the stochastic version of ExFM (ExFM-S) with use of expressions given in Sec.~\ref{seq:SDE:samples} 
on four standard toy datasets.
%: $\mathcal{N}\rightarrow$ moons, $\mathcal{N}\rightarrow$ 8gaussians, $\mathcal{N}\rightarrow$ 2spirals, and moons $\rightarrow$ 8 gaussians.
The primary experimental setup follows that used in \cite{tong2024improving}. Additional details on the hyperparameters used are available in Appendix~\ref{seq:AppAdd:numeric}. Based on the findings presented in Table~\ref{tab:ExFM_S_comparison}, we determine that ExFM-S surpasses I-CFM on all four datasets in terms of generative performance ($\mathcal{W}_2$) and also outperforms in terms of OT optimality (NPE) on two of them, exhibiting similar results on the remaining datasets. It also demonstrates performance similar to OT-CFM. While ExFM-S is not as robust as the basic ExFM, it enables the matching of one dataset to another (moons $\rightarrow$ 8gaussians) as it does not necessitate the presence of an explicit formula for $\rho_0$.
Among other things, this experiment demonstrates the feasibility of our methods when both distributions $\rho_0$ and $\rho_1$ are unknown. % and we are given only samples of them.
\begin{table}[!tbh]
\caption{ExFM-S evaluation on four toy datasets ($\mu \pm \sigma$ over three seeds). For comparison we take I-CFM, OT-CFM, and ExFM (no values for moons $\rightarrow$ 8gaussians due to the absence of explicit formula for $\rho_0$). Performance in generative modeling ($\mathcal{W}_2$) and dynamic OT optimality (NPE) is assessed. The best result for each metric is highlighted in bold. Instances where we outperform CFM are underscored.}
\resizebox{\linewidth}{!}{
\begin{tabular}{ccccccccc}
\hline
Metric $\rightarrow$                         & \multicolumn{4}{c}{$\mathcal{W}_2 \downarrow$}                                                                                                                                                              & \multicolumn{4}{c}{NPE $\downarrow$}                                                                                                                                                       \\ \cline{2-9} 
Algorithm $\downarrow$ Dataset $\rightarrow$ & $\mathcal{N}\rightarrow$ moons                & $\mathcal{N}\rightarrow$ 8gaussians            & moons $\rightarrow$ 8gaussians                 & $\mathcal{N}\rightarrow$ 2spirals              & $\mathcal{N}\rightarrow$ moons       & $\mathcal{N}\rightarrow$ 8gaussians            & moons $\rightarrow$ 8gaussians       & $\mathcal{N}\rightarrow$ 2spirals              \\ \hline
I-CFM                                          & $0.522 \pm 0.015$                             & $0.647 \pm 0.078$                              & $0.966 \pm 0.21$                               & $1.662 \pm 0.067$                              & $0.328 \pm 0.051$                    & $0.209 \pm 0.009$                              & $0.945 \pm 0.025$                    & $0.098 \pm 0.04$                               \\
OT-CFM                                       & $0.427 \pm 0.038$                             & $0.528 \pm 0.053$                              &\bf 0.569 $\pm$ 0.018                              & $1.322 \pm 0.052$                              & \bf 0.065 $\pm$ 0.068 & \bf 0.031 $\pm$ 0.018           & \bf 0.074 $\pm$ 0.026 & \bf 0.031 $\pm$ 0.02            \\
ExFM                                        & \bf 0.318 $\pm$ 0.010          & \bf 0.445 $\pm$ 0.075           & --                                             & \bf 1.276 $\pm$ 0.043           & $0.382 \pm 0.050$                    & $0.213 \pm 0.023$                              & --                                   & \underline{0.069 $\pm$ 0.064} \\
ExFM-S                                      & \underline{0.486 $\pm$ 0.09} & \underline{0.570 $\pm$ 0.053} & \underline{0.728 $\pm$ 0.063} & \underline{1.361 $\pm$ 0.181} & $0.35 \pm 0.143$                     & \underline{0.166 $\pm$ 0.039} & $0.946 \pm 0.059$                    & \underline{0.083 $\pm$ 0.059} \\ \hline
\end{tabular}}

\label{tab:ExFM_S_comparison}
\end{table}

%\FloatBarrier

\section{Conclusions}
%In conclusion, 
The presented method introduces a new loss function in tractable form (in terms of integrals) that 
improves upon the existing Conditional Flow Matching approach. 
%The gradient of the ExFM loss on the parameter of the model is the same as gradient of the usual CFM loss.
%Thus, the argument minimums (vector field) of the considered two losses are the same. 
%But 
New loss as a function of the model parameters, reaches zero at its minimum.
Thanks to this, we can:
%\begin{itemize}
%\item
a) write an explicit expression for the vector field on which the loss minimum is achieved;
%\item
b) get a smaller variance when training on the discrete version of the loss, 
therefore, we can learn the model faster and more accurately.
%\end{itemize}
%This new loss function allows for an explicit expression of the vector field delivering the minimum, which can be analytically derived for linear mapping. The algorithm demonstrates better training variance and achieves better learning results compared to traditional model training approaches. By using a smaller learning rate, the dispersion in gradients for powerful models with a large number of parameters can be reduced. 

Numerical experiments conducted on toy 2D data show reliable outcomes under uniform conditions and parameters. 
Comparison of the absolute values of loss for the proposed method and for CFM for the same distributions show that the absolute values of loss for these models differ strikingly, by a factor of $10^2$--$10^3$.
Experiments on high-dimensional datasets also confirm the theoretical deductions about the variance reduction of our method. 
However, we emphasize that we do not expect to use the proposed method in its pure form. 
On the contrary, we expect that the theoretical implications of our formulas will contribute to the construction of better learning or inference algorithms in conjunction with other heuristics or methods.

% Additionally,
Algebraic analysis of variance for some cases (in particular, for the case~$t=0$ or for the case of two Gaussians as initial and final distributions) show an improvement in variance when using the new loss. However, it is rather difficult to analyze in the general case, for all times $t$ and general distributions~$\rho_0$ and~$\rho_1$.

Having the expression for the vector field and score in the form of integrals, we can explicitly write out their expressions for some simple cases; in the case of Gaussian distributions we can also write out the exact solution for the trajectories.
Thus, our approach allows one to advance the theoretical study of FM-based and Diffusion Model-based frameworks.

%As future works we point out the theoretical use of the explicit formula for the vector field 
%in order to study the properties of Flow Matching and its modifications, 
%as well as the invention of new numerical schemes using this formula.

\printbibliography

%\newpage
%\input{checklist}
%%%%%%%%%%%%%%%%%%%%%%%%%%%%%%%%%%%%%%%%%%%%%%%%%%%%%%%%%%%%

\newpage

\appendix

\section{Proof of the theorems}
\subsection{Proof of the Theorem~\ref{th:grads_eq}}
\label{seq:App:diff}

\begin{proof}
    We need to proof, that
$
    \dv{\LCFM(\theta)}{\theta}
    =
    \dv{\Lour(\theta)}{\theta}
$.

% Let us show that $\LCFM$ is equal to $\Lour$ up to some constant. From \eqref{eq:LCM_ind} $\LCFM$ could be rewritten in following format:
To establish the equivalence of $\LCFM$ and $\Lour$ up to a constant term, we begin by expressing $\LCFM$ in the format specified by equation \eqref{eq:LCM_ind}:
\begin{equation*}
\LCFM =
    \mathbb E_{t,x_1,x\sim\rho_m(\cdot,t)}
    \norm{
        \model(x,t)-
        w(t,x_1,x)
    }^2
    \times
    \rho_{c}(x|x_1,t)/\rho_1(x_1).
\end{equation*}
Utilizing the bilinearity of the 2-norm,
we can rewrite $\LCFM$ as:
\begin{multline}
\label{eq:App:LCM}
\LCFM = 
    \mathbb E_{t,x_1,x\sim\rho_m(\cdot,t)}
    \frac{\norm{
        \model(x,t)
    }^2 \rho_{c}(x|x_1,t)}{\rho_1(x_1)}
    -{}
    \\
    2\mathbb E_{t,x_1,x\sim\rho_m(\cdot,t)}
    %\langle
    \frac{\model(x,t)^T \cdot w(t,x_1,x)
    %\rangle
    \rho_{c}(x|x_1,t)}{\rho_1(x_1)} +
    C.
\end{multline}
Here,
$T$ denotes transposed vector,
dot denotes scalar product,
$C$ represents a constant independent of $\theta$. 

Noting that $\mathbb E_{x_1}\rho_{c}(x|x_1,t)/\rho_1(x_1) = 1$:
\begin{equation*}
\mathbb E_{x_1}\frac{\rho_{c}(x|x_1,t)}{\rho_1(x_1)} =
\int\frac{\rho_{x_1}(x,t)\rho_1(x_1) \dd{x_1}}{
        \int
        \rho_{x_1}(\xt,t)
        \rho_1(x_1)
        \dd {x_1}
    } = 1,
\end{equation*}
we can simplify the first term 
in the expansion~\eqref{eq:App:LCM}:
%\mathbb E_{t,x_1,x\sim\rho_m(\cdot,t)}
%    \norm{
%        \model(x,t)
%    }^2 \rho_{c}(x|x_1,t)$ as:
\begin{multline}
\mathbb E_{t,x_1,x\sim\rho_m(\cdot,t)}
\frac{
    \norm{
        \model(x,t)
    }^2  \rho_{c}(x|x_1,t)}{\rho_1(x_1)}
    = \\
    E_{t,x\sim\rho_m(\cdot,t)}\norm{
        \model(x,t)
    }^2 
    \,\mathbb E_{x_1}\frac{\rho_{c}(x|x_1,t)}{\rho_1(x_1)}
    = 
    E_{t,x\sim\rho_m(\cdot,t)}\norm{
        \model(x,t)
    }^2.
\label{eq:App:LCM:simpl}
\end{multline}

For our loss~\Lour in the form~\eqref{eq:Lour}
we also use the bilinearity of the norm:
%Now, by utilizing the bilinearity and formula, 
%we can rewrite the expression for our loss $\Lour$ as:
\begin{equation}
\Lour=
\mathbb E_{t,x\sim\rho_m(\cdot,t)}
    \norm{
        \model(x,t)
    }^2 -{}
    2\mathbb E_{t,x\sim\rho_m(\cdot,t)}
    \mathbb E_{x_1}
    %\langle
    \frac{\model(x,t)^T \cdot w(t,x_1,x)
    %\rangle
    \rho_{c}(x|x_1,t)}{\rho_1(x_1)} +
    C.
\end{equation}

Comparing the last expression and the~Eq.~\eqref{eq:App:LCM} 
with the modification~\eqref{eq:App:LCM:simpl}
and also taking into account 
the independence of random variables~$x$ and~$x_1$,
we come to the conclusion
%Ultimately, 
%we can conclude 
that $\Lour$ is equal to $\LCFM$ up to some constant
independent of the model parameters.

\end{proof}

\subsection{Sketch of the proof of the Theorem~\ref{th:var_lower}}
\label{seq:App:skech}
\begin{proof}
We need to prove that $\mathbb D\dv{\Lour^d(\theta)}{\theta} \leq\mathbb D\dv{\LCFM^d(\theta)}{\theta}$, where $\Lour^d(\theta)$ and $\LCFM^d(\theta)$ discrete loss functions presented in \eqref{eq:discrete_loss_our} and \eqref{eq:discrete_loss_CFM}. Firstly, let us rewrite the derivative of loss functions using the bilinearity:
%and fact that $w(t,x_1,x)$ do not depend on $\theta$:
\begin{equation*}
\dv{\Lour^d(\theta)}{\theta} = 
%%%%
2\sum\limits_{i,j}
%\biggl(
%\Bigl\langle 
\left(
\dv{\model(x^{j,i},t^j)}{\theta}
\right)^T
\cdot
\bigl(\model(x^{j,i},t^j) - 
%\sum\limits_{k} w(t^j,x_1^k,x^{j,i}) \rho_{c}(x^{j,i}|x_1^k,t^j
v^d(x^{j,i},t^j)
%\Bigr\rangle
\bigr)
.
\end{equation*}

Note that in this expression, values~$x^{j,i}$ as well as~$t^j$,
which are included in the argument of the function~$v$, 
are fixed (our goal to calculate the variance with fixed model arguments).
%and we do not consider it a random variable 
%since it is sampled outside the norm operator. 
Thus, we need to consider the variance of the remaining expression 
arising from the randomness of~$\ox^k$.

%Consider the the case of the independent 
%variables~$x_0$ and~$x_1$
%and such a density~$\rho_0(x_0)$, that $\rho_0(0)>\rho_0(y)>0$
%$\forall y\neq 0$.
Recall (below we will omit the indices at variables~$x$ and~$t$),
$$
\def\updown#1{\rho_0\bigl(\phicex^{-1}t{\ox^k}(#1)\bigr)}%
v^d(x,\,t)=
    \frac{\sum_{k=1}^N
        w(t,\ox^k,x)\cdot\updown{x}}%
    {\sum_{k=1}^N\updown{x}}
    .
$$
Note, that if~$N=1$,
(\ie 
we do not sample any additional points 
other than the ones we have already sampled) 
this expression is exactly the same as the derivative of the common discretized CFM loss~$\dv{\LCFM^d(\theta)}{\theta}$.

Moreover, 
recall that one of the points (without loss of generality, 
we can assume that its index is 1) $\ox^1$ 
is added from the set from which point~$x$ was derived:
$x=\phicex{t}{\ox^1}(x_0)$.
(Here $x_0$ is the paired point to~$\ox^1$)

Thus, we can rewrite expression for~$v^d$:
\begin{equation}
    \label{eq:main_numeric:general}
v^d(x,\,t)
=
\def\uandd#1{\rho_0\left(
\phicex^{-1}{t}{\ox^#1}(x)
\right)
}
    \frac{w(t,\ox^1,x)\rho_0(x_0)  + \sum_{k=2}^N
       w(t,\ox^k,x)
       \cdot
        \uandd k
    }{\rho_0(x_0) + \sum_{k=2}^N\uandd k
    }    
    .
\end{equation}
Thus, our task was reduced to evaluating 
how well the additional terms (for~$k$ starting from~$2$) 
improve approximate 
of the original integrals that are in loss~\eqref{eq:Lour}.

So, we need to estimate the following dispersion ratio,
where in the numerator is the variance of discrete loss CFM, 
and in the denominator~--- the variance of loss ExFM:
$$
k_D=
\frac{
\mathbb D\bigl(v_\theta(x, t) - w(t,\ox^1,x) \bigr)
}{
\mathbb D\left(v_\theta(x, t) - 
\def\updown#1{\rho_0\bigl(\phicex^{-1}t{\ox^k}(#1)\bigr)}%
    \frac{\sum_{k=1}^N
        w(t,\ox^k,x)\cdot\updown{x}}%
    {\sum_{k=1}^N\updown{x}}
\right)
}
$$
The smaller coefficient~$k_D$ is, the better the proposed loss ExFM works.

Formally, we can write our problem as an importance sampling problem for the following integral:
$$
I=\int f(x)p(x)\dd{x}.
$$
This integral we estimate by sample mean of the following expectation
over some random variable with density function~$q(x)$:
$$
I=\mathbb E_{x\sim q}\bigl(w(x)f(x)\bigr)
$$
with
$$
w(x)=\frac{p(x)}{q(x)}.
$$
We replace the exact value of~$I$ with the value
\def\Ii{\overline I}
$$
\Ii=\frac{\sum_{k=1}^N w(\ox^i)f(\ox^k)}{\sum_{i=k}^N w(\ox^k)}.
$$
It follows from the strong law of large numbers that in the limit~$N\to\infty$,
$I\to\Ii$ almost surely.
From the central limit theorem we can find the asymptotic variance:
\begin{equation}
\mathbb D\Ii=
\frac1N
\mathbb E_{x\sim q} \bigl(w^2(x)(f(x)-I)^2\bigr).
\label{eq:App:DI}
\end{equation}

In our case (loss \Lour), we have 
$q(x_1)=\rho_1(x_1)$,
$f(x_1)=w(t,x_1,x)$ and
$w(x_1)=\rho_0\left(
\phicex^{-1}{t}{x_1}(x)
\right)$.

Despite the fact that the equation~\eqref{eq:App:DI} for the variance contains~$N$ in the denominator, 
it is rather difficult to give an estimate of its behavior in general. 
The point is that this formula is well suited for the case when~$w$ in it is of approximately the same order.
In the considered case, this is achieved at times~$t$ noticeably less than~$1$.

But in the case, when~$t$ is closed to~$1$ we have, 
for example, for the linear map, 
that 
$$
w(x_1)=\rho_0\left(
\phicex^{-1}{t}{x_1}(x)
\right)
=
\rho_0\left(
\frac{x-x_1t}{1-t}
\right)
$$
and this function has a sharp peak near the point~$x/t$ if it is considered as a function of~$x_1$.
Thus, at such values of~$t$, only a small number of summands will give a sufficient contribution to the sum
compared to the first term.

Finally,
inequality $k_D<1$ is formally fulfilled, but how much~$k_D$ is less 
than one depends on many factors.

\end{proof}

%First, let us estimate the dispersion of the common loss~\LCFM in the form of the expression~\eqref{eq:LCMour_1},
%where variables $x$ and $x_1$ are independent,
%namely, the dispersion of the following expression
%$$
%\mathbb D_{x_1}\bigl(
%\model(x,t) - 
%w(x,x_1,t)\bigr)=
%\mathbb D_{x_1}
%w(x,x_1,t).
%$$
%\end{proof}
%In the case of the linear conditional map~$\phi$
%this expression reads
%\begin{multline}
%\mathbb D_{x_1}
%w(x,x_1,t)=
%\mathbb D_{x_1}\bigl(
%\frac{x_1-x}{1-t}
%\bigr)
%=\\
%\frac1{(1-t)^2}\mathbb D_{x_1}(x_1-x)
%=
%\frac1{(1-t)^2}\mathbb D_{x_1}
%=\\
%\frac1{(1-t)^2}\int\left(x_1-\int \xi\rho_1(\xi)\dd \xi\right)^2\rho_1(x_1)\dd x_1
%.
%\end{multline}

\subsection{Expressions for the regularized map}
To justify the expression~\eqref{eq:vtrue},
we use a invertable 
transformation and then strictly take the limit~$\sigma_s\to0$.

Expression Eq.~\eqref{eq:vtrue}, \eqref{eq:main_numeric}
are obtained for the simple map 
$\phic(x_0)=(1-t)x_0+tx_1$
which is not invertable at~$t=1$.
For the map with small regularizing parameter~$\sigma_s>0$
$\phic(x_0)=(1-t)x_0+tx_1+\sigma_sx_0$, which is invertable at all time values $0\leq t\leq1$,
Eq.~\eqref{eq:vtrue}, \eqref{eq:main_numeric} needs modifications.
Namely, for this map the following exact formulas holds true
\begin{equation}
\!v(x,t)
=\!\!\int \!\!w(t,x_1,x)
\rho_{c}(x|x_1,t)
\rho_1(x_1)
\dd{x_1}\!=\!
\dfrac{
\int
\bigl(x_1-x(1-\sigma_s)\bigr)
\rho_0\left(
\frac{x-x_1 t}{1+\sigma_st-t}
\right)
%\frac1{1-t}
\rho_1(x_1)
\dd{x_1}}{
( 1+\sigma_st-t )
\int
\rho_0\left(
\frac{x-x_1 t}{1+\sigma_st -t}
\right)
\rho_1(x_1)
\dd{x_1}}{
}
.
\label{eq:vtrue:sigma}
\end{equation}
By direct substitution we make sure that for this vector field
\begin{equation}
v(x,\,0)=
%\mathbb E_{x_1}x_1 - 
%\frac x{1+\sigma_s}
%=
\int x_1\rho_1(x_1)\dd{x_1}-
%\frac x{1+\sigma_s}
x(1-\sigma_s)
\label{eq:v:t0sigma}
\end{equation}
and
%\begin{multline}
\begin{equation}
v(x,\,1)=
\dfrac{
\int
(x-y)
\rho_0(
y
)
%\frac1{1-t}
\rho_1(x-y\sigma_s)
\dd{y}}{
\int
\rho_0(y)
\rho_1(x-y\sigma_s)
\dd{y}}{
}
,
\label{eq:v:t1sigma}
\end{equation}
%\end{multline}
where we perform
change of the variables $y\gets\frac{x_1-x}{\sigma_st}$.

\subsubsection{Prof of the explicit formula~\texorpdfstring{\eqref{eq:vtrue}}{(12)} for the vector field}
\label{sec:App:expl}
\begin{assumption}
\label{ass:rho_1}
Density $\rho_1$ is continuous at any point~$x\in(-\infty,\,\infty)$.
\end{assumption}
\begin{theorem}
In equations~\eqref{eq:vtrue:sigma}, \eqref{eq:v:t0sigma} and~\eqref{eq:v:t1sigma}
we can take the limit~$\sigma_s\to0$
under integrals
to get Eq.~\eqref{eq:vtrue} and~\eqref{eq:v:t0}.
\end{theorem}
\begin{proof}
Assuming that the distribution~$\rho_1$ has a finite first moment: $|\int \xi\rho_1(\xi)\dd{\xi}|<C_1$ 
and that the density of~$\rho_0$ is bounded: $\rho_0(x)<C_2$, $\forall x\in(-\infty, \infty)$, 
we obtain that 
the 
%integrals 
integrand functions
in the numerator and denominator
in the Eq.~\eqref{eq:vtrue:sigma}
can be bounded  
by the following integrable 
%integrals of the following 
functions 
independent of~$\sigma_s$ and~$t$:
$$
\rho_0\left(
\frac{x-x_1 t}{1+\sigma_s-t}
\right)
%\frac1{1-t}
\rho_1(x_1)
<
C_1\rho_1(x_1)
$$
and
$$
\begin{aligned}
0&\leq
x_1
\rho_0\left(
\frac{x-x_1 t}{1+\sigma_st-t}
\right)
\rho_1(x_1)
%\right|
<
%\bigl|
x_1
C_1\rho_1(x_1),&
x&\geq 0,
%\bigr|
\\
0&>
%\left|
x_1
\rho_0\left(
\frac{x-x_1 t}{1+\sigma_st-t}
\right)
%\frac1{1-t}
\rho_1(x_1)
%\right|
>
%\bigl|
x_1
C_1\rho_1(x_1),
%\bigr|\\
&
x&<0.
\end{aligned}
$$
It follows that both integrals in expression~\eqref{eq:vtrue:sigma} converge absolutely and uniformly.
So,
we can swap the operations of taking the limit and integration, 
and 
we can take the limit $\sigma_s\to0$ 
in the integrand
for any time~$t\in[0,\,t_0]$ for arbitrary~$t_0<1$.

Now, let us consider the case $t=1$.
From Assumption~\ref{ass:rho_1}
the boundedness of the density~$\rho_1$
follows: $\rho_1(x)<C_2$, $\forall x\in(-\infty,\,\infty)$.
Thus, 
integrand functions
in the numerator and denominator
in the Eq.~\eqref{eq:v:t1sigma}
can be bounded  
by the following integrable 
functions 
independent of~$\sigma_s$:
$$
\rho_0(y)
\rho_1(x-y\sigma_s)
<
\rho_0(y)
C_2
$$
and
$$
\begin{aligned}
0&\leq
y
\rho_0(y)
\rho_1(x-y\sigma_s)
<
y 
C_2\rho_0(y),&
y&\geq0,\\
0&>
y
\rho_0(y)
\rho_1(x-y\sigma_s)
>
y 
C_2\rho_0(y),&
y&<0.
\end{aligned}
$$
The existence of the limit 
$$
\lim_{\sigma_s\to0}\rho_1(x-y\sigma_s)
=\rho_1(x),
$$
follows from Assumption~\ref{ass:rho_1}.

Finally, 
we conclude that formula~\eqref{eq:vtrue},
regarded as the limit~$\sigma_s\to0$ of the~\eqref{eq:vtrue:sigma}
at any~$t\in[0,\,1]$, is true.
\end{proof}

% To complete the proof, we first need to 
\begin{theorem}
The vector field in Eq.~\eqref{eq:vtrue}
delivers minimum to the Flow Matching objective (see the work~\cite{lipman2023flow}),
$$
\mathbb E_t\mathbb E_{x\sim\rho(x,t)}\norm{\overline v(x,t)-v(x,t)},
$$
where $\rho(x,t)$ and $\overline v(x,t)$ 
satisfy the equation~\eqref{eq:FP}
% are solution of the Eq.
with the given densities~$\rho_0$ and~$\rho_1$.
\end{theorem}
\begin{proof}
The proof is based on the previous statements 
and on a Theorem~1 from~\cite{lipman2023flow} 
(that the marginal vector field based on conditional vector fields 
generates the marginal probability path based on conditional  probability paths.

To complete the proof,
we must justify that, with $\sigma_s$ tending to zero, 
the marginal path at $t=1$ coincides with a given probability~$\rho_1$.

Consider the marginal probability path~$p_t(x,t)$
\begin{equation}
p_t(x,t)=
\int p_t(x|x_1,\sigma_s)\rho_1(x_1)\dd x_1
\label{eq:App:someproof}
\end{equation}
where $p_t(x|x_1,\sigma_s)$ is 
conditional  probability paths obtained by regularized linear conditional map.
Distribution $p_t$  in the time~$t=0$ is equal to standard normal distribution $p_0(x|x_1,\sigma_s)=\mathcal N(x\mid 0, 1)$
and at the time~$t=1$ it is a stretched Gaussian centered at~$x_1$: $p_1(x|x_1,\sigma_s)=\mathcal N(x\mid x_1, \sigma_s I)$.

Substituting~$p_1$ into the Eq.~\eqref{eq:App:someproof} and 
considering that there exists a limit~$\sigma_s\to0$ due to Assumption~\ref{ass:rho_1},
we obtain
$$
p_1(x)=
\lim_{\sigma_s\to0}
\int p_t(x|x_1,\sigma_s)\rho_1(x_1)\dd x_1=
\rho_1(x_1).
$$
This finish the proof.
\end{proof}

\subsubsection{Learning procedure for \texorpdfstring{$\sigma_s>0$}{ss>0}}
Using standard normal distribution as initial density $\rho_0$,
and the regularized map
$\phic(x_0)=(1-t)x_0+tx_1+\sigma_stx_0$
we obtain the following
approximation formula
\begin{equation*}
v^d(x,\,t)
    =
    %\\
 %arg = - 0.5*(((x - t*x1)/(1-t))**2).sum(axis=-1)
    \frac{\sum_{k=1}^N
   %\left( 
       \frac{\ox^k - x(1-\sigma_s)}{1-t(1-\sigma_s)} 
        \exp(Y^k)
   %\right)
    }{\sum_{k=1}^N\exp(Y^k)},
    \qq{where}
    Y^k = -\frac12
    %\sum_{\alpha=1}^d \left( \frac{ (x- t \cdot \ox^k)_{(\alpha)}}{1-t}  \right)^2.
    \frac{\norm{x- t \cdot \ox^k}^2_{\mathbb R^d}}{1-t(1-\sigma_s)}.
    \label{eq:main_numeric:sigma}
\end{equation*}
In practical applications, 
the exponent calculation is replaced by the $\softmax$ function calculation, 
which is more stable.

\section{Estimation of integrals}
\label{sec:Int-est}
In general,
we need to estimate the following expression
\def\underint{f(x_1,\,\eta)\rho_1(x_1)
\dd{x_1}}
\begin{equation*}
I(\eta)=
    \frac{
    \int
w(x_1,\,\eta)
\underint
    }{
    \int
\underint    
    }.
\end{equation*}
In particular,
substituting $\eta\to\{x,t\}$,
$w(x,\eta)\to(x_1-x)/(1-t)$
we obtain formula~\eqref{eq:vtrue} 
and similar ones with similar substitutions.

If we can sample from the $\rho_1$
distribution, we can estimate this integral in two ways: \emph{self-normalized importance sampling} and \emph{rejection sampling}.

Let
$\mathcal X=\{x_1^k\}_{k=1}^N$ 
be $N$  samples from the distribution~$\rho_1$.

\paragraph{Self-normalized Importance Sampling}
In this case
  \def\underint{f(x_1^k,\,\eta)\rho_1(x_1^k)
\dd{x_1}}
\begin{equation}
I(\eta)\approx
    \frac{
    \sum\limits_{k=1}^N
w(x_1^k,\eta)
\underint
    }{
    \sum\limits_{k=1}^N
\underint    
    }.
\label{eq:A:int_est}
\end{equation}  
This estimate is biased in theory, but there several methods to reduce this bias and improve this estimate, see, for example, 
\cite{NEURIPS2022_04bd683d}.
Our numerical experiments generally show that the estimation~\eqref{eq:A:int_est} in the form is already sufficient for stable results; we don not observe any bias.

\paragraph{Rejection sampling}
Let
$\mathcal Y=\{y^k\}_{k=1}^{M}\subset\mathcal X$ 
be a subset of the the initially given set of samples,
which is formed according to the following rule. 
Let $C=\sup_x\rho_1(x)$.
For a given sample~$x_1^j$
we generate a random uniformly distributed variable~$\xi_j\sim\mathcal U(0,1)$
and if
$$
f(x_1^j) \geq C\xi_j,
$$
then we put the point $x_k^j$
to the set $\mathcal Y$; otherwise we reject it.

Having formed the set $\mathcal Y$,
we evaluate the integral as 
\begin{equation*}
    I(\eta)
\approx
\frac1M
\sum_{k=1}^M
w(y^k,\eta).
\end{equation*}
To justify
the last estimation,
we note, that the 
points from the set~$\mathcal Y$
are distributed
according to (non-normalized)
density~$\rho(x)f(x,\eta)\rho_1(x)$.
One can show it using the proof
of the rejection sampling method.
This is the same density as in Eq.~\eqref{eq:cond_distr}
and thus we estimate the expression~\eqref{eq:vtru_comm}
using Important Sampling 
without any additional denominator.

\paragraph{Comparison}

When we apply these techniques to evaluating the expression for the vector field, 
we know that when the time parameter~$t$ is close to~$1$, 
the function~$f(x_1,\eta)$ (which is a scaled $\rho_0$) has a peak at the point $x=x_1$.
This means that only a small number of points from the original set will end up in the set~$\mathcal Y$. 
Moreover, in the case when the time~$t$ is very close to one and the data are well separated,
only one point~$x_1$ will end up in~$\mathcal Y$. 
This explains why we initially put this point in the set~$\mathcal X$, 
because otherwise it would be possible that the set~$\mathcal Y$ is empty and~$M=0$.

As a future work, 
we indicate a theoretical finding of the probability of hitting a particular point~$x_1$ 
in the set~$\mathcal Y$ and, 
thus, 
a modification of our algorithm, when the sample~$x_1$
will not always go to the set~$\mathcal X$,
but with some probability --- the greater the~$t$ the closer this probability to~$1$.

\section{The main Algorithm and extensions and generalization of the exact expression}
\label{sec:ex:A}
\begin{algorithm}[t!bh]
    \caption{Vector field model training algorithm}
    \label{alg:main}
	\begin{algorithmic}[1]
        \REQUIRE{Sampler from distribution~$\rho_1$ (or a set of samples);
            parameters $n$ and $m$ (number of spatial and time points, correspondingly);
            parameter~$N$ (number of averaging point);
            model~$v_\theta(x,t)$;
            algorithm with parameters for SGD
        }
        \ENSURE{quasi-optimal parameters~$\theta$ for the trained model}
        \STATE{Initialize $\theta$ (maybe random)}
        \WHILE{exit condition is not met}
        \STATE{Sample $m$ points $\{t^j\}$ from  $\mathcal U[0,1]$}
        \STATE{Sample $n$  points pairs  $\{x_0^i,x_1^i\}_{i=1}^n$ from joint distribution $\pi$ 
        ($\pi(x_0,x_1)=\rho_0(x_0)\rho_1(x_1)$ if variables are independent) }
        %\STATE{Sample $n$ points $\{x_1^i\}$ from  $\rho_1$}
        \STATE{Sample $N-n$ points $\{\hat x_1^l\}$ from  $\rho_1$
and form $\{\ox^k\}=\{x_1^i\}\cup\{\hat x_1^l\}$
\comm{We can take all available samples as $\{\ox^k\}$
if we don't have access to a sampler, but only ready-made samples.}
        }
        \STATE{For all $i$ and $j$ calculate the sum at the right side of~\eqref{eq:discrete_loss_our}
        (using~\eqref{eq:main_numeric} if $\rho_0$ is standard Gaussian
        or~\eqref{eq:main_numeric:general} in general)}
        \STATE{Calculate the sum on $i$ and $j$  in discrete loss~\eqref{eq:discrete_loss_our},
            and take backward derivative,
            obtaining approximate grad~$G\approx\grad_\theta\Lour$ of loss~$\Lour$ on model parameters~$\theta$.
        }
        \STATE{Update model parameters~$\theta\gets SGD(\theta, G)$}
        \ENDWHILE
        %\RETURN{adsfaf}
	\end{algorithmic}
\end{algorithm}

General form of the proposed Algorithm is given in Alg~\ref{alg:main}.

When using other maps, formula~\eqref{eq:vtrue} is modified accordingly.
For example, if we use the regularized map
$\phic(x_0)=(1-t)x_0+tx_1+\sigma_stx_0$,
we get the formula~\eqref{eq:vtrue:sigma}.% given in Appendix.
Note, that in this case 
the final density~$\rho(x,\,1)$,
obtained from the continuity equation
%considered as an equation on
%\begin{equation*}
%\left\{
%\begin{aligned}
%    \pdv{\rho(x,t)}{t}
%    &= -\mathrm{div} (\rho(x,t) \overline v(x,t)), \\
%    \rho(x,0) &= \rho_0(x), \\
%\end{aligned}
%\right.
%\end{equation*}
is not equal to~$\rho_1$, but is its smoothed modification.

When using a different initial density~$\rho_0$ (not the normal distribution), 
an obvious modification will be made to formula~\eqref{eq:main_numeric}.

\paragraph{Diffusion-like models}
We can treat so-called Variance Preserving \cite{ho2020denoising} model as
CFM with the map
$$
\phic(x)=
\alpha_{1-t}x+\sqrt{1-\alpha^2_{1-t}}x_1.
$$
and $\rho_0$ as standard normal distribution: $\rho_0=\normald01$ 
In this case, the common expression~\eqref{eq:vtru_comm} for vector filed
transforms to
\begin{equation}
v(x,t)=
\frac{
\int
(x\alpha_{1-t}-x_1)\alpha'_{1-t}\,
\rho_0\left(
\frac{x-x_1 \alpha_{1-t}}{\sqrt{1-\alpha^2_{1-t}}}
\right)
%\frac1{1-t}
\rho_1(x_1)
\dd{x_1}}{
( 1- \alpha^2_{1-t} )
\int
\rho_0\left(
\frac{x-x_1 \alpha_{1-t}}{\sqrt{1-\alpha^2_{1-t}}}
\right)
\rho_1(x_1)
\dd{x_1}
}{
}
,
    \label{eq:vtrue:VP}
\end{equation}
where $\alpha'_s=\dv{\alpha_s}{s}$.

Similarity
we can treat so-called Variance Exploding \cite{song_generative_2019} model as
CFM with the map
$$
\phic(x)=
\sigma_{1-t}x+x_1.
$$
and $\rho_0$ also as standard normal distribution: $\rho_0=\normald01$ 
In this case, the common expression~\eqref{eq:vtru_comm} for vector filed
transforms to
\begin{equation}
v(x,t)=
\frac{
\int
(x_1-x)\sigma'_{1-t}\,
\rho_0\left(
\frac{x-x_1 }{\sigma_{1-t}}
\right)
%\frac1{1-t}
\rho_1(x_1)
\dd{x_1}}{
\sigma_{1-t}
\int
\rho_0\left(
\frac{x-x_1 }{\sigma_{1-t}}
\right)
\rho_1(x_1)
\dd{x_1}
}{
}
,
\label{eq:vtrue:VE}
\end{equation}
where $\sigma'_s=\dv{\sigma_s}{s}$.

\paragraph{Joint Distribution}
Moreover, 
in addition to the independent densities $x_0\sim\rho_0$ and $x_1\sim\rho_1$, 
we can use the joint density $\{x_0,\,x_1\}\sim\pi(x_0,\,x_1)$.
In the papers~\cite{Tong2024,tong2024improving},
optimal transport (OT) and Schrödinger's bridge are taken as~$\pi$.
In this case the expression for the vector field changes insignificantly:
the conditional probability~$\rho_c$ from Eq.~\eqref{eq:cond_distr}
is subject to change:
%(again, for simplicity we assume, that $\det[\pdv{\phicex^{-1}t{\ox^k}(x)}{x}]$
%do not depend on~$x_1$)
\begin{equation}
    \rho_{c}(x|x_1,t)
    = 
    \frac{\pi\left(\phic^{-1}(x),x_1\right)
    \det[\pdv{\phic^{-1}(x)}{x}]
    }{
        \int
        \pi\left(\phic^{-1}(x),x_1\right)
        \det[\pdv{\phic^{-1}(x)}{x}]
        \dd {x_1}
       % \,\rho_1(x_1)
    }
    .
\label{eq:cond:general}
\end{equation}
%Here, $\rho_1(x_1)=\int\pi(y,\,x_1)\dd y$.
Then, Eq.~\eqref{eq:vtru_comm}
remains the same in general case.
In the case of linear~$\phi$, 
the extension of Eq.~\eqref{eq:vtrue}
%have the following form:
reads
\begin{equation}
v(x,t)
=
\dfrac{
\int
(x_1-x)\,
\pi\!\left(
\phic^{-1}(x)
,x_1\right)
\det[\pdv{\phic^{-1}(x)}{x}]
\dd{x_1}}{
( 1-t )
\int
\pi\!\left(
\phic^{-1}(x)
,x_1\right)
\det[\pdv{\phic^{-1}(x)}{x}]
\dd{x_1}}{
}
.
\label{eq:vtrue:pi}
\end{equation}

In all of the above cases,
the essence of Algorithm~\ref{alg:main} does not change
(except that in the case of dependent~$x_0$ and $x_1$ 
we should be able either to calculate the value of
$\pi\!\left(
\phic^{-1}(x)
,x_1\right)/\rho_1(x_1)$ or to estimate it).

%Analog of the Eq.~\eqref{eq:main_numeric}
%is the following
%\def\uandd{\pi\!\left(%
%\phicex^{-1}{t}{\ox^k}(x)%
%,\ox^k\right)\big/\rho_1(\ox^k)
%%
%}%
%\begin{multline}
%    \label{eq:main_numeric:general}
%       \sum_{k=1}^N
%       w(t,\ox^k,x)
%    \rho_{c}(x\mid \ox^k,t)
%    =\\=
%    \frac{\sum_{k=1}^N
%       \frac{\ox^k - x}{1-t} 
%       \cdot
%        \uandd
%    }{\sum_{k=1}^N\uandd
%    }.
%    %\\
%    %\qq{where}
%    %Y^k = -\frac12
%    %\frac{\norm{x- t \cdot \ox^k}^2_{\mathbb R^d}}{1-t}.
%\end{multline}
%When~$x_0$ and~$x_1$ are independent,
%the following simplification in the last formula is correct:
%$\uandd=\rho_0(\phicex{t}{\ox^k}(x))$.

\section{Several analytical results, following from the explicit formula}
\label{sec:anal}
In this section, we present several analytical results that directly follow from our exact formulas for the vector field,
which,
to the best of our knowledge, have not been published before.
\subsection{Exact path from one Gaussian to another Gaussian}
\label{sec:A:G-G}
Consider the flow from a one-dimensional Gaussian distribution 
$\rho_0\sim\normald{\mu_0}{\sigma_0}$
into another (with other parameters) Gaussian distribution
$\rho_1\sim\normald{\mu_1}{\sigma_1}$.
Note that in this case the generalization to the multivariate case is done directly, so the spatial variables are separated.

From the general formula~\eqref{eq:vtrue}
we have:
\def\underint{
\normaldx{\mu_0}{\sigma_0}{\frac{x-tx_1}{1-t}}
\normaldx{\mu_1}{\sigma_1}{x_1}
\dd x_1}
\begin{multline*}
v(x,t) = \frac{\int (x_1 - x)\underint}{(1-t)\int\underint}
=\\
\def\underint{\exp\left(-\bigl(\frac{x-tx_1}{1-t}-\mu_0\bigr)^2/(2\sigma_0^2)-(x_1-\mu_1)^2/(2\sigma_1^2)\right)\dd x_1}
=\frac{\int (x_1 - x)\underint}{(1-t)\int\underint}.
\end{multline*}
Both integrals in the last expression are taken explicitly:
\begin{multline*}
%\frac1{2\pi\sigma_0\sigma_1}
\int\underint
 =\\=
\frac{\exp \left(-\frac{(x-\mu_0 (1-t)-\mu_1 t)^2}%
{2 \left(\sigma_1^2 t^2+
\sigma_0^2 (1-t)^2\right)}\right)}{\sqrt{2 \pi }
 \sqrt{\sigma_0^2+\frac{\sigma_1^2t^2}{(t-1)^2}}}
=
\nosqrt\normaldx {\frac{\mu_0(1-t) + \mu_1 t}{1-t}}%
{{\sigma_0^2+\frac{\sigma_1^2t^2}{(t-1)^2}}}%
{\frac x{1-t}}.
\end{multline*}
Note that the last relation can be obtained as a distribution of two Gaussian random variables with corresponding parameters.

The second integral:
\begin{multline*}
\int\frac{x_1-x}{1-t}\underint
 =\\=
\frac{\exp \left(-\frac{(x-\mu_0 (1-t)-\mu_1 t)^2}%
{2 \left(
\sigma_1^2 t^2+
\sigma_0^2 (1-t)^2\right)}\right)}{\sqrt{2 \pi }}
\frac{(1-t) \left(\sigma_1^2 t (x-\mu_0)+\sigma_0^2 (t-1)
   (x-\mu_1)\right)}{\left(\sigma_1^2 t^2+\sigma_0^2
   (1-t)^2\right)^{3/2}}.
\end{multline*}

Thus, in the considered case we can explicitly write the expression for the vector field~$v$:
\begin{equation}
v(x,t) = 
\frac{\sigma_1^2 t (x-\mu_0)-\sigma_0^2 (1-t)
   (x-\mu_1)}%
{\sigma_1^2 t^2+\sigma_0^2
   (1-t)^2}.
\label{eq:vtrue:G-G}   
\end{equation}
For this vector field we can explicitly solve the equation for the path~$x(t)$ starting from the arbitrary point~$x_0$
\begin{equation*}
\left\{
\begin{aligned}
\pdv{x(t)}{t} &=  v(x(t),\,t), \\
    x(0) &= x_0 \\
\end{aligned}
\right.
.
\end{equation*}
The solution is:
\begin{equation}
x(t)=(1-t)\mu_0  + t\mu_1 + (x_0-\mu_0)\sqrt{ (\sigma_1/\sigma_0)^2 t^2 +(1-t)^2}.
\label{eq:traj:G-G}
\end{equation}
Note that although this solution does not correspond to the Optimal Transport joint distribution, 
since the  obtained path is not a straight line in general, 
(\ie we do not have a solution to the Kantorovich's formulation of the OT problem)
the endpoint 
$x(1)=\mu_1+(x_0-\mu_0)\dfrac{\sigma_1}{\sigma_0}$ 
falls exactly in the one that is optimal if we solve the OT problem in the Monge formultation.
Thus, the map $x(0)\to x(1)$ is the OT map for the case of 2 Gaussian.

\begin{figure}[!ht]%
\centering
\subfigure[Trajectories]{
\includegraphics[clip,width=0.6\linewidth]{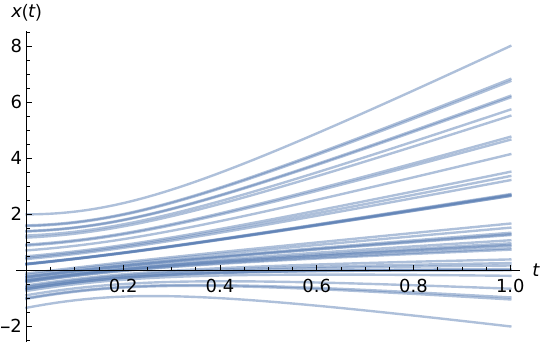}
}
\subfigure[Vector field]{
\includegraphics[clip,width=0.36\linewidth]{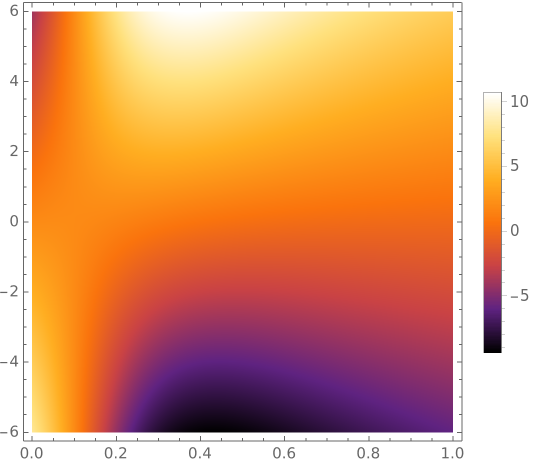}
}
\caption{a) $N=40$ random trajectories from from $\normald01$ to $\normald23$;
(b) 2D plot of the vector field in this case}
\label{fig:G-G}
\end{figure}

See the Fig.~\ref{fig:G-G}
for the examples of the paths for the obtained solution.

\subsection{From one Gaussian to Gaussian Mixture}
\label{sec:vtrue:G-GG}
Let initial distribution be standard Gaussian
$\rho_0=\normald 01$,
and the target distribution be Gaussian Mixture (GM)
of two symmetric Gaussians:
$\rho_1(x)=1/2(\normaldx\mu\sigma x)+ \normaldx{-\mu}\sigma x)$,
In this case, we can obtain exact form for~$v$
\begin{multline}
v(x,t) = 
\frac{\exp \left(-\frac{\mu ^2}{2 \sigma ^2}+\frac{\mu ^2 t^2+x^2}{\sigma ^2
   t^2+(t-1)^2}-\frac{x^2}{2 (t-1)^2}\right) }%
{\left(\sigma ^2 t^2+(t-1)^2\right)
   \left(e^{\frac{\left(x-\mu  t\right){}^2}{2 \left(\sigma ^2
   t^2+(t-1)^2\right)}}+e^{\frac{\left(\mu  t+x\right){}^2}{2 \left(\sigma ^2
   t^2+(t-1)^2\right)}}\right)} \times\\
\left[\mu  (t-1) \left(\exp \left(\frac{\left(\mu 
   (t-1)^2-\sigma ^2 t x\right){}^2}{2 \sigma ^2 (t-1)^2 \left(\sigma ^2
   t^2+(t-1)^2\right)}\right)-\exp \left(\frac{\left(\mu  (t-1)^2+\sigma ^2 t x\right){}^2}{2
   \sigma ^2 (t-1)^2 \left(\sigma ^2 t^2+(t-1)^2\right)}\right)\right)+\right.
   \\+\left.
   x \left(\sigma ^2
   t+t-1\right) \left(\exp \left(\frac{\left(\mu  (t-1)^2-\sigma ^2 t x\right){}^2}{2 \sigma ^2
   (t-1)^2 \left(\sigma ^2 t^2+(t-1)^2\right)}\right)+\exp \left(\frac{\left(\mu  (t-1)^2+\sigma
   ^2 t x\right){}^2}{2 \sigma ^2 (t-1)^2 \left(\sigma ^2
   t^2+(t-1)^2\right)}\right)\right)\right],
   \label{eq:vtrue:G-GG}
\end{multline}
but the expression for the path~$x(t)$ is unknown.

\begin{figure}[!ht]%
\centering
\subfigure[Trajectories]{
\includegraphics[clip,width=0.6\linewidth]{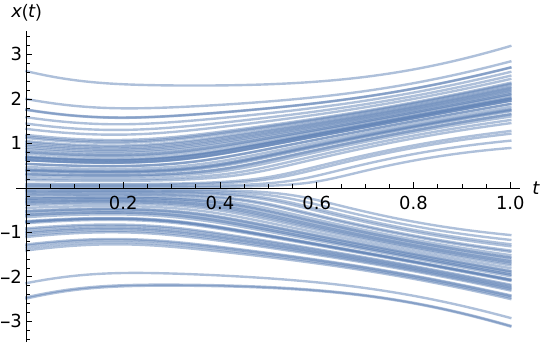}
}
\subfigure[Vector field]{
\includegraphics[clip,width=0.36\linewidth]{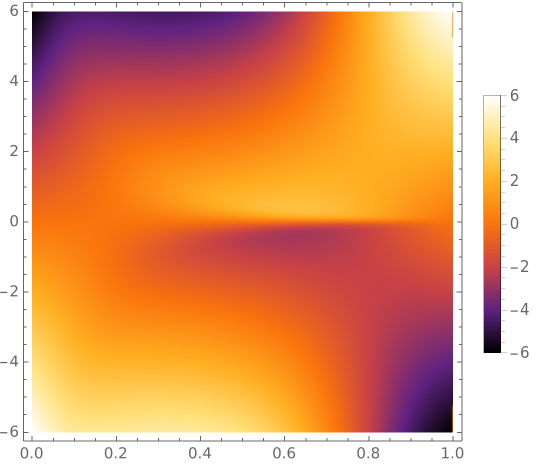}
}
\caption{a) $N=80$ random trajectories from $\normald01$ to GM of $\normald{-2}{1/2}$ and $\normald2{1/2}$;
(b) 2D plot of the vector field in this case}
\label{fig:G-GM}
\end{figure}

Numerically solution of the differential equation with the obtained vector field
give
the trajectories shown in Fig.~\ref{fig:G-GM}.

\subsection{From Gaussian to Gaussian with stochastic}
Using Eq.~\eqref{eq:vtrue:SDE:1}-\eqref{eq:strue:SDE:1}
we can explicitly calculate vector field~$v$
and score~$s$
with the setup as in Sec.~\ref{sec:A:G-G}
but with additional noise,
\ie in the stochastic case.

\subsubsection{Gaussian to Gaussian with noise}

Consider like in the Sec.~\ref{sec:A:G-G}
the flow from a one-dimensional standard  Gaussian distribution
$\rho_0\sim\normald01$
into another (with other parameters) Gaussian distribution
$\rho_1\sim\normald{\mu_1}{\sigma_1}$
but with additional noise as described above.

In this case we have for the field.
\begin{equation}
v(x,\,t)=
    \frac{
    x\bigl( t\sigma_1^2 + (1-t)\se^2/2  \bigr)
    -(x-\mu_1)\bigl((1-t)+t\se^2/2\bigr)
}{
t(1-t)\se^2+\sigma_1^2t^2+(1-t)^2
}
\label{eq:vtrue:G-G:SDE}
\end{equation}

We can solve ODE with this field and get the expression for the trajectories,
starting from the given point~$x_0$:
\begin{equation}
x(t)=
\mu _1 t+
x_0 \sqrt{t(1-t)\sigma_e^2+\sigma_1^2t^2+(1-t)^2}.
\label{eq:traj:G-G:SDE}
\end{equation}
These trajectories, for different~$x_0$ are depicted in Fig.~\ref{fig:G-GM:SDE}.
\begin{figure}[!ht]%
\centering
\def\pc#1{
\subfigure[\small Trajectories, $\sigma_e=#1$]{
\includegraphics[clip,width=0.29\linewidth]{traj_se=#1.pdf}
}
\subfigure[\small Vector field, $\sigma_e=#1$]{
\includegraphics[clip,width=0.15\linewidth]{img/field_se=#1.png}
}
}
%\pc{0.01}
\pc{0.3}
\pc{1}
\pc{3}
\pc{10}
\caption{a) $N=40$ random trajectories from $\normald01$ to $\normald23$
and
2D plot of the vector field in this case
for different~$\sigma_e$
}
\label{fig:G-GM:SDE}
\end{figure}

At the limit~$\se\to0$
expressions~\eqref{eq:vtrue:G-G:SDE}
and~\eqref{eq:traj:G-G:SDE}
turn into 
expressions~\eqref{eq:vtrue:G-G}
and~\eqref{eq:traj:G-G}
as expected.

For the score~$s$
in the considered case
we have
\begin{equation*}
    s(x,\,t)=
    \frac{
    t\mu_1-x
    }{
    (1-t)^2
    + t(1-t)\se^2
    + t^2\sigma_1^2
    }
\end{equation*}

Thus,
we can explicitly write expressions
for the stochastic process for the evolution
from the initial distribution $rho_0$ (standard Gaussian) to the final distribution~$\rho_1$:
\begin{multline*}
\dd x(x)=
\left[
    \frac{
    x\bigl( t\sigma_1^2 + (1-t)\se^2/2  \bigr)
    -(x-\mu_1)\bigl((1-t)+t\se^2/2\bigr)
}{
t(1-t)\se^2+\sigma_1^2t^2+(1-t)^2
}
\right.
+{}
\\
\left.
{}+
\frac{
g^2(t)
}2
%\biggl(
   \frac{
    t\mu_1-x
    }{
    (1-t)^2
    + t(1-t)\se^2
    + t^2\sigma_1^2
    }
%\biggr)^2
\right]
\dd t+
g(t)
\dd{W(t)}.
\end{multline*}
Here~$g(t)$ is arbitrary smooth function.
In the case of Shr\"odinger Bridge we take $g(t)=\se\sqrt{t(1-t)}$.

%Using Eq.~\eqref{eq:vtrue:SDE:1} we have for~$v$
%\begin{equation*}    
%v(x,\, t)=
%\frac{
%4(1-t)^2\bigl(\sigma_1^2 t (x-\mu_0)+\sigma_0^2 (t-1) (x-\mu_1)\bigr)
%- \se^2 (1-2 t)^2 t \left(x-\mu _1\right)
%}{
%(1-t)
%\Bigl(
%4 (1-t) \bigl(
% \sigma _0^2 (1-t)^2
% +
% \sigma _1^2  t^2
% \bigr)
% +\se^2 (1-2 t)^2t
% \Bigr)
%}
%.
%\end{equation*}
%
%
%For the score\todo{name this huinya right} 
%we use Eq.~\eqref{eq:strue:SDE:1}
%to obtain:
%\begin{equation*}
%s(x,t)=
%\frac{
%2 \se \sqrt{(1-t) t} 
%\,(2 t-1) \left(x - \mu _0 (1-t)-\mu _1 t\right)
%}{
%4 (1-t) \bigl(
% \sigma _0^2 (1-t)^2
% +
% \sigma _1^2  t^2
% \bigr)
%+\se^2 (1-2 t)^2t
%}
%.
%\end{equation*}

\section{Detail on the SDE case}
\label{sec:SDE}
\subsection{Optimal vector field and score for stochastic map}
Following~\cite{Tong2024}
we consider a so-called \emph{Brownian bridge}~$B(t)$ from  
$x_0$ to $x_1$ 
%as a stochastic path between two fixed points.
with constant diffusion rate~$\se$.
This stochastic process
can be expressed through a multidimensional standard Winner process~$W(t)$ as
\begin{equation}
    B(t\mid x_0,x_1)=(1-t)x_0 + tx_1
    + \se(1-t)W\left(\frac t{1-t}\right).
\label{eq:A:Bb}
\end{equation}
Thus,
the conditional distribution~$p(t,x\mid x_0,\,x_1)$
conditioned on the starting~$x_0$ and end point~$x_1$
is Gaussian:
\begin{equation*}
p(x,t\mid x_0,\,x_1)
=\nosqrt\normaldx {(1-t)x_0 + tx_1}{\se^2t(1-t)}x.
\end{equation*}
We can not directly use the results
Theorem 3 from~\cite{lipman2023flow}
(or similar Theorem 2.1 from~\cite{tong2024improving} )
for the Gaussian paths,
as in this case $\sigma(0)=0$.
%states, that we can explicitly write
%conditional vector field as
To circumvent this obstacle and to be able to write an expression for the conditional velocity, 
we assume that we have a Gaussian distribution with a very narrow peak at the initial ($t=0$) and final ($t=1$)
points. 
In other words, we will consider conditional probabilities of the form
\begin{equation}
p(x,t\mid x_0,\,x_1)
=\nosqrt\normaldx {(1-t)x_0 + tx_1}{\se^2(t+\eta)(1-t+\eta)}x,
\label{eq:SDE:pt}
\end{equation}
where parameter~$\eta$ is  small enough.
Then we can use the above Theorems and immediately write
\begin{equation}
    v_{x_0,x_1}(x,t)=
    \frac{\sigma'(t)}{\sigma(t)}\bigl(x - \mu(t)\bigr)+\mu'(t)
    =
    \frac{1-2t}{2(t+\eta)(1-t+\eta)}\bigl(x - (1-t)x_0 - tx_1\bigr)+x_1-x_0.
\label{eq:v:SDE}
\end{equation}
After integrating over~$x_0$ and~$x_1$, we can take the limit $\eta\to0$.
Thus, now for fixed~$x_0$ and~$x_1$ we do not have a fixed value of~$x_t$ in which to train the model,
but a random one. 
In general case, we end up to the loss:
\begin{equation}
\mathcal L_v=
        \mathbb E_{
        t\sim\mathcal U(0,1),\,
        \{x_1,x_0\}\sim\pi,\,
        x\sim p(\cdot,t\mid x_0,x_1)}
        \norm{
        \model(x,t)-
        v_{x_0,x_1}(x,t)
    }^2,
\label{eq:L:SDE}
\end{equation}
where~$\pi(x_1,\,x_0)$
is the density of the joint distributions with the marginal equal to the two given probabilities:
\begin{equation*}
\int\pi(x_1,\,x_0)\dd{x_1}=\rho_0(x_0),
\quad
\int\pi(x_1,\,x_0)\dd{x_0}=\rho_1(x_1)
.
\end{equation*}

In the simple case, $\pi(x_1,\,x_0)=\rho_0(x_0)\rho_1(x_1)$.
Vector field in Eq.~\eqref{eq:L:SDE}
if taken in the form of Eq.~\eqref{eq:v:SDE}.

%In the case of the 
%Brownian bridge
%we have 
%$$
%\phic(x_0)=x_0(1-t) + x_1t
%$$
%and
%$$
%p^\epsilon_t=
%\nosqrt\normald 0{\sigma_\varepsilon(t)}=
%\nosqrt\normald 0{\se^2 t(1-t)}.
%$$

Now, we can obtain an explicit form for the vector field~$v$ at which the written loss is reached its minimum
by performing the same calculations as in the derivation of formula~\eqref{eq:vtru_comm}:
%We get for the optimal value for vector field~:
\def\underint{p(x,t\mid x_0,\,x_1)\,\pi(x_0,x_1)
\dd{x_0}%
\dd x_1%
}
\begin{equation}
v(x,t)
=
\dfrac{
\iint
v_{x_0,x_1}(x,t)\,
\underint}{\iint
\underint}{
}
.
\label{eq:vtrue:SDE:1}
\end{equation}
%where $\rho^t_{\mathcal N}$ is the density of the normalized variable~$\epsilon_t$:
%$
%\frac{\epsilon_t-\mathbb E\epsilon_t}{\sqrt{\mathbb D\epsilon_t}}
%\sim
%\rho^t_{\mathcal N}
%$
%and we explicitly use the simplest map $
%\phic(x_0)=x_0(1-t) + x_1t
%$.

As in the work~\cite{Tong2024} we can also train score network.
Namely, as marginals for Brownian bridge are Gaussian,
we can write explicit conditional score
for conditional probabilistic path %(see \cite{Tong2024} for details):
\begin{equation*}
\grad\log p(x,t\mid x_0,\,x_1)=
\frac{\mu(t)-x}{\sigma^2_e(t)}=
\frac{x_0(1-t) + x_1t-x}{\se^2t(1-t)}.
\end{equation*}
In the work~\cite{Tong2024}
the following loss is introduced to train a model for this score
\begin{equation}
\mathcal L_s=
        \mathbb E_{
        t\sim\mathcal U(0,1),\,
        \{x_1,x_0\}\sim\pi,\,
        x\sim p(\cdot,t\mid x_0,x_1)}
        \norm{
        s_\theta(x,t)-
        \nabla\log p(x,t\mid x_0,\,x_1)
    }^2.
\label{eq:Ls:SDE}
\end{equation}

Similar to~\eqref{eq:vtrue:SDE:1}, for the optimal score~$s$ we have:
\def\underint{p(x,t\mid x_0,\,x_1)\,\pi(x_0,x_1)
\dd{x_0}%
\dd x_1%
}
\begin{equation}
s(x,t)
=
\dfrac{
\iint
\grad\log p(x,t\mid x_0,\,x_1)\,
\underint}{\iint
\underint}{
}
,
\label{eq:strue:SDE:1}
\end{equation}
where $p$ is given in~\eqref{eq:SDE:pt}.

\subsection{Use stochastic}
Note that the obtained vector field gives marginal distributions~$p(x,t)$, which (in the limit $\eta\to0$) 
at $t=1$ leads to the distribution we need: $p(x,t=1)=\rho_1(x)$. However, the addition of the stochastic term allows us to extend the scope of application of the explicit formula for the vector field. In particular, it can be applied to the situation when we have two sets of samples and both distributions are unknown, as well as the possibility of constructing SDE and solving it using, for example, the Euler--Maruyama method (see examples below).

As consequence of Theorem 3.1 from~\cite{Tong2024} we have that,
if $v$ is given by Eq.~\eqref{eq:vtrue:SDE:1}
then ODE 
\begin{equation}
    \pdv{\rho(x,t)}{t}=-\divtrue\bigl(\rho(x,t)v(x,t)\bigr)
\label{eq:A:Ph-P:1}
\end{equation}
recovers the marginal~$\rho(x,t)$ (with the given initial conditions)
of the stochastic process~$P(t)$ which is obtained by marginalization 
conditional Brownian bridge~\eqref{eq:A:Bb}
over
initial and target distribution
\begin{equation*}
    P(t)=\int B(t\mid x_0,x_1)\pi(x_0,x_1)\dd{x_0}\dd{x_1}.
\end{equation*}

As the second consequence of this Theorem,
the SDE
\begin{equation}
   \dd{x(t)}
   =
   \Bigl(v\bigl(x(t),t\bigr)+\frac{g^2(t)}2s\bigl(x(t),t\bigr)\Bigr)\dd{t}
   +
   g(t)\dd{W(t)}
\label{eq:A:SDE:sp}
\end{equation}
generates so-called  Markovization of the process~$P(t)$.
Indeed, we can rewrite PDE Eq.~\eqref{eq:A:Ph-P:1}
in the form
\begin{equation*}
    \pdv{\rho(x,t)}{t}=-\divtrue\Bigl(
    \rho(x,t)v(x,t)
    +\frac{g^2(t)}2\grad\rho(x,t)
    \Bigr)
    +\frac{g^2(t)}2\Delta\rho(x,t),
%\label{eq:A:Ph-P:1}
\end{equation*}
where 
%operator~$\grad$ is the gradient with respect to variable~$x$
nabla operator is defined as~$\Delta=\divtrue\grad$.
Thus, we get the Fokker–Planck equation
for the density of the stochastic process~\eqref{eq:A:SDE:sp}.

\subsection{Particular cases}
In particular case of Brownian bridge when $\sigma_e(t)=\sigma_\epsilon\sqrt{t(1-t)}$, 
then
$\sigma_e'(t)=\sigma_\epsilon (1-2t)/\bigl(2\sqrt{t(1-t)}\bigr)$.
In this section we consider simple case of
separable variables
$\pi(x_0,x_1)=\rho_0(x_0)\rho_1(x_1)$.
%and $\rho^t_{\mathcal N}$ 
%is the density of the standard Gaussian distribution:
%$\rho^t_{\mathcal N}=\normald01$ for all $0\leq t\leq1$.

\subsubsection{Gaussian initial distribution}
In the case, when $\rho_0$ is standard Gaussian distribution: $\rho_0=\normald01$,
we can take integral on~$x_0$  and then take the limit~$\eta\to0$
in the expressions for~$v$ and $s$.
First, consider the expression for~$v$:
where we use explicit expression~\eqref{eq:SDE:pt} for conditional density path
and 
Eq.~\eqref{eq:v:SDE} for conditional velocity:
\def\underint{%
    \nosqrt\normaldx{x_1t}{\se^2 t (1-t)+(1-t)^2}x
\rho_1(x_1)
\dd x_1%
}%
\def\wholeex{%
\dfrac{
\int
w(x,t\mid x_1)
\underint}{
\int
\underint}{
}%
}%
\begin{multline}
v(x,t)
=
\wholeex
=\\=
\def\underint{%
\rho_0
\Bigl(
\frac{x-x_1t}{\sqrt{\se^2 t (1-t)+(1-t)^2}}
\Bigr)
\rho_1(x_1)
\dd x_1%
}
\wholeex,
\label{eq:vtrue:SDE:G}
\end{multline}
where 
$w(x,t\mid x_1)$
is the conditional velocity, generated by the conditional map
$\phic(x)=\sqrt{\se^2 t (1-t)+(1-t)^2}+tx_1$:
\begin{equation*}
w(x,t\mid x_1)
=
%\frac{\se^2 \bigl((1-2 t)x + t x_1\bigr)-2 (1-t) (x-x_1)}%
%{2 (1-t)
%   \bigl(\left(\se^2-1\right) t+1\bigr)}
%=
\frac{x_1-x}{1-t+t\se^2}+
\se^2\frac{(1-2 t)x + t x_1}%
{2 \bigl((1-t)^2+
   (1-t)t\se^2\bigr)}.
\end{equation*}
%$\mathcal N(\cdot\vert\, \mu, \Sigma)$ denotes the Gaussian distribution with mean $\mu$ and covariance $\Sigma$;
%I$ is identity matrix.

Thus, note that
in the case of Gaussian distributions,
all the difference between this expression and the expression without the stochastic part is the appearance of additional (time-dependent, in general) variance.
Marginal distributions are still Gaussian's.
% So, it is effectively equivalent to simply using a different initial distribution. %(which, however, in this case depends on the value of the variable $t$).

Similar, using Eq.~\eqref{eq:strue:SDE:1}
we have for the score~$s$:
\def\underint{%
    \nosqrt\normaldx{x_1t}{\se^2 t (1-t)+(1-t)^2}x
\rho_1(x_1)
\dd x_1%
}%
\def\wholeex{%
\dfrac{
\int
(t x_1-x)
\underint}{
\bigl((1-t)^2+
   (1-t)t\se^2\bigr)
\int
\underint}{
}%
}%
\begin{multline}
s(x,t)
=
\wholeex
=\\=
\def\underint{%
\rho_0
\Bigl(
\frac{x-x_1t}{\sqrt{\se^2 t (1-t)+(1-t)^2}}
\Bigr)
\rho_1(x_1)
\dd x_1%
}
\wholeex.
\label{eq:strue:SDE:G}
\end{multline}

\subsubsection{Samples instead of distributions}
\label{seq:SDE:samples}
Consider the case where we only have access to the samples 
$\{x^i_0\}_{i=1}^{N_0}$
and
$\{x^i_1\}_{i=1}^{N_1}$
from both distributions, $\rho_0$ and $\rho_1$,
but do not know their explicit expressions.
In this case, we can
estimate the vector field using
by a method similar to the one we used to estimate the vector field in~\eqref{eq:vd}:
%formally put these distributions as equal sums of delta-functions at all available points of these distributions:
%\begin{equation*}
%    \pi(x,y)=
%    \biggl(
%    \frac1{N_0}
%    \sum_{i=1}^{N_0}\delta(x-x^i_0)
%    \biggr)
%    \biggl(
%    \frac1{N_1}
%    \sum_{j=1}^{N_1}\delta(y-x^j_1)
%    \biggr)
%.
%\end{equation*}
%Substituting this expression into the Eq.~\eqref{eq:vtrue:SDE:1},
%and using $v_{x_0,x_1}$ from~\eqref{eq:v:SDE}
%we obtain
\def\underint{p(x,t\mid x^i_0,\,x^j_1)
}
\begin{equation}
v(x,t)
\approx
\dfrac{
\sum_{i=1}^{N_0}
\sum_{j=1}^{N_1}
v_{x^i_0,x^j_1}(x,t)\,
\underint}{
\sum_{i=1}^{N_0}
\sum_{j=1}^{N_1}
\underint}
.
\label{eq:vtrue:approx:D}
\end{equation}

Similar for the score
\begin{equation}
s(x,t)
\approx
\dfrac{
\sum_{i=1}^{N_0}
\sum_{j=1}^{N_1}
\grad p(x,t\mid x^i_0,x^j_1)\,
\underint}{
\sum_{i=1}^{N_0}
\sum_{j=1}^{N_1}
\underint}
.
\label{eq:strue:approx:D}
\end{equation}

In addition, we can also use the importance sampling method in this case.
Namely we can use both approaches: self-normalized importance sampling and rejection sampling,
similar to what is described in Sec.~\ref{sec:Int-est}

%\subsubsection{Solving SDE}

\section{Consistency of Eq.\texorpdfstring{~\eqref{eq:main_numeric:general}}{(29)} in the case of optimal transport}
\label{sec:A:OT}
Let us analyze what happens if in formula~\eqref{eq:main_numeric:general}
the joint density~$\pi$
represents the following Dirac delta-function\footnote{Further reasoning is not absolutely rigorous, 
and in order not to introduce the axiomatics 
of generalized functions, 
we can assume that the delta function is the limit 
of the density of a normal distribution with mean~$0$ 
and variance tending to zero.}:
$$
\pi(x_0,x_1)
=
\delta\bigl(x_0-F(x_1)\bigr),
$$
\ie we have a deterministic mapping~$F$ from~$x_1$ to $x_0$.
% which is assumed to be a bijection.
Then, 
the Eq.~\eqref{eq:vtrue:pi} come to
%$ :3 hello
\def\phinv#1#2{\frac{#1-tx_1}{1-t}}
\begin{equation*}
v(x,t)
=
\dfrac{
\int
(x_1-x)\,
\delta\!\left(
\phic^{-1}(x)
-
F(x_1)\right)
\dd{x_1}}{
( 1-t )
\int
\delta\!\left(
\phic^{-1}(x)
-
F(x_1)\right)
\dd{x_1}}{
}
.
%=\\
%\dfrac{
%\int
%(x_1-x)\,
%\delta\!\left(
%\phinv x{x_1}
%-
%F(x_1)\right)
%\dd{x_1}}{
%( 1-t )
%\int
%\delta\!\left(
%\phinv x{x_1}
%-
%F(x_1)\right)
%\dd{x_1}}{
%}.
\end{equation*}
Let~$y(x,t)$ be the unique solution of the equation
\begin{equation}
%\phinv xy
\phicex^{-1}ty(x)
=F(y),
\label{eq:App:eqOT}
\end{equation}
considered as an equation on~$y$.
Then
$$
v(x,t)
=
\frac{x-y(x,t)}{1-t}.
$$

Now, let us use linear mapping
$\phic(x)=x_1 t+x(1-t)$, 
with inverse
$\phic^{-1}(x)=\phinv x{x_1}$, 
and consider the simplest case when the original distribution is 
a $d$-dimensional standard Gaussian
and~$\rho_1$ is a $d$-dimensional Gaussian with mean~$\mu$ and diagonal 
variance~$\Sigma=\hbox{diag}(\sigma)$.
We know the OT correspondence
between Gaussians, 
namely
$$
%F(x_1)=\sigma x_1+\mu.
\bigl(F(x_1)\bigr)_i= \frac{(x_1-\mu)_i}{\Sigma_{ii}},
\quad \forall 1\geq i\geq d
.
$$
Here and further by index~$i$
we denote $i$th component of the corresponding vector.
Then, the Eq.~\eqref{eq:App:eqOT} reads as
$$
\frac{(x-yt)_i}{1-t}=
\frac{(y-\mu)_i}{\Sigma_{ii}},
$$
with the solution
$$
\bigl(y(x,t)\bigr)_i=\frac{\mu_i(1-t) +x_i\Sigma_{ii}}{1+(\Sigma_{ii}-1) t}.
$$
Then the expression for the vector field is
$$
\bigl(v(x,t)\bigr)_i=
\frac{\mu_i +x_i(\Sigma_{ii}-1) }{1+(\Sigma_{ii}-1)t}.
$$
Now, 
knowing the expression for velocity, 
we can write the equations for the trajectories~$x(t)$:
$$
\left\{
\begin{aligned}
\bigl(x'(t)\bigr)_i&=
\frac{\mu_i +(x(t))_i(\Sigma_{ii}-1) }{1+(\Sigma_{ii}-1)t},
\\
x(0)_i&=(x_0)_i
\end{aligned}
\right.
.
$$
This equation have closed-form solution:
$$
x(t)=\mu t +x_0-\left(1-\sigma\right)tx_0.
$$
Analyzing the obtained solution, 
we conclude that, first, the trajectories
obey the given mapping~$F$:
$$
\bigl(F(x(1))\bigr)_i= (x_0)_i=\frac{(x(1)-\mu)_i}{\Sigma_{ii}},
$$

And, second, the trajectories are straight lines (in space),  
as they should be when the flow carries points along the optimal transport.

As a final conclusion, note that, 
of course, 
if we are mapping optimal transport~$F$, 
then it is meaningless to use numerical formula~\eqref{eq:main_numeric}. 
However, usually the exact value of the mapping~$F$
is not known, 
and our 
theoretical formula~\eqref{eq:vtrue:pi} 
can help to rigorously establish the error 
that is committed when an approximate mapping is used 
instead of the optimal one.

\section{Analytical derivations for example in Fig.~\ref{fig:CFM_ExFM_disps}}
\label{sec:G-G:A}
\subsection{CFM dispersion}
To derive the analytical expression for the optimal flow velocity in the case of two normal distributions 
$\rho_0 \sim N(0,I)$ and 
$\rho_1 \sim N(\mu,\sigma^2I)$,
we start by substituting
$\mu_0=0$, $\sigma_0=1$,
$\mu_1=\mu$, $\sigma_1=\sigma$,
to the exact expression~\eqref{eq:vtrue:G-G}
to get
\def\vpart{w}
\begin{equation}
v(x,t) =  \frac{t\sigma^2 + t -1}{(1-t)^2 + t^2\sigma^2} x + \frac{1}{(1-t)^2 + t^2\sigma^2}(\mu-t\mu)
=\vpart(t)x + C,
\label{eq:truev_gaus2}
\end{equation}
%To derive the explicit formula of dispersion update for the CFM objective at time $t$,
%we start by
%substituting 
%We can express this velocity function  as:
%\begin{equation}
%v(x,t) = 
%\end{equation}
where 
$$
\vpart(t) = \frac{t\sigma^2 + t -1}{(1-t)^2 + t^2\sigma^2},
$$
and
$C$ is constant independent of $x$.
%and $t$. 
We then redefine the dispersion based on Eq.~\eqref{eq:def_disp} using $x = (1-t)x_0 + t x_1$ with $x_0 \sim \rho_0$ and $x_1 \sim \rho_1$:
\begin{equation}
\mathbb{D}_{x,x_1}f(x,\,x_1) = \mathbb{D}_{x_0,x_1}f\bigl((1-t)x_0 + tx_1,\,x_1\bigr) 
\end{equation}
This leads us to the final expression:
\begin{align*}    
\mathbb{D}_{x,x_1}\Delta v(x,t) &= \mathbb{D}_{x_0,x_1}((1-\vpart(t))x_1 -
(1 + \vpart(t)(1-t))x_0) =\\
&=
(1 + \vpart(t)(1-t))^2\mathbb{D}_{x_0}x_0+
(1-\vpart(t))^2\mathbb{D}_{x_1}x_1.
\end{align*}

This provides a comprehensive representation of the updated dispersion for the CFM objective at any given time~$t$.

\subsection{ExFM dispersion}

\begin{algorithm}[t!bh]
    \caption{Computation ExFM dispersion algorithm}
    \label{alg:ExFM_disp_numeric}
	\begin{algorithmic}[1]
 \REQUIRE{
 Density function for initial distribution $\rho_0$; sampler for target distribution $\rho_1$; parameter $M$ (number of samples for evaluation); parameter $N$ (number of samples from $\rho_1$ for certain samples $x \sim \rho_m(x,t)$); optimal model $v(x,t)$; time for evaluation $t$.
 }
 \ENSURE{
 numerical evaluation of dispersion update for ExFM objective
 }
 \STATE Sample $(M\cdot N)$ samples $x_1^{i,j}$ from $\rho_1$, where $i \in [1,M]$ and $j \in [1,N]$
 \STATE Sample $(M)$ samples $x_0^i$ from $\rho_0$, where $i \in [1,M]$
 \STATE Compute points $x^{i}$ as $(1-t)x_0^i + tx_1^{i,0}$
 \STATE Compute $v^d(x^i,t) = \sum\limits_{j=1}^N\Tilde{\rho}^{i,j}(t)\frac{x_1^{i,j}-x^i}{1-t}$, where $\Tilde{\rho}^{i,j}(t) = \rho_0\left(\frac{x^{i}-tx_1^{i,j}}{1-t}\right)/\sum\limits_{j=1}^N\rho_0\left(\frac{x^{i}-tx_1^{i,j}}{1-t}\right)$
 \STATE Compute and return dispersion $\mathbb{D}_i(v(x^i,t) - v^d(x^i,t))$
\end{algorithmic}
\end{algorithm}

The analytical derivation of the updated dispersion for the ExFM objective proves to be complex in practice. Therefore, for the example at hand, a numerical scheme was employed for evaluation. The procedure outlined in Alg.~\ref{alg:ExFM_disp_numeric} was utilized for this task. The experiment's parameters for the algorithm were as follows: $M=200k$, $N=128$, $\rho_0 = N(0,I)$, $\rho_1 = N(\mu,\sigma^2I)$, and the optimal model $v(x,t)$ was derived from equation~\eqref{eq:truev_gaus2}.

%Note, that
%the boundedness of the density~$\rho_1$, which we used earlier,
%follows
%from Assumption~\ref{ass:rho_1}.

\section{Additional Experiments}
\label{seq:AppAdd:numeric}

\subsection{2D toy examples}
To ensure the reliability and impartiality of the outcomes, we carried out the experiment under uniform conditions and parameters. Initially, we generated a training set of batch size $N = 10{,}000$ points. The employed model was a simple Multilayer Perceptron with ReLu activations and 2 hidden layers of 512 neurons, \texttt{Adam} optimizer with a learning rate of $10^{-3}$, EMA with rate of $0.9$ and no learning rate scheduler. 
We determined the number of learning steps equal to $15\,000$ and $30\,000$ learning steps for \texttt{rings} dataset since the more comprehensive structure of the data. Subsequently, we configured the mini batch size $n = 512$ during the training procedure, with the primary objective of minimizing the Mean Squared Error (MSE) loss. The full training algorithm and notations can be seen in Algorithm~\ref{alg:main}.
To perform sampling, we employed the function \texttt{odeint} with \texttt{dopri5} method from the python package \texttt{torchdiffeq} with \texttt{atol} and \texttt{rtol} equal $10^{-5}$.
\begin{table}[!htb]
\centering
\def\g{\bf}
%\small
\caption{Energy Distance comparison for ExFM, CFM and OT-CFM methods for 2D-toy datasets for $15\, 000$ learning steps ($30\, 000$ learning steps for \texttt{rings} dataset), mean and std taken from 10 sampling iterations.}
\begin{sc}
%\small
\resizebox{0.9\linewidth}{!}{%
\begin{tabular}{lrrr}
\toprule
   Data & ExFM & CFM & OT-CFM \\
\midrule
 swissroll    & 1.20e-03 $\pm$ 9.6e-04 & 1.58e-03 $\pm$ 6.4e-04 & \g 8.28e-04 $\pm$ 3.12e-04 \\
 moons        & \g 5.58e-04 $\pm$ 3.45e-04 & 1.27e-03 $\pm$ 8.2e-04 & 6.99e-04 $\pm$ 4.38e-04 \\
 8gaussians   & \g 1.26e-03 $\pm$ 6.4e-04 & 1.62e-03 $\pm$ 6.0e-04 & 1.88e-03 $\pm$ 8.0e-04 \\
 circles      & \g 6.66e-04 $\pm$ 4.69e-04 & 8.34e-04 $\pm$ 4.72e-04 & 9.70e-04 $\pm$ 5.40e-04 \\
 2spirals     & \g 8.15e-04 $\pm$ 2.91e-04 & 1.91e-03 $\pm$ 7.7e-04 & 1.74e-03 $\pm$ 5.5e-04 \\
 checkerboard & \g 1.32e-03 $\pm$ 5.6e-04 & 3.41e-03 $\pm$ 1.19e-03 & 2.00e-03 $\pm$ 1.00e-03 \\
 pinwheel     & \g 8.65e-04 $\pm$ 6.12e-04 & 2.48e-03 $\pm$ 8.8e-04 & 1.11e-03 $\pm$ 3.2e-04 \\
 rings        & \g 5.75e-04 $\pm$ 3.61e-04 & 1.53e-03 $\pm$ 4.3e-04 & 1.19e-03 $\pm$ 3.6e-04 \\
        
\bottomrule
\end{tabular}
}
\end{sc}
\label{tab:CFM_ExFM_toy}
\end{table}

We present visual and quantitative results to evaluate the performance of our proposed method, ExFM. Visualizations of the learned distributions are presented in Figure \ref{fig:Extended_toy}. The corresponding data densities can be found in Figure \ref{fig:Extended_density_toy}. We sampled data from both the beginning and end of the training process. The results clearly show that ExFM outperforms the baseline CFM and the OT-CFM, particularly on the \texttt{rings} dataset. This can be attributed to ExFM's ability to effectively capture the complexities of this challenging distribution.

To further support the effectiveness of ExFM, we analyzed the training losses. The complete progression of these losses is visualized in Figure \ref{fig:images_2d_loss}. This figure highlights the significantly lower variance observed in ExFM's training loss compared to the CFM method.

For quantitative evaluation, we employed the Energy Distance metric and Wasserstein distance. The results of Energy Distance are presented in Table~\ref{tab:CFM_ExFM_toy}, Wasserstein distance in Table~\ref{tab:W_toy} while Figure \ref{fig:images_2d_ed} showcases the progression of this metric during the training procedure. Interestingly, CFM, OT-CFM and ExFM models achieve rapid convergence in terms of this metric at the beginning of learning. Additionally, the metric values remain relatively stable throughout the training process. However, the superior visual quality achieved by ExFM (as observed in Figure \ref{fig:Extended_toy}) suggests that the Energy Distance metric might not be the most suitable choice for evaluating this specific task.

\begin{figure}[!ht]
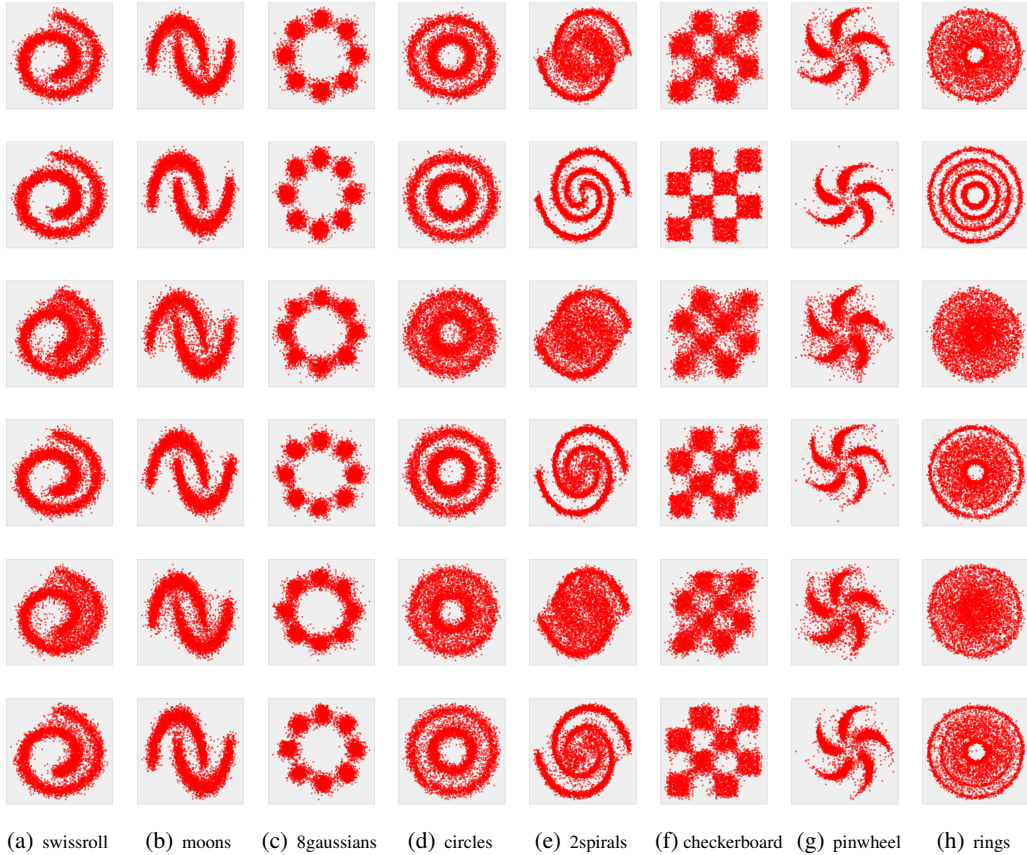
%
\def\sbsub#1#2{\includegraphics[trim=0cm 0cm 0cm 0cm,width=0.13\linewidth]{#1_#2_1500}}%
\def\sbsubb#1#2{\includegraphics[trim=0cm 0cm 0cm 0cm,width=0.13\linewidth]{#1_#2_15000}}%
\def\sb#1#2{\subfigure[\scriptsize#2]{%
\hbox to 4.7em{\hss\vbox{
\sbsub{ExFM}{#1}\\
\sbsubb{ExFM}{#1}\\
\sbsub{CFM}{#1}\\
\sbsubb{CFM}{#1}\\
\sbsub{OT-CFM}{#1}\\
\sbsubb{OT-CFM}{#1}
}\hss}}%
}
\centering
\sb{swissroll}{swissroll}
\sb{moons}{moons}
\sb{8gaussians}{8gaussians}
\sb{circles}{circles}
\sb{2spirals}{2spirals}
\sb{checkerboard}{\hbox to 4.4em{\kern-0.5ex checkerboard}}
\sb{pinwheel}{\hbox to 3.0em{pinwheel}}
\sb{rings}{rings}
\caption{Visual comparison of methods on toy 2D data. First and second rows sampled by ExFM, third and fourth rows sampled by CFM, fifth and six rows sampled by OT-CFM. The upper row in pairs of the same method sampled after $1\,500$ learning iterations ($3\,000$ for \texttt{rings} dataset), the lower row in pairs of the same method sampled after $15\,000$ learning iterations ($30\,000$ for \texttt{rings} dataset).
%Epoch size for "swissroll",  "moons", "8gaussians" is 400, for "circles" is 1000, for  "checkerboard", "pinwheel", "2spirals" is 2000, for "rings" is 5000. 
}%
\label{fig:Extended_toy}%
\end{figure}

\begin{figure}[!ht]%
\def\sbsub#1#2{\includegraphics[trim=0cm 0cm 0cm 0cm,width=0.13\linewidth]{dense_#1_#2_1500}}%
\def\sbsubb#1#2{\includegraphics[trim=0cm 0cm 0cm 0cm,width=0.13\linewidth]{dense_#1_#2_15000}}%
\def\sb#1#2{\subfigure[\scriptsize#2]{%
\hbox to 4.7em{\hss\vbox{
\sbsub{ExFM}{#1}\\
\sbsubb{ExFM}{#1}\\
\sbsub{CFM}{#1}\\
\sbsubb{CFM}{#1}\\
\sbsub{OT-CFM}{#1}\\
\sbsubb{OT-CFM}{#1}
}\hss}}%
}
\centering
\sb{swissroll}{swissroll}
\sb{moons}{moons}
\sb{8gaussians}{8gaussians}
\sb{circles}{circles}
\sb{2spirals}{2spirals}
\sb{checkerboard}{\hbox to 4.4em{\kern-0.5ex checkerboard}}
\sb{pinwheel}{\hbox to 3.0em{pinwheel}}
\sb{rings}{rings}
\caption{Densities comparison of methods on toy 2D data. First and second rows sampled by ExFM, third and fourth rows sampled by CFM, fifth and six rows sampled by OT-CFM. The upper row in pairs of the same method sampled after $1500$ learning iterations ($3000$ for \texttt{rings} dataset), the lower row in pairs of the same method sampled after $15,000$ learning iterations ($30,000$ for \texttt{rings} dataset).
%Epoch size for "swissroll",  "moons", "8gaussians" is 400, for "circles" is 1000, for  "checkerboard", "pinwheel", "2spirals" is 2000, for "rings" is 5000. 
}%
\label{fig:Extended_density_toy}%
\end{figure}

\begin{figure*}[!tb]
\centering
\label{fig:images_2d_ed}
\def\ic#1{\includegraphics[clip, width=0.32\linewidth]{ema_Energy_Distance_#1.png}}
\ic{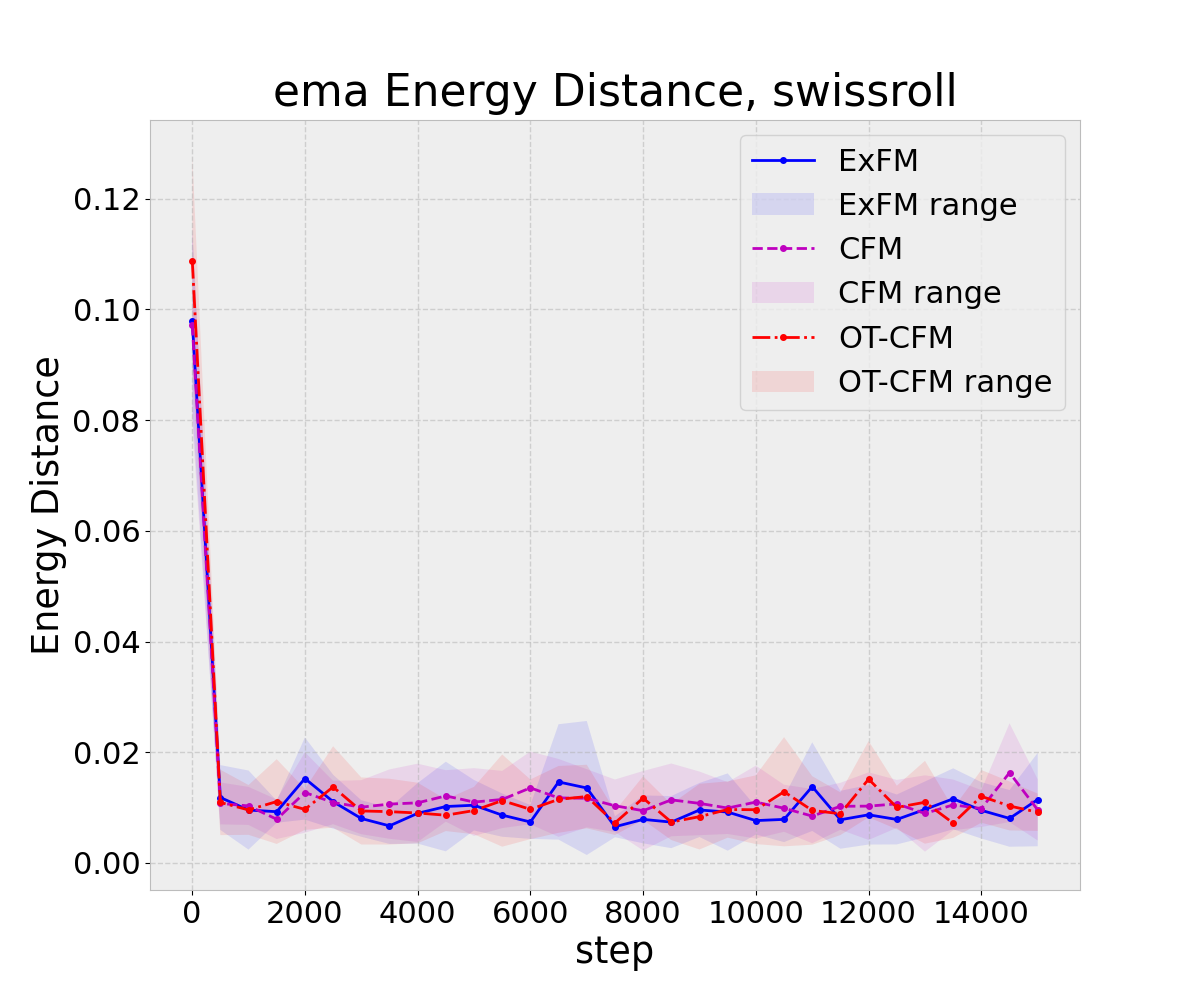}
\ic{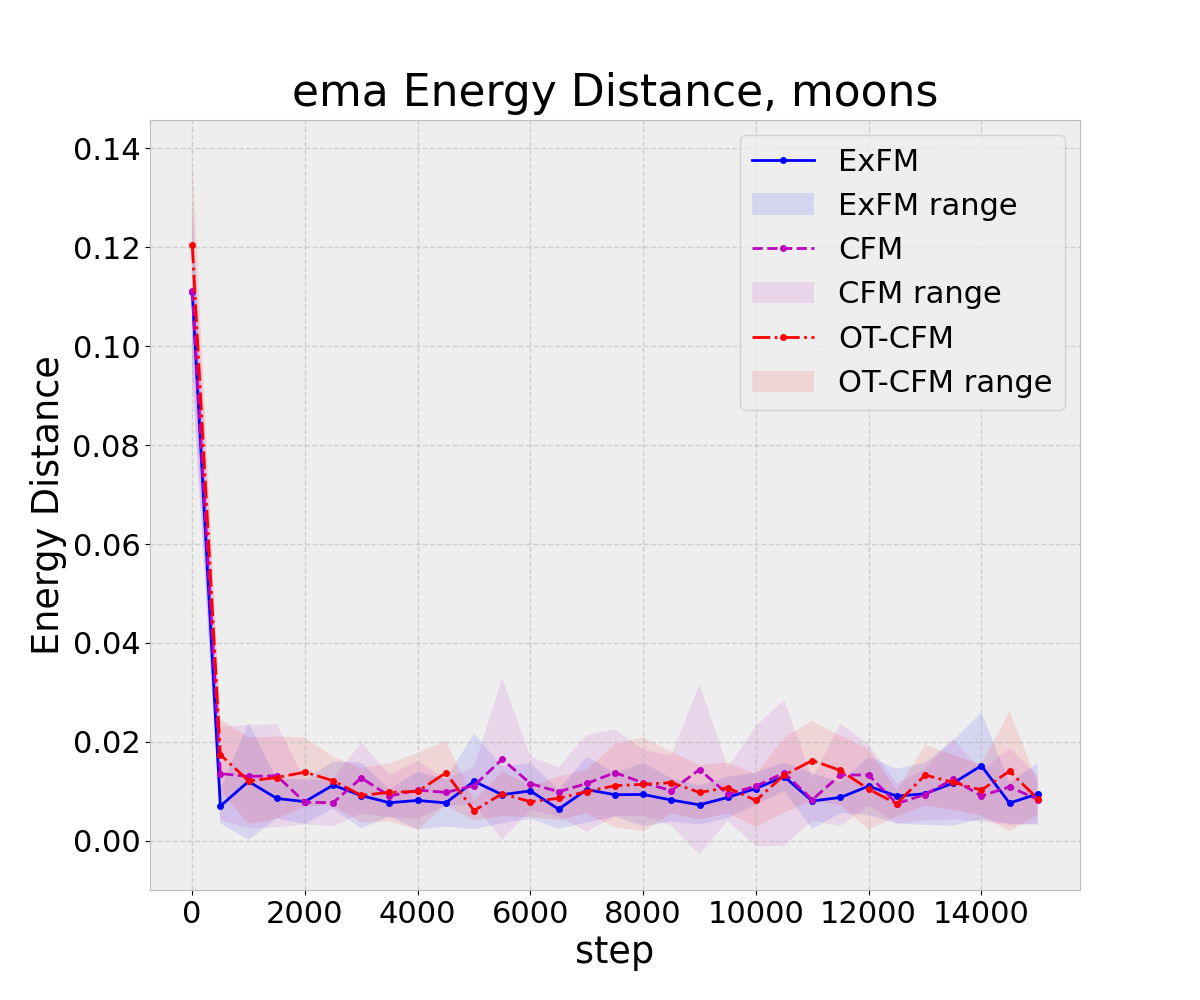}
\ic{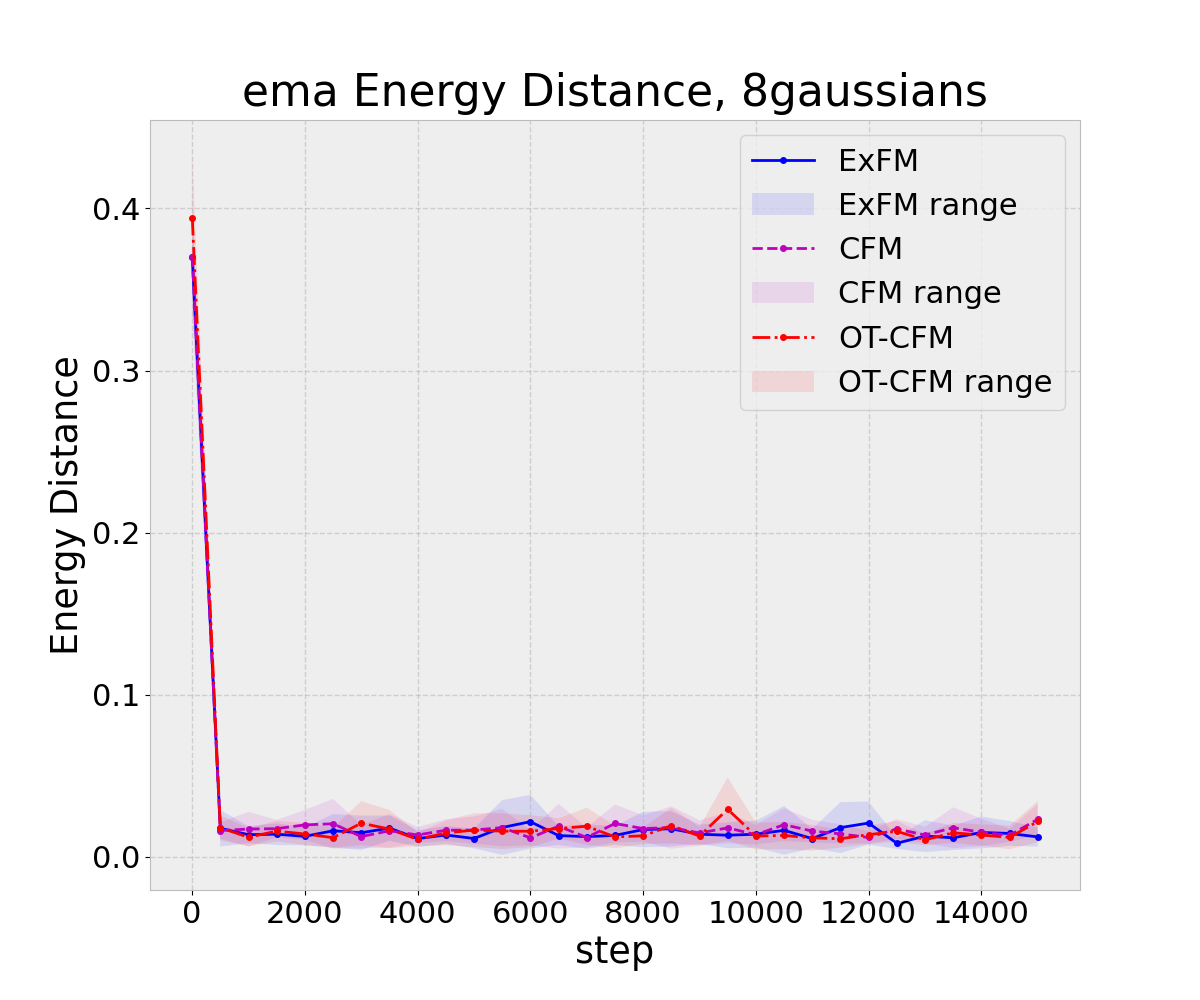}
\ic{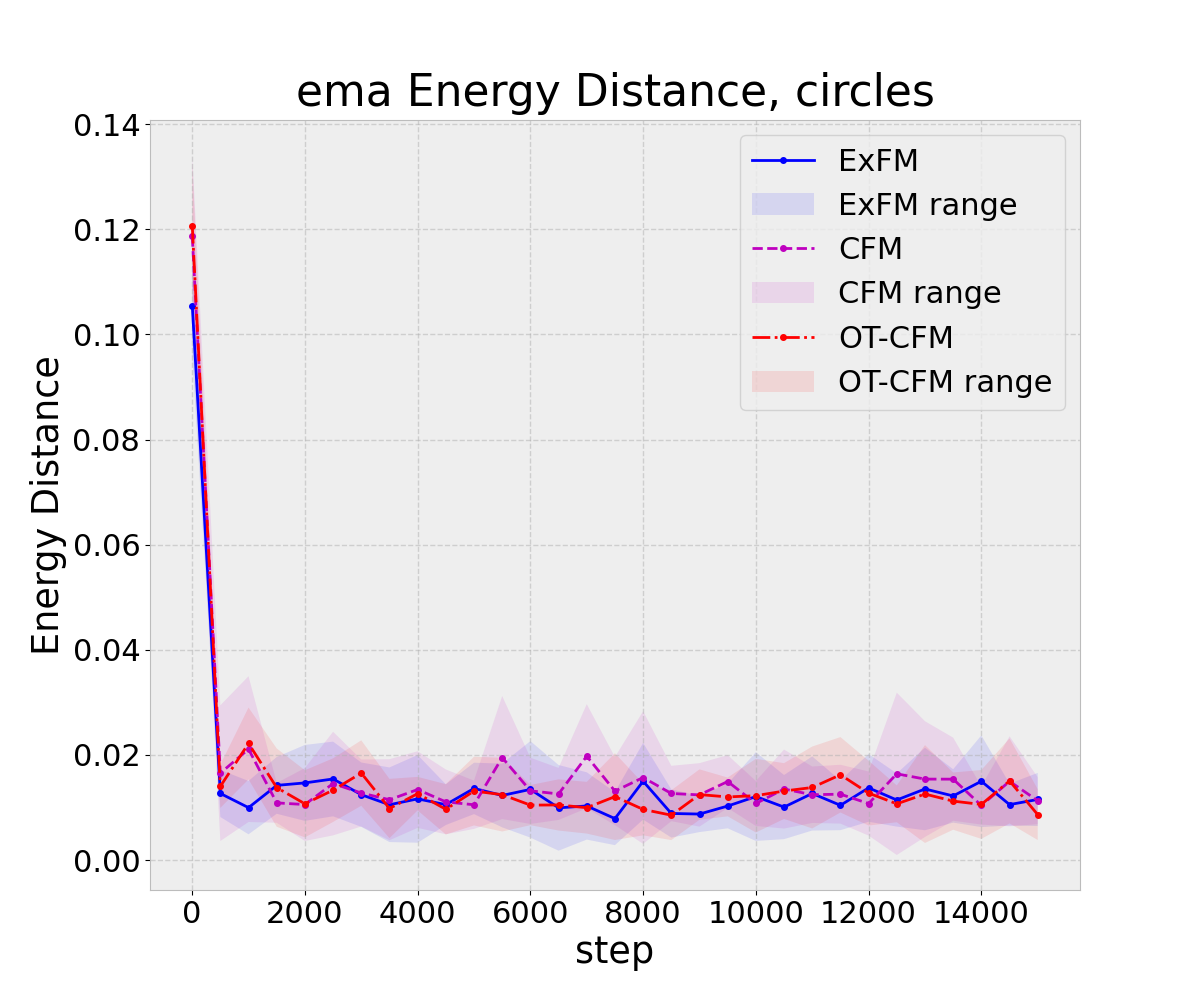}
\ic{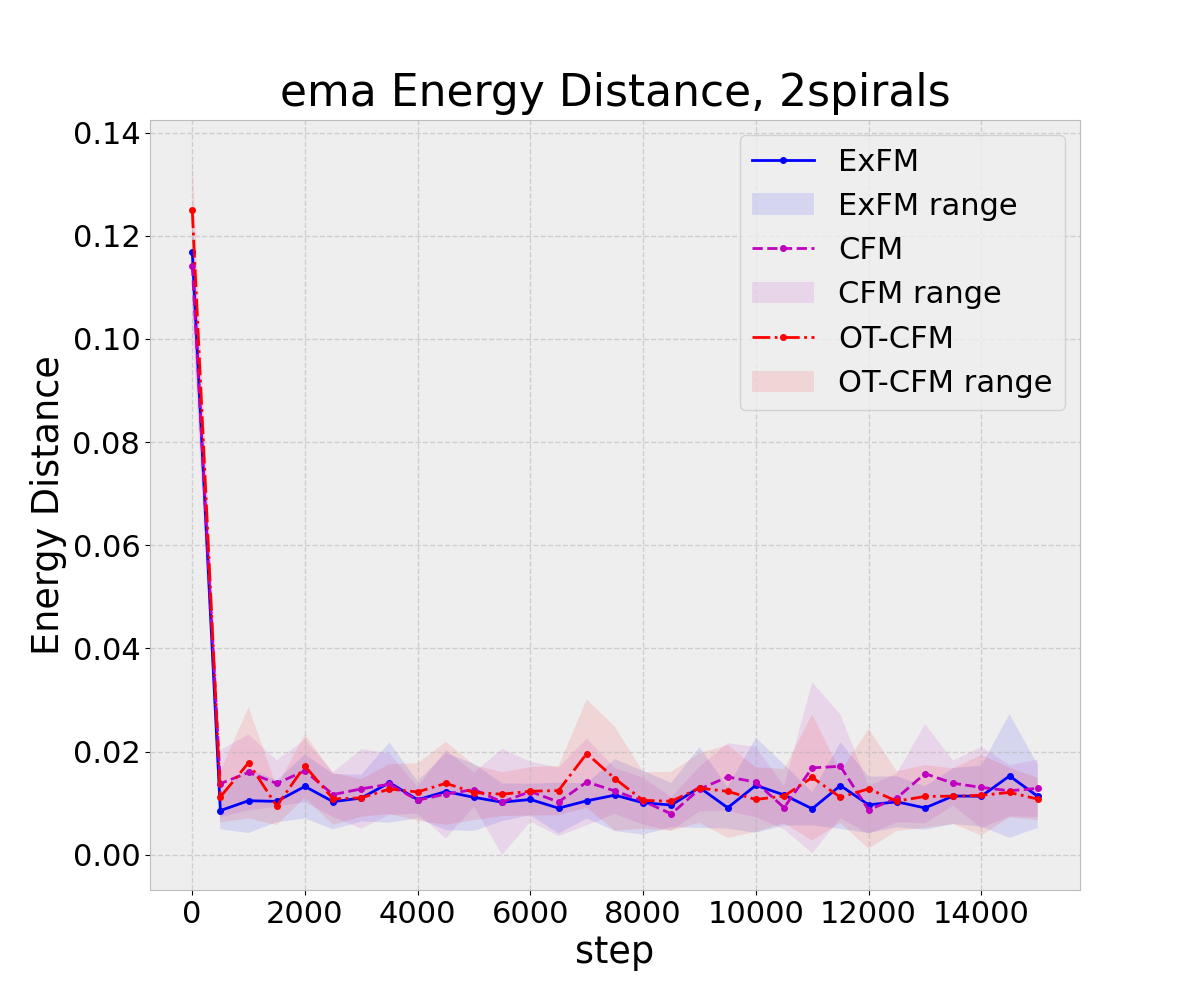}
\ic{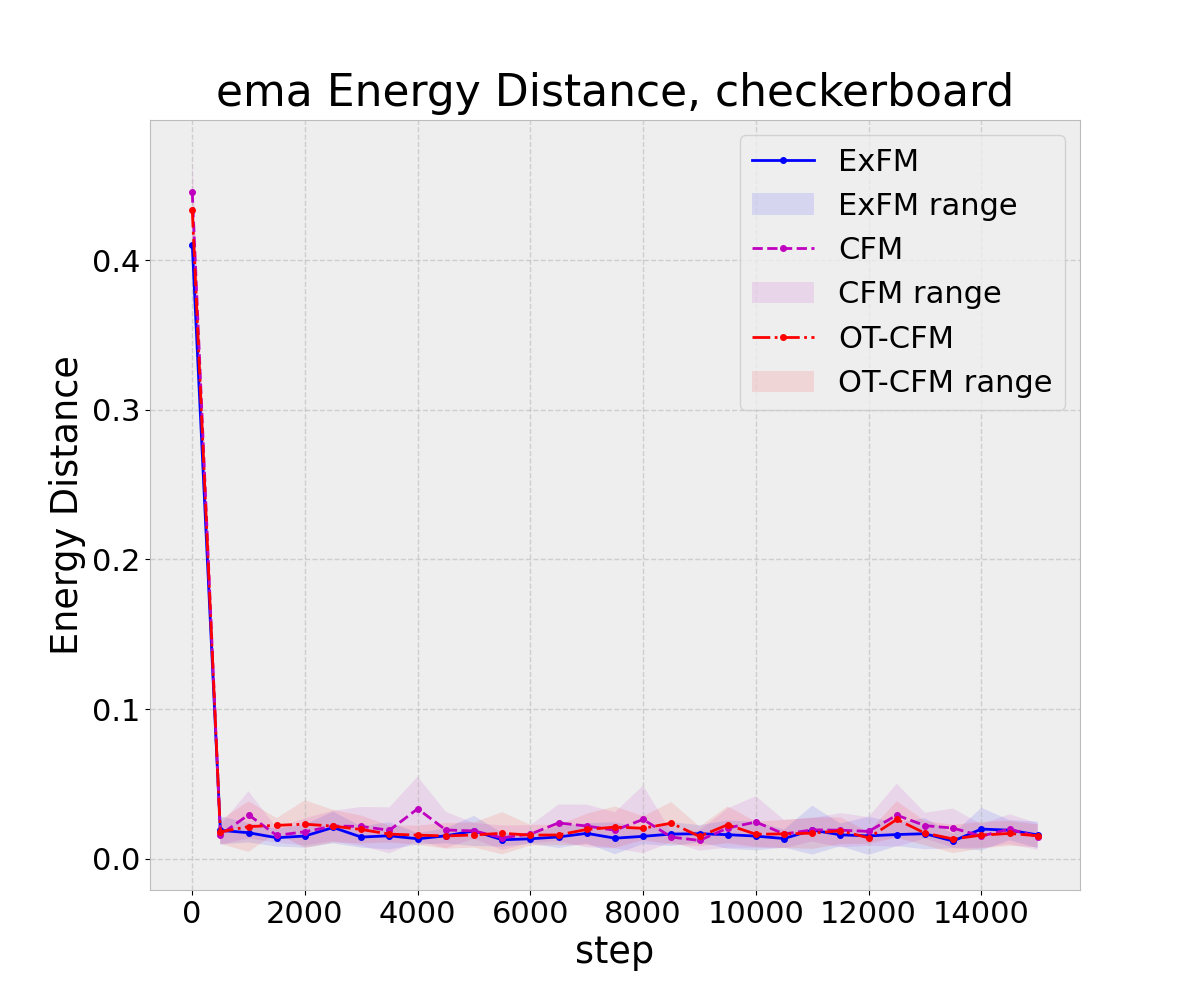}
\ic{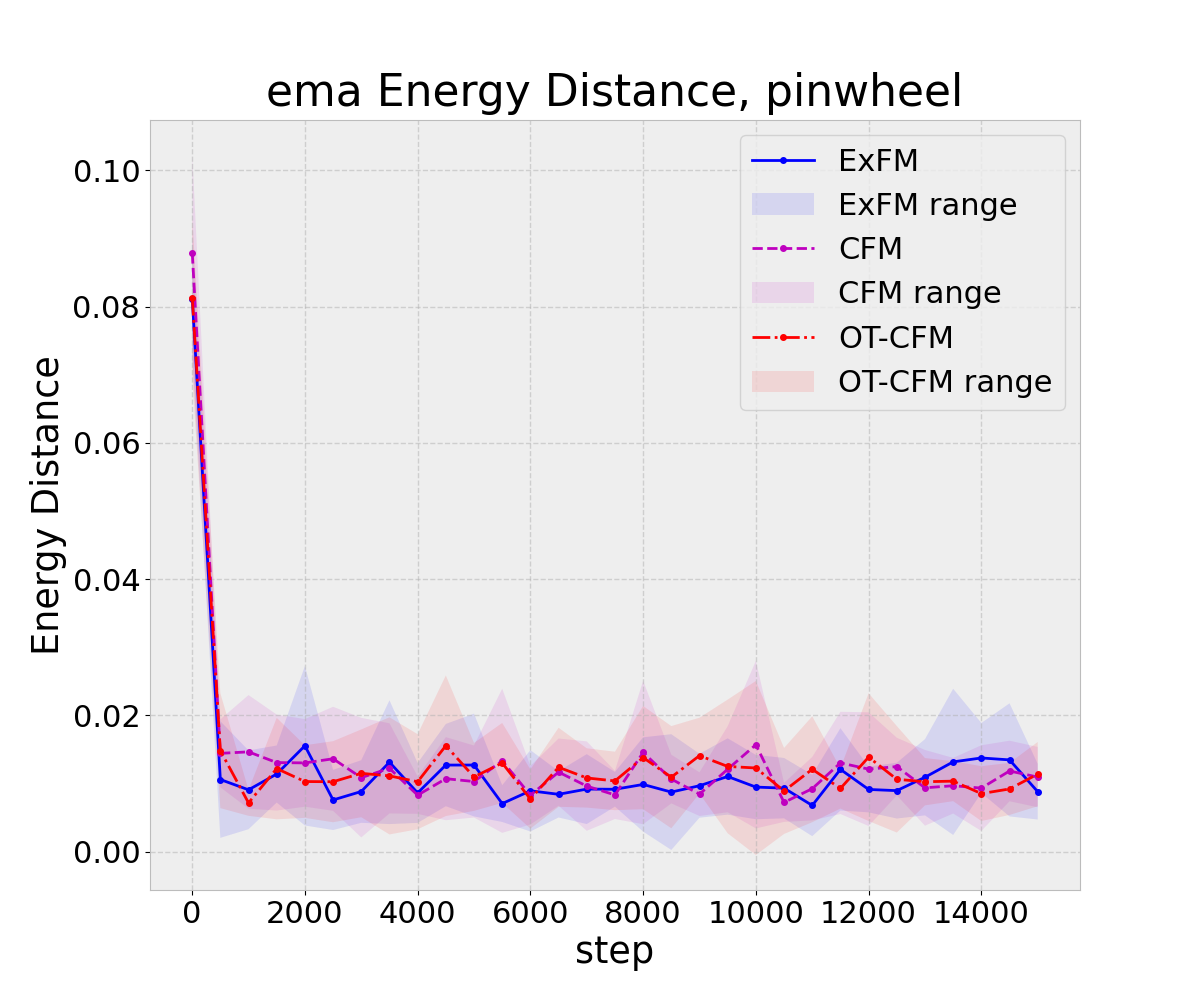}
\ic{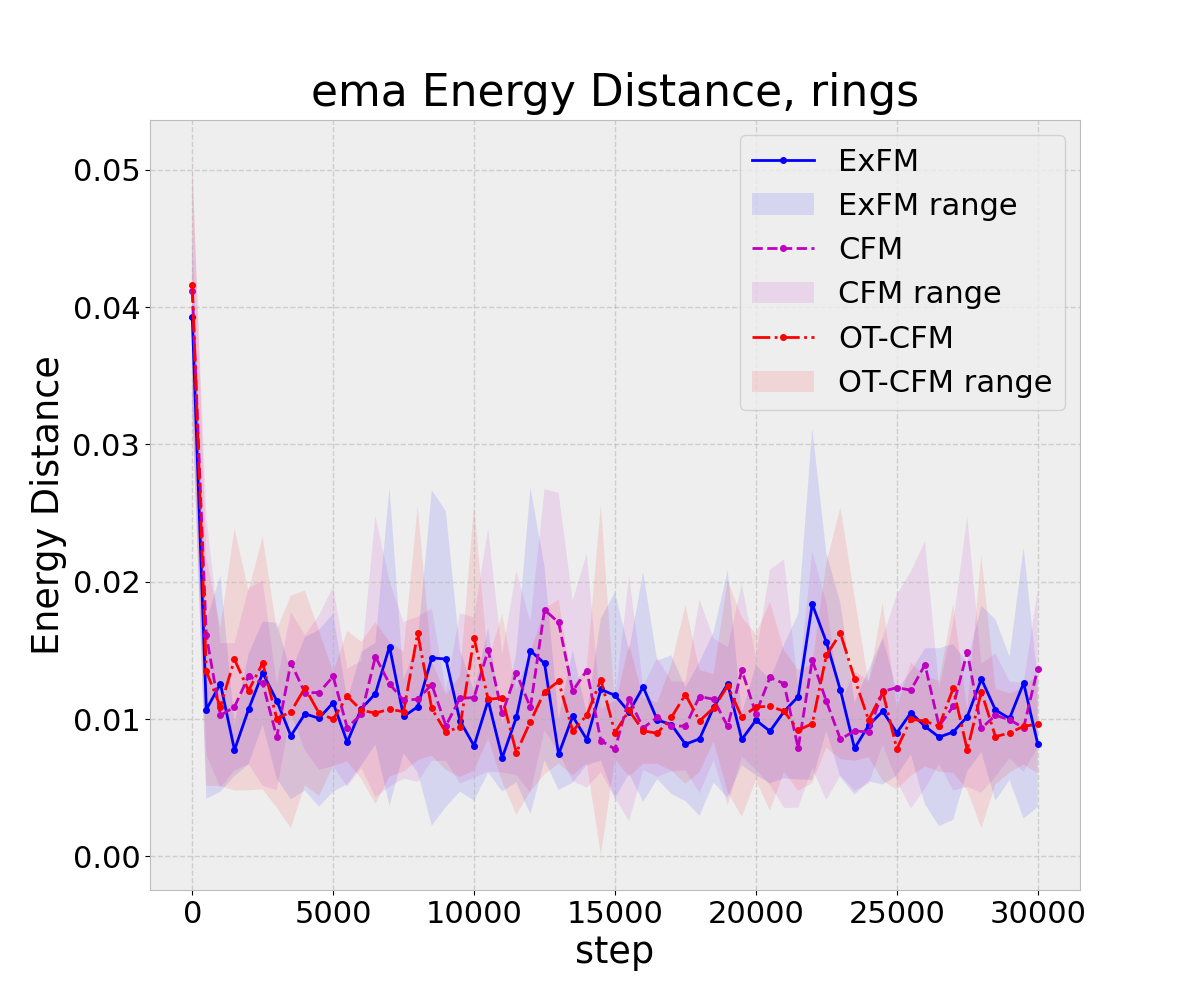}
\caption{Energy distance comparison for ExFM, CFM and OT-CFM methods for toy datasets for $15\,000$ learning steps, $30\,000$ learning steps for \texttt{rings} dataset.}
\end{figure*}

\begin{figure*}[!tb]
\centering
\label{fig:images_2d_loss}
\def\ic#1{\includegraphics[clip, width=0.32\linewidth]{_Training_loss_#1.png}}
\ic{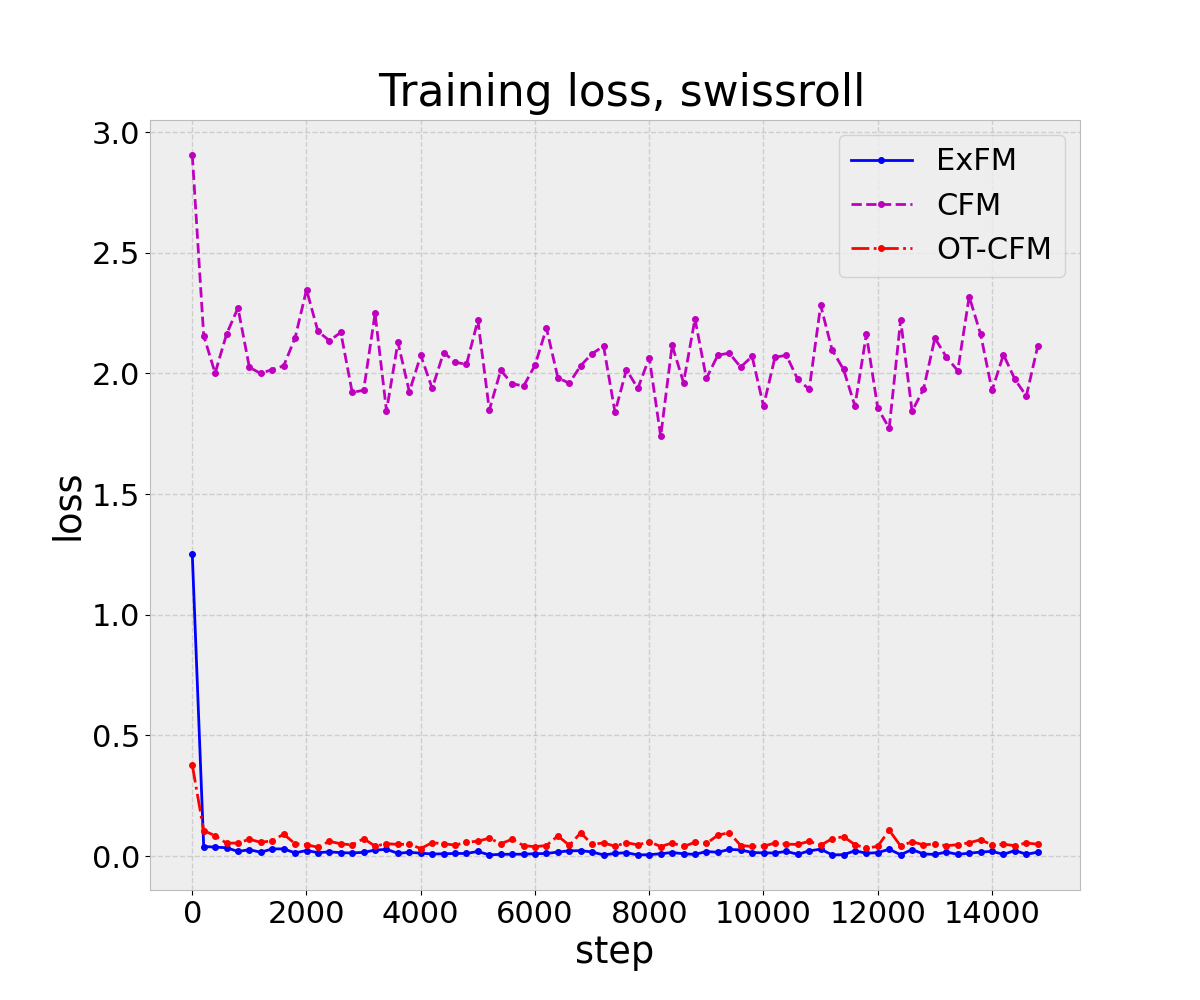}
\ic{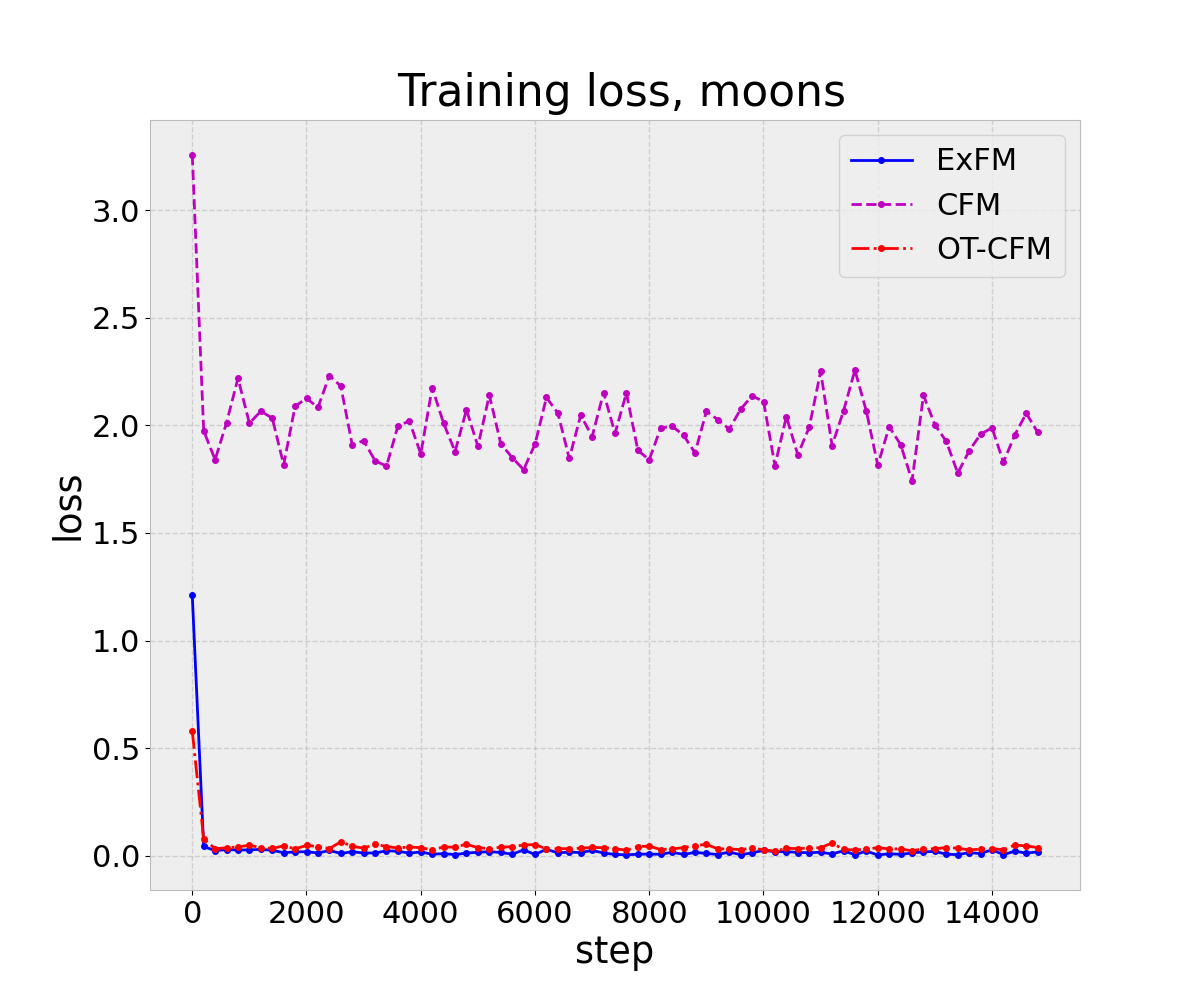}
\ic{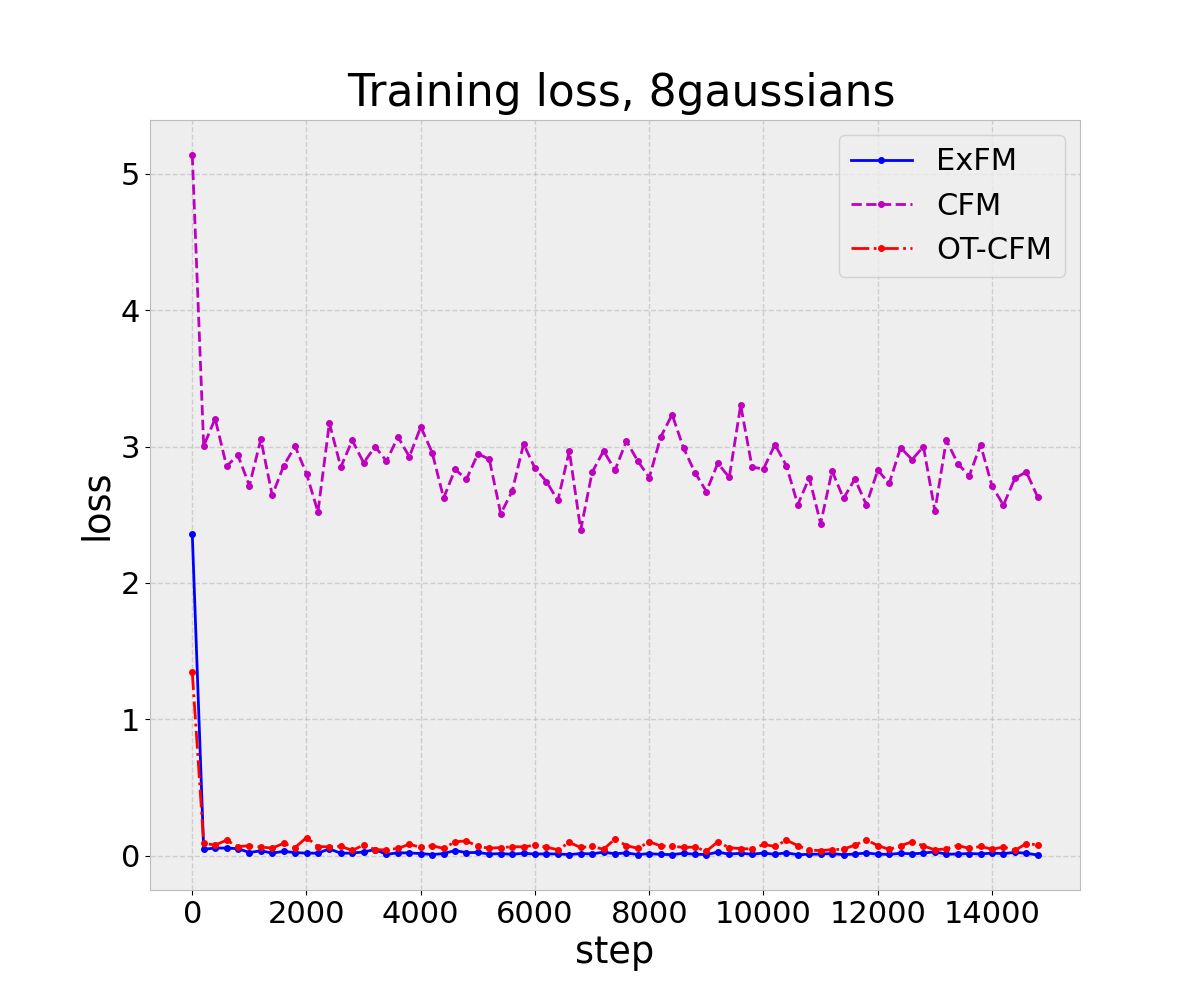}
\ic{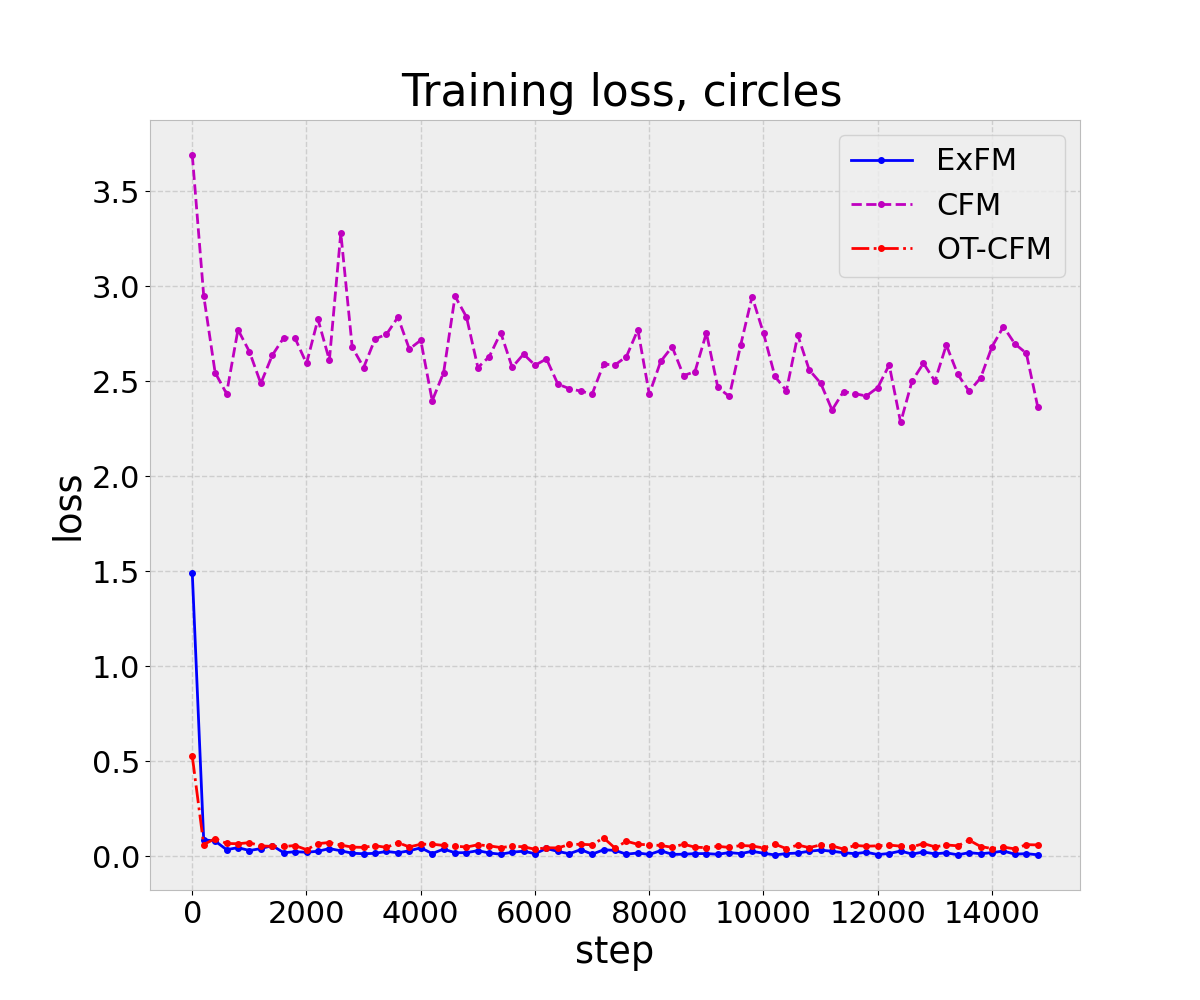}
\ic{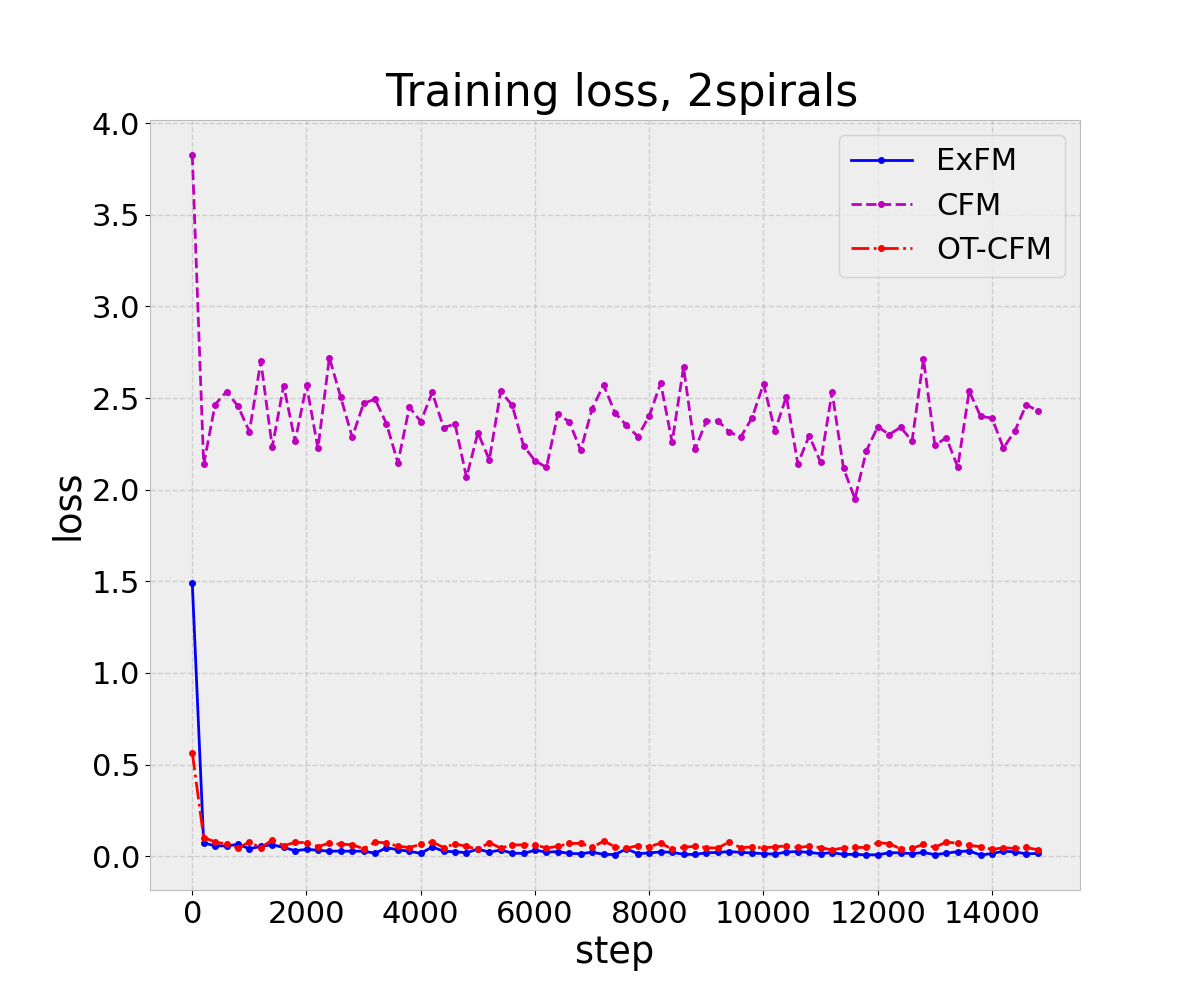}
\ic{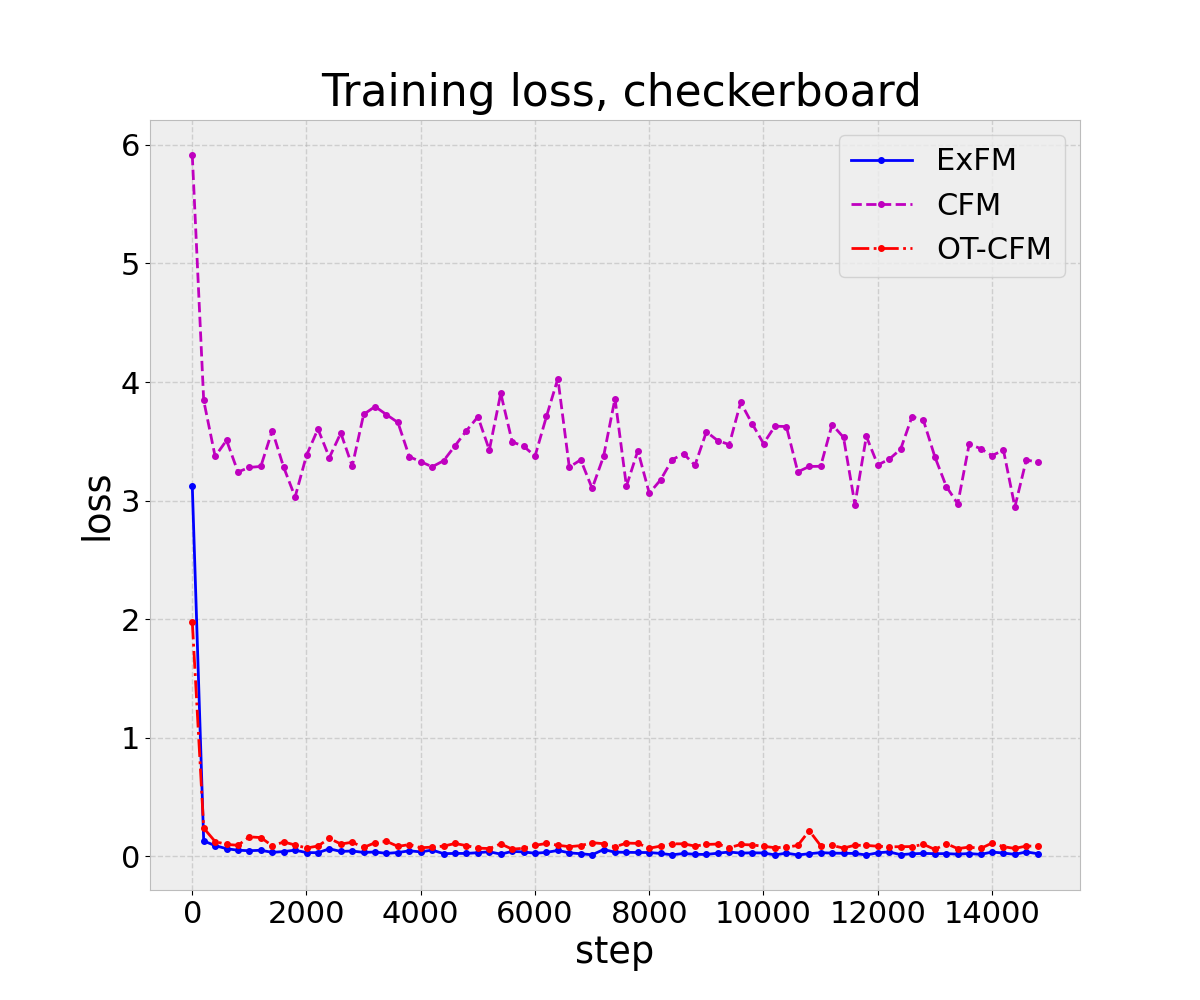}
\ic{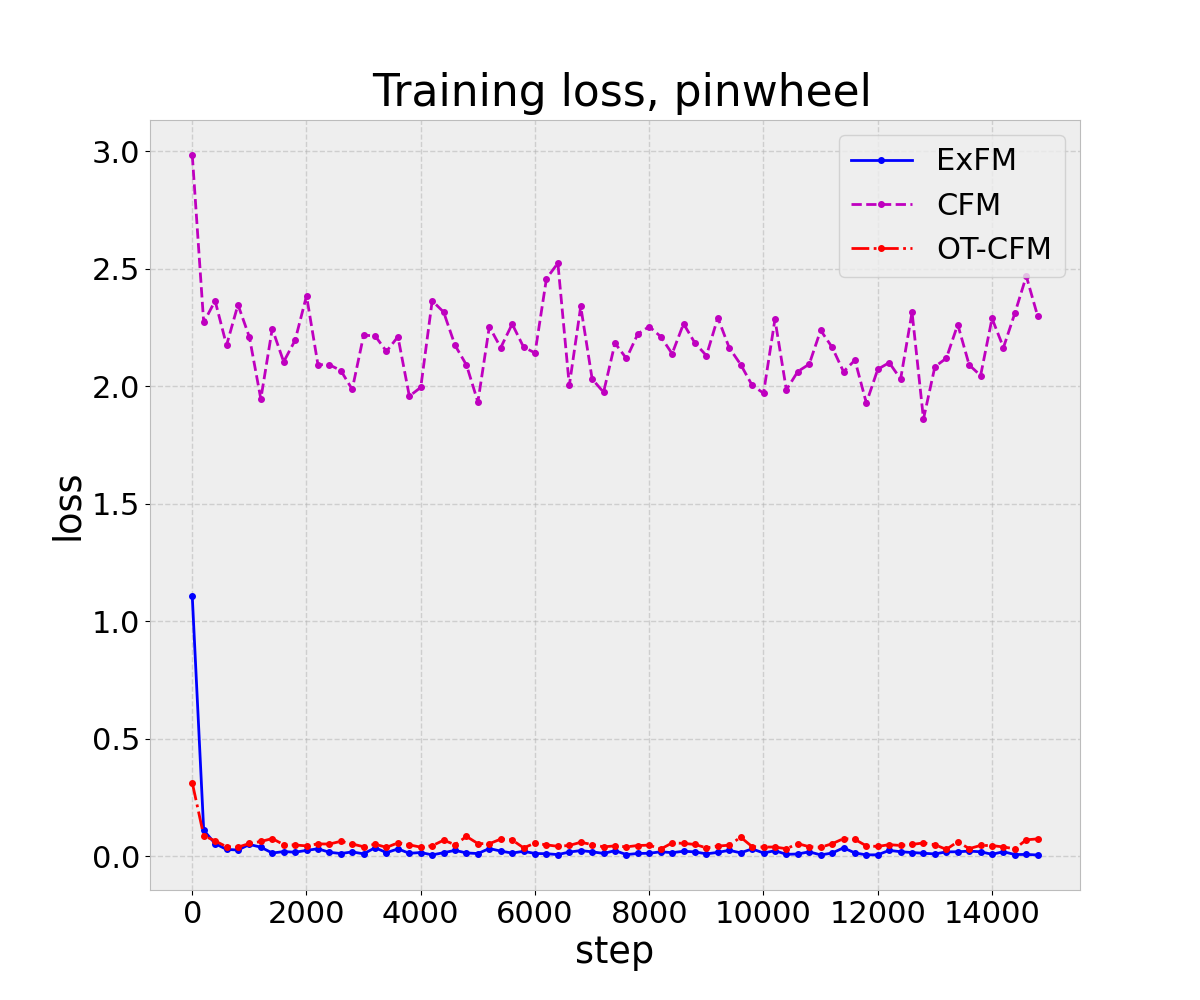}
\ic{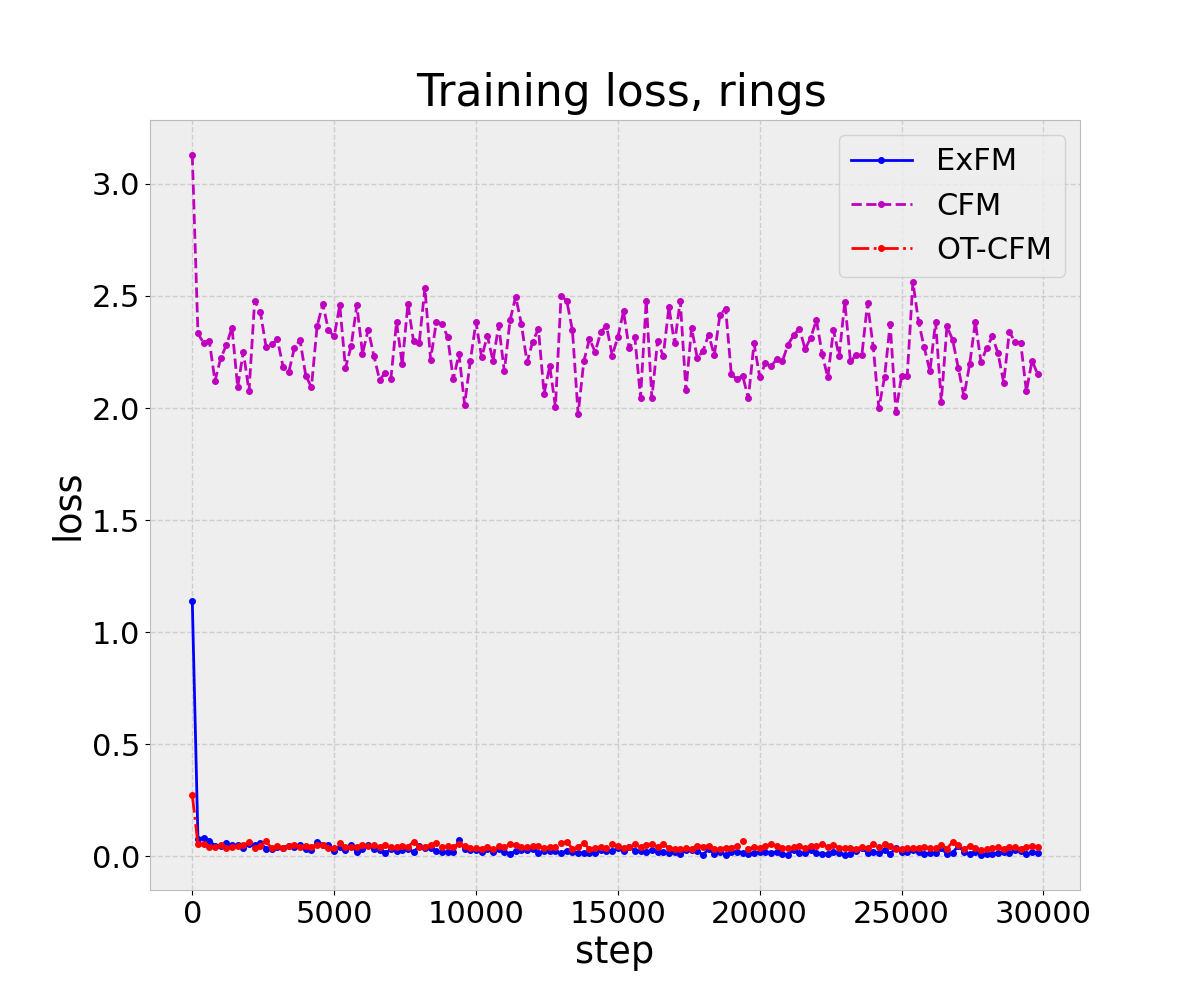}
\caption{Training loss comparison for ExFM, CFM and OT-CFM methods for toy datasets for $15\,000$ learning steps, $30\,000$ learning steps for \texttt{rings} dataset.}
\end{figure*}

\subsection{Tabular}
The \texttt{power} dataset (dimension = 6, train size = 1659917, test size = 204928) consisted of electric power consumption data from households over a period of 47 months. The \texttt{gas} dataset (dimension = 8, train size = 852174, test size = 105206) recorded readings from 16 chemical sensors exposed to gas mixtures. The \texttt{hepmass} dataset (dimension = 21, train size = 315123, test size = 174987) described Monte Carlo simulations for high energy physics experiments. The \texttt{minibone} (dimension = 43, train size = 29556, test size = 3648) dataset contained examples of electron neutrino and muon neutrino. Furthermore, we utilized the \texttt{BSDS300} dataset (dimension = 63, train size = 1000000, test size = 250000), which involved extracting random 8 x 8 monochrome patches from the BSDS300 datasets of natural images~\cite{937655}. 

These diverse multivariate datasets are selected to provide a comprehensive evaluation of performance across various domains. To maintain consistency, we followed the code available at the given GitHub link\footnote{\url{https://github.com/gpapamak/maf}} to ensure that the same instances and covariates were used for all the datasets.

To ensure the correctness of the experiments we conduct them with the same parameters. To train the model we use the same MultiLayer Perceptron model with ReLu activations, number of neurons and layers differed for the datasets along with the learning rate for the optimizer, that can be seen in Table~\ref{tab:params_tabular}. Same as for toy data, we use \texttt{Adam} as optimizer, EMA with $0.9$ rate and no learning rate scheduler. As in the pretrained step, we use separately training and testing sets for training the model and calculating metrics. We train the models for $10,000$ learning steps with batch size $N= 5000$ (\texttt{batch\_size}) and mini batches $n = 256$ elements (\texttt{mini\_batch\_size}).

For both 2D-toy and tabular data:
we take \hbox{$m=n$} time variable,
individual value of variable~$t$ corresponds to its pair~$(x_0,\, x_1)$.
The notations~$N$, $n$ and $m$ corresponds to those in Algorithm~\ref{alg:main}.
%The full training algorithm can be seen in . 
To perform sampling, we employed the function \texttt{odeint} with \texttt{dopri5} method from the python package \texttt{torchdiffeq} with \texttt{atol} and \texttt{rtol} equal $10^{-5}$.

\begin{table}[!htb]
\centering
\def\g{\bf}
\small
\caption{Learning parameters for Tabular datasets.}
\begin{sc}
\begin{tabular}{lrr}
\toprule
   Data & MLP layers & lr   \\
\midrule
        power & [512, 1024, 2048] & 1e-3 \\ 
        gas & [512, 1024,1024] & 1e-4 \\ 
        hepmass & [512, 1024] & 1e-3 \\ 
        bsds300 & [512, 1024,1024]& 1e-4 \\ 
        miniboone & [512, 1024] & 1e-3 \\ 
        
\bottomrule
\end{tabular}
\end{sc}
\label{tab:params_tabular}
\end{table}
Due to the inherent difficulty in visualizing tabular datasets, Negative Log-Likelihood (NLL) (Table~\ref{tab:CFM_ExFM_tabular}) metrics were employed to quantitatively compare the performance of ExFM, CFM, and OT-CFM methods. Figure \ref{fig:images_tabular_loss} presents a comparison of the training losses incurred by each method. As can be observed, all three methods exhibit rapid convergence. Notably, our proposed method demonstrates superior training stability compared to the baseline CFM, as evidenced by its smoother loss curve.

In Figure \ref{fig:images_tabular}, we illustrate the NLL values recorded across training steps for all three methods on various datasets. While our method achieves competitive performance, it occasionally yields slightly lower NLL scores compared to OT-CFM on specific datasets.

\begin{figure*}[!tb]
\centering
\label{fig:images_tabular_loss}
\def\ic#1{\includegraphics[clip, width=0.32\linewidth]{_Training_loss_#1.png}}
\ic{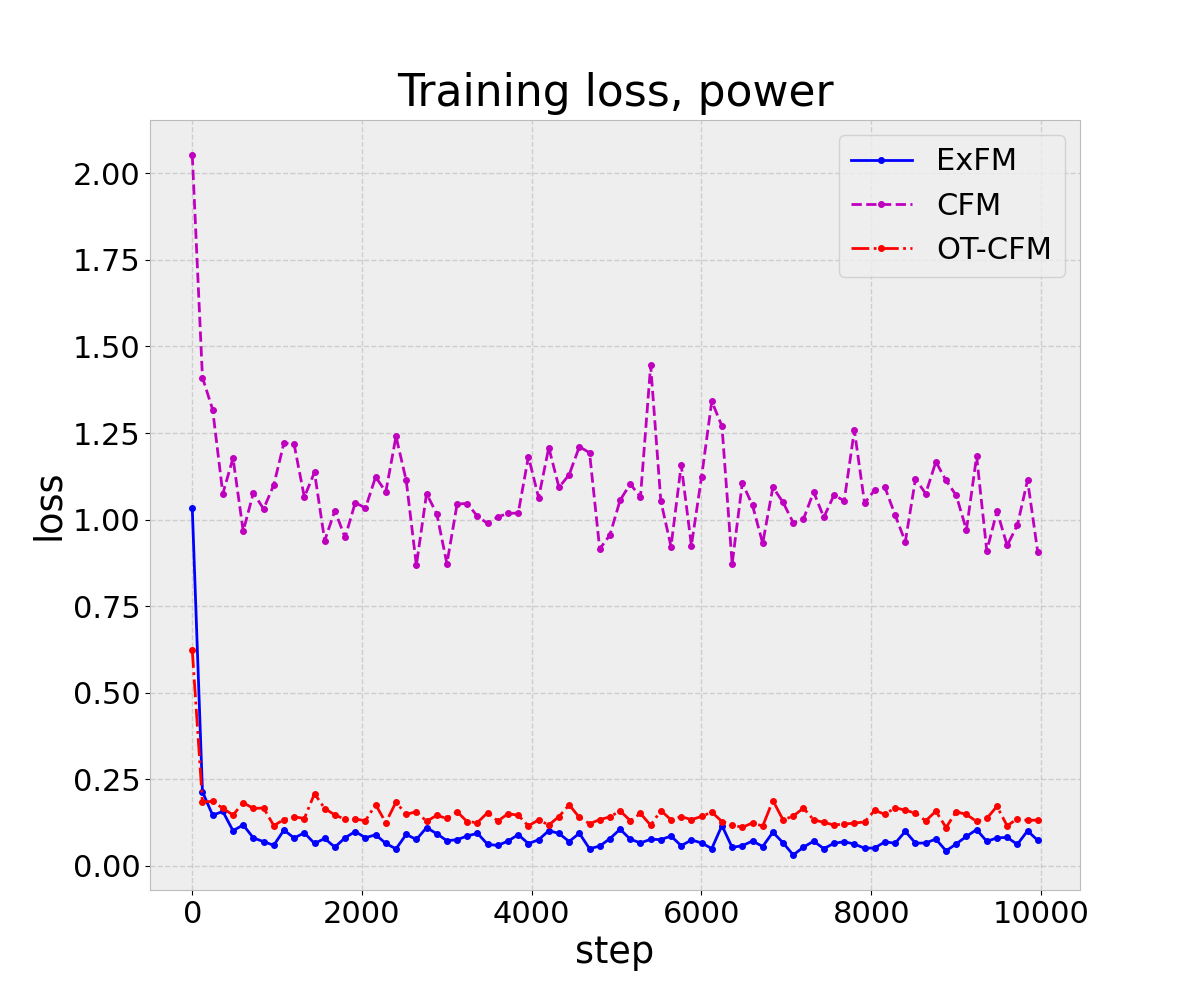}
\ic{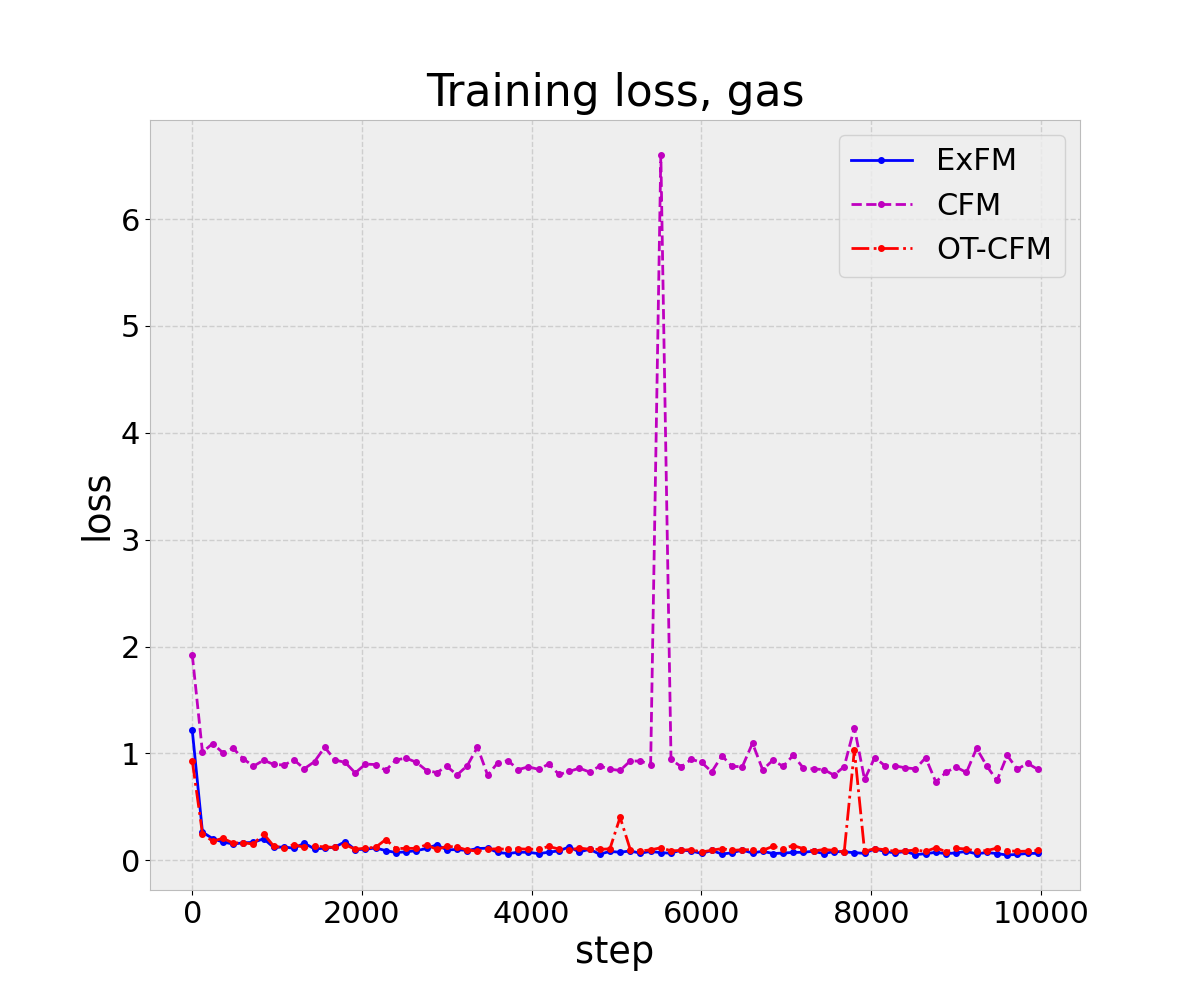}
\ic{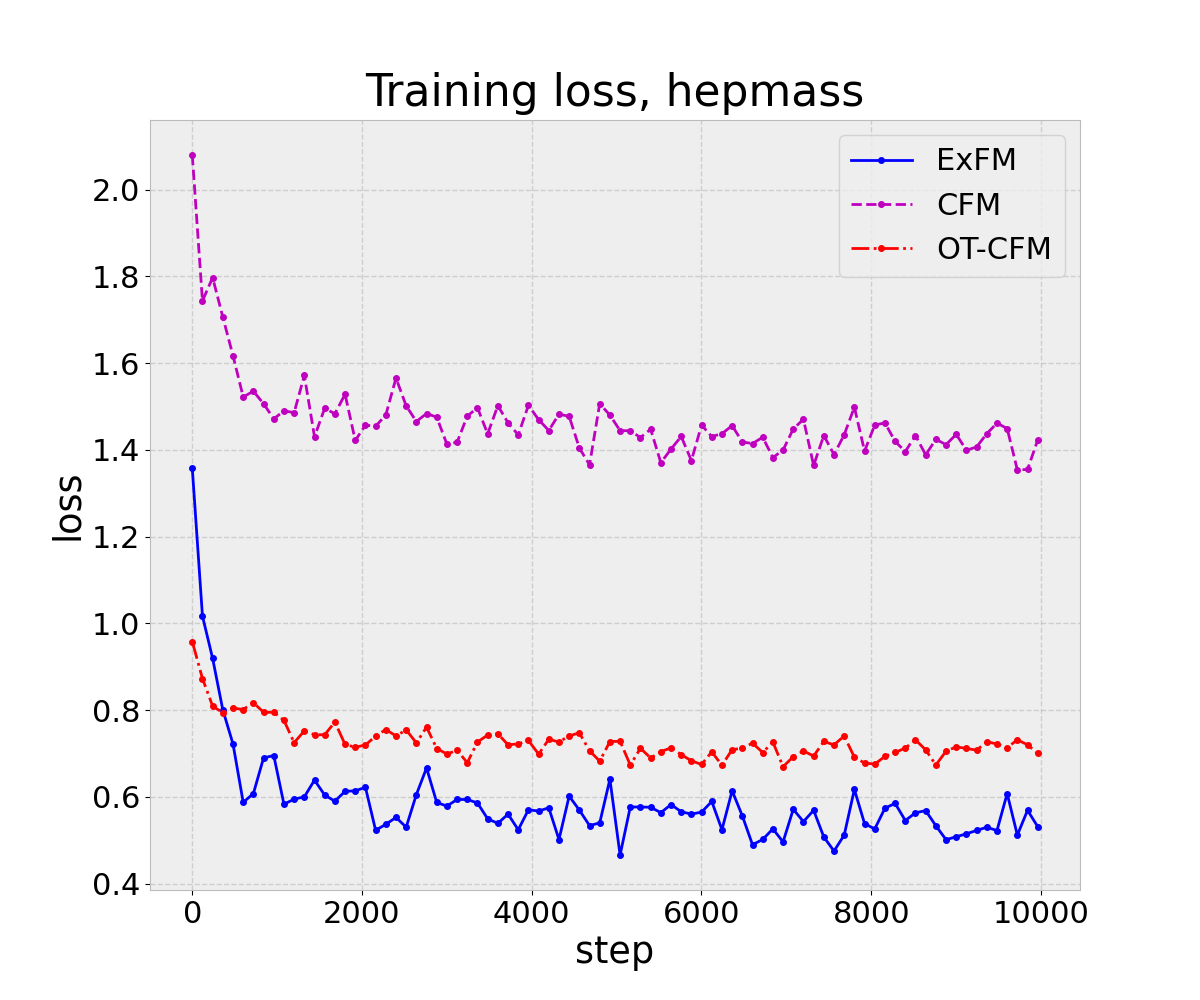}
\ic{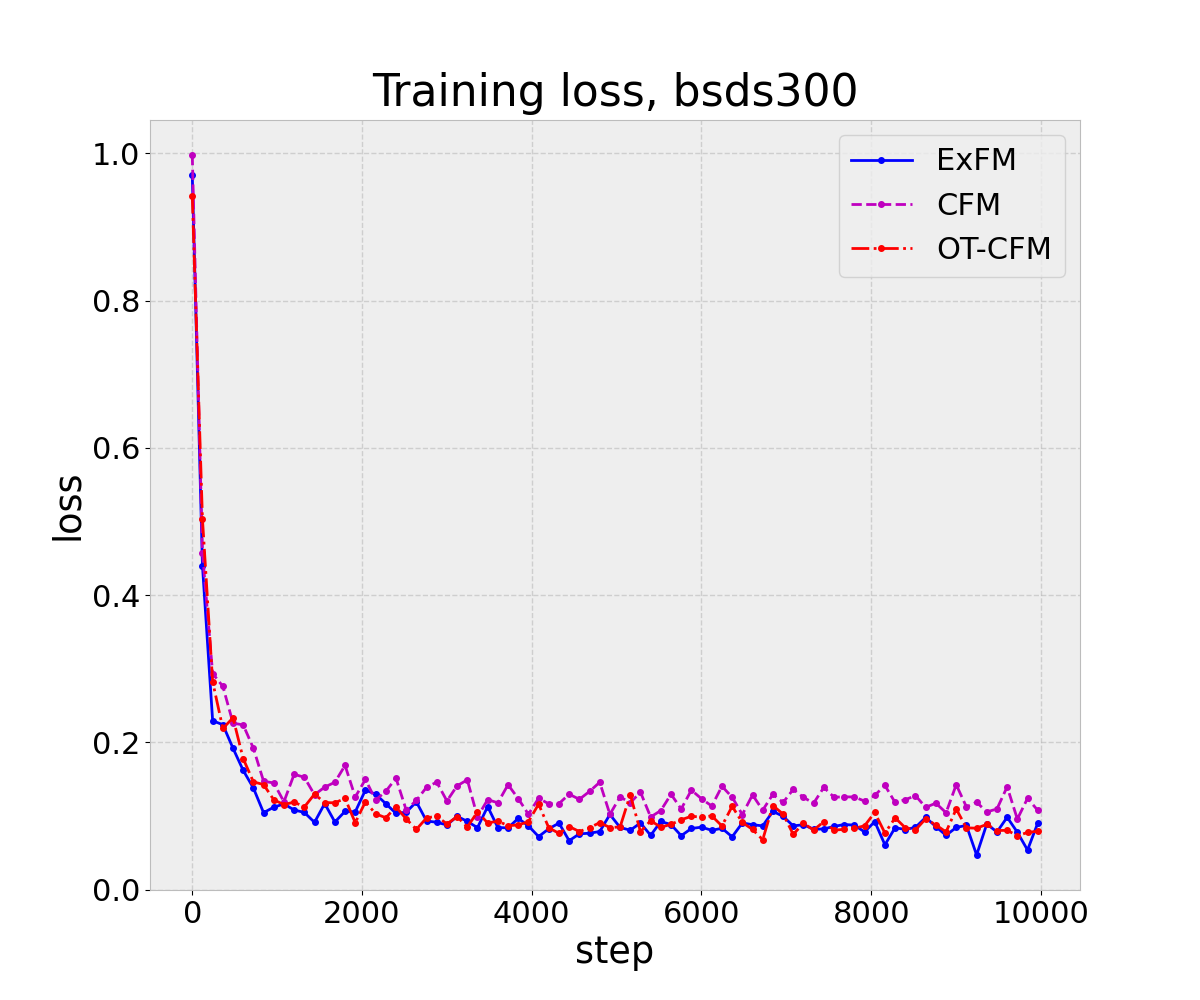}
\ic{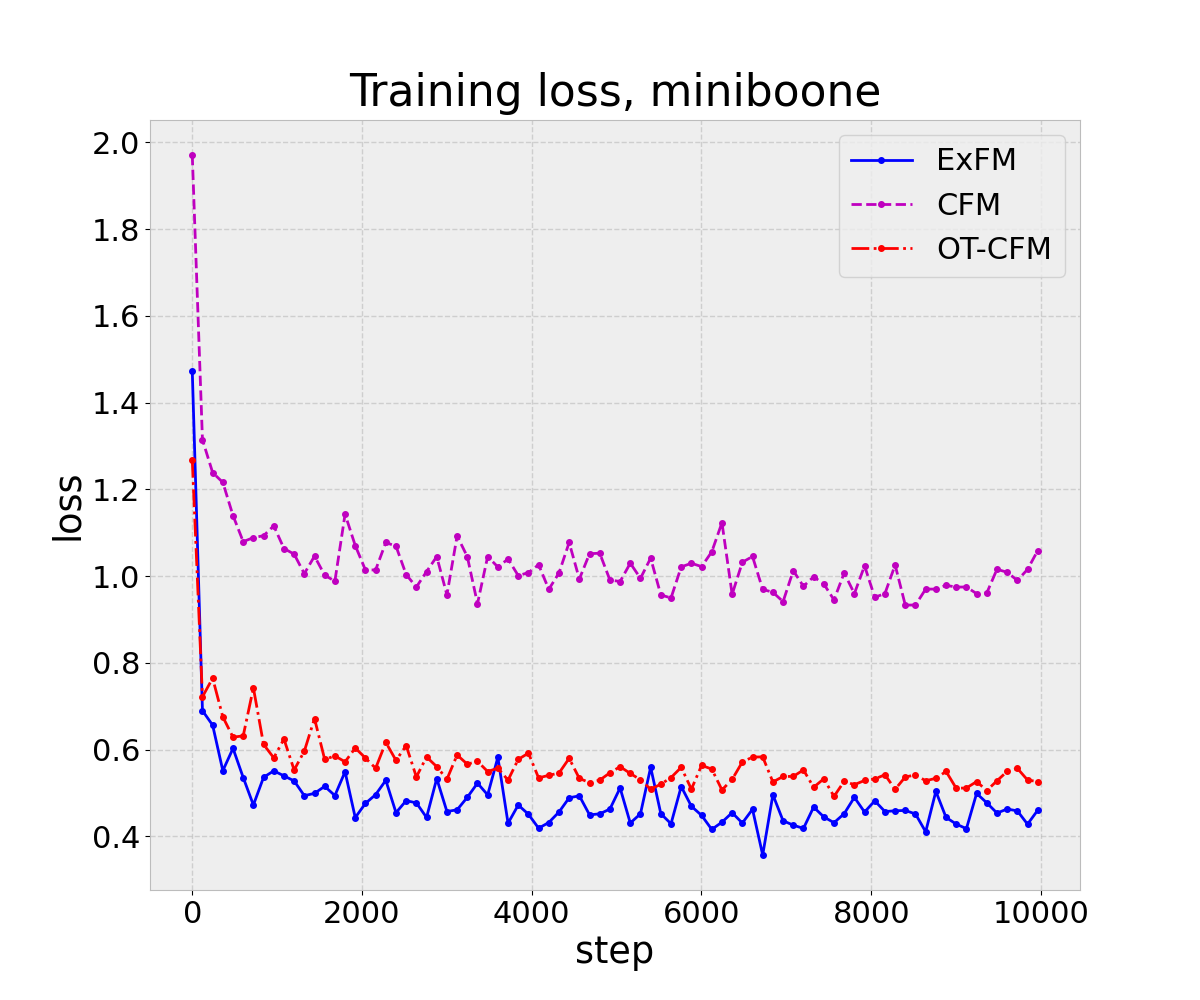}
\caption{Training loss comparison for tabular datasets for ExFM, CFM and OT-CFM methods over $10\,000$ learning steps.}

\end{figure*}

\begin{figure*}[!tb]
\centering
\label{fig:images_tabular}
\def\ic#1{\includegraphics[clip, width=0.32\linewidth]{ema_NLL_#1.png}}
\ic{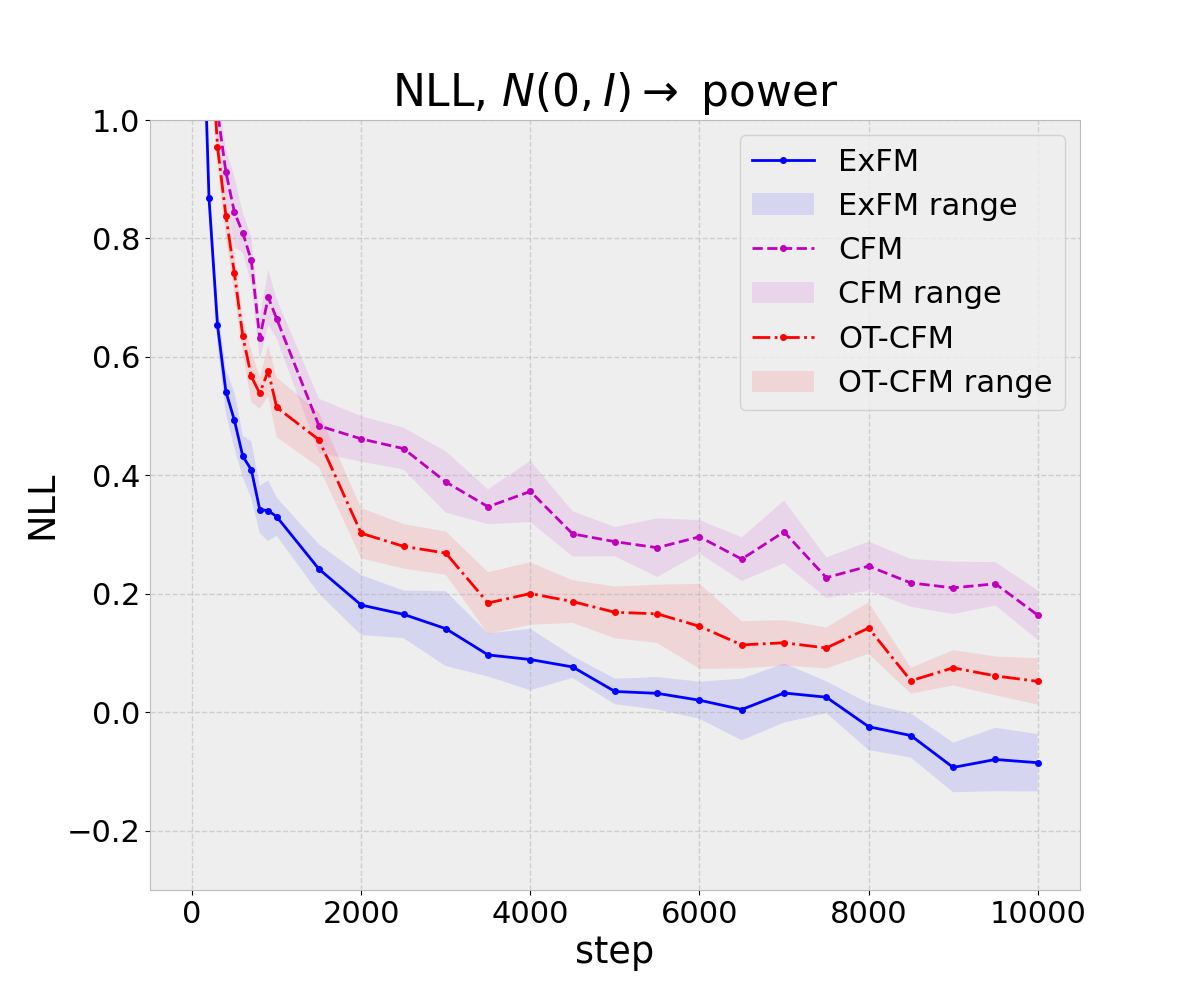}
\ic{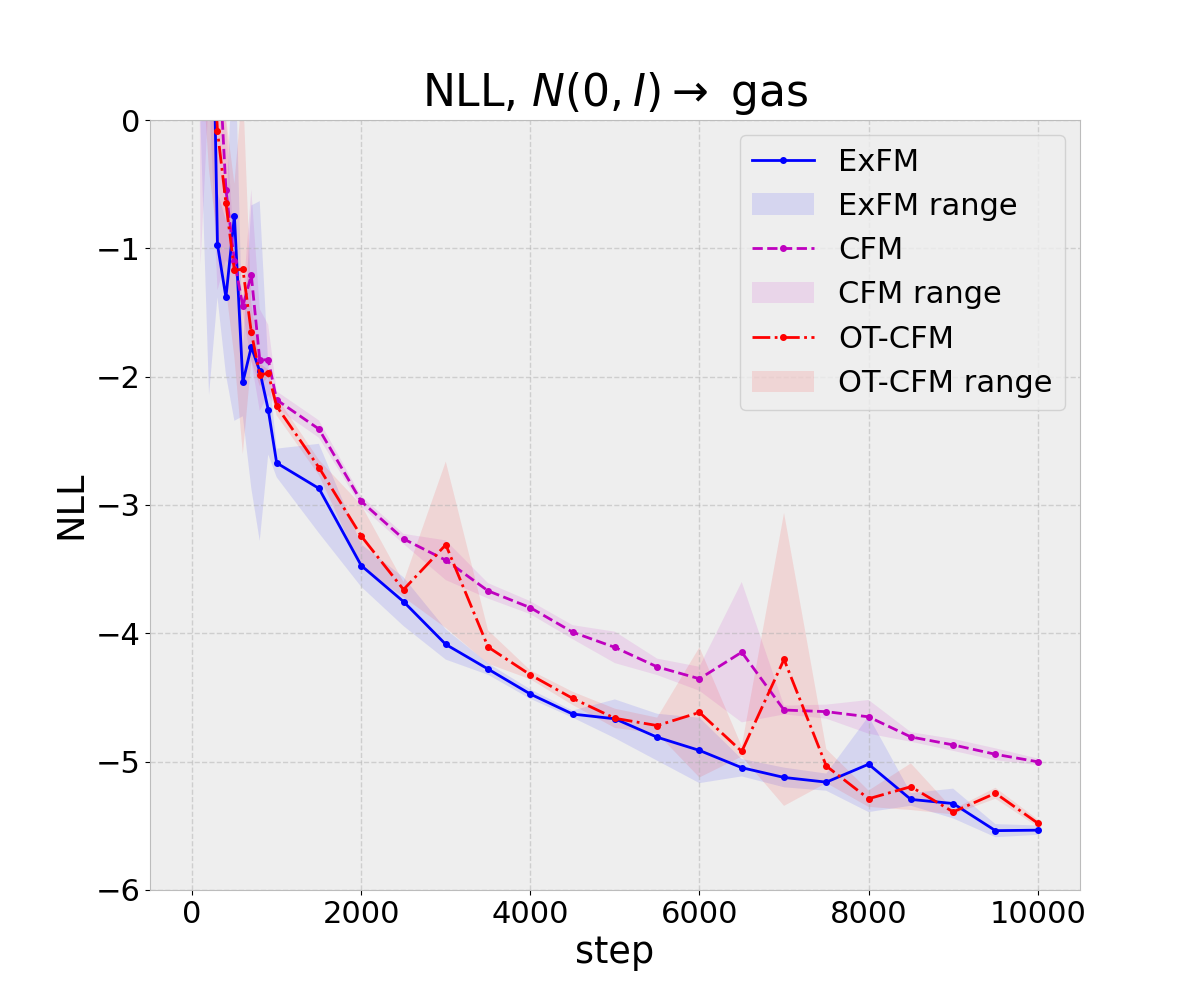}
\ic{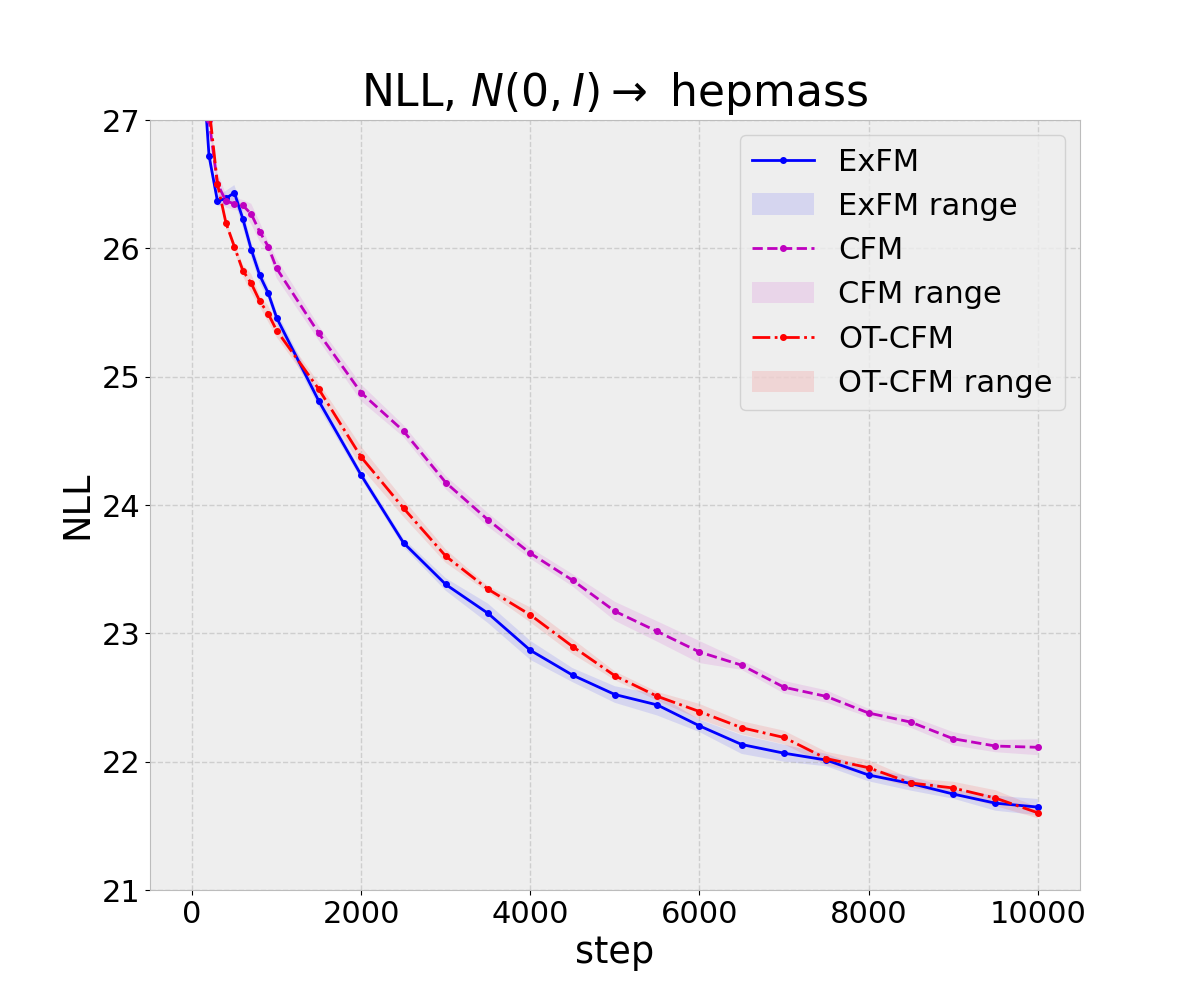}
\ic{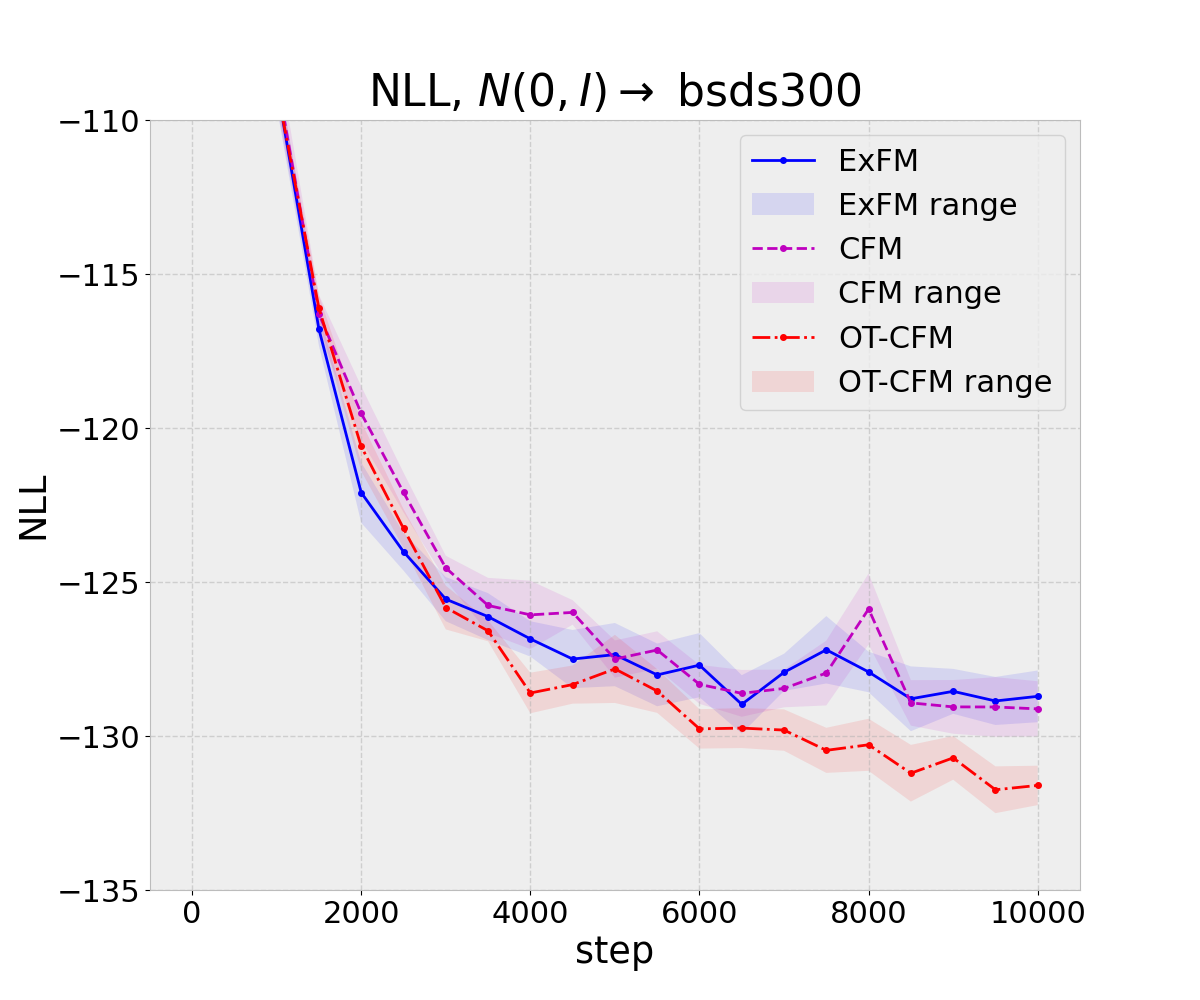}
\ic{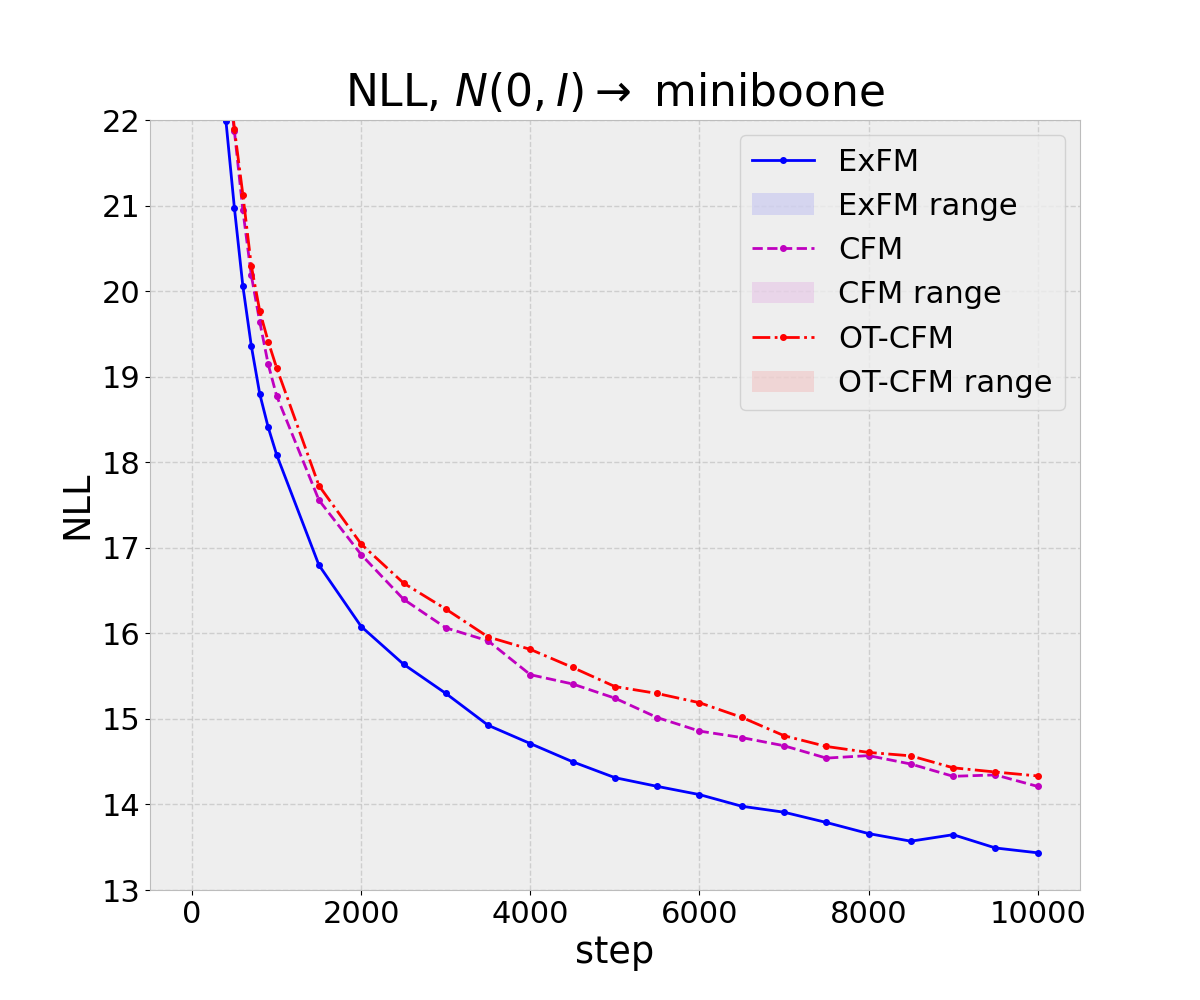}
%\includegraphics[clip, width=0.32\linewidth]{ema_NLL_power.png}
%\includegraphics[clip, width=0.32\linewidth]{ema_NLL_gas.png}
%\includegraphics[clip, width=0.32\linewidth]{ema_NLL_hepmass.png}
%\includegraphics[clip, width=0.32\linewidth]{ema_NLL_bsds300.png}
%\includegraphics[clip, width=0.32\linewidth]{ema_NLL_miniboone.png}
%\vspace*{-2ex}
\caption{NLL comparison for ExFM, CFM and OT-CFM methods over $10\,000$ learning steps,
mean and std for range taken from 10 sampling iterations.}

\end{figure*}

\subsection{ExFM-S evaluation}
The models were assessed using four toy datasets of two dimensions each. A three-layer MLP network was utilized, featuring SeLU activations and a hidden dimension of 64. Optimization was carried out using the AdamW optimizer with a learning rate of $10^{-3}$ and a weight decay of $10^{-5}$. The model was trained over $2\,000$ iterations with a batch size of $128$. Inference was conducted using the Euler solver for Ordinary Differential Equations (ODE) with $100$ steps. 
To validate the models, the POT library was employed to compute the Wasserstein distance based on $4\,000$ samples. The experiments were performed on a single Nvidia H100 GPU with 80gb memory.

\subsection{CIFAR 10 and MNIST}
We conducted experiments related to high dimensional data, the parameters for training were taken from the open-source code\footnote{\url{https://github.com/atong01/conditional-flow-matching}}
from the works~\cite{Tong2024,tong2024improving}. For training we used proposed U-Net model. We saved the leverage of additional heuristics(EMA, lr scheduler) and their parameters. For the final evaluation of CIFAR 10 dataset we used Fréchet inception distance (FID) metrics, and the values can be seen in Table~\ref{tab:fid_cifar10}, and we also evaluated FID during training for different learning steps, that can be seen in Table~\ref{tab:fid_comp_cifar10} and in Figure~\ref{fig:FID}. 

\begin{table}[!htb]
\centering
\def\g{\bf}
\small
\caption{FID comparison for 4 sampling iterations, 400 000 learning steps.}
\begin{sc}
\begin{tabular}{lr}
\toprule
   Method & FID \\
\midrule
        ExFM & \g 3.686 $\pm$ 0.029 \\ 
        CFM & 3.727 $\pm$ 0.026 \\ 
        OT-CFM & 3.843 $\pm$ 0.033 \\ 
        
\bottomrule
\end{tabular}
\end{sc}
\label{tab:fid_cifar10}
\end{table}

\begin{table}[!htb]
    \centering
    \def\g{\bf}
    \caption{FID comparison for ExFM, CFM and OT-CFM methods over 400 000 learning steps, mean and std taken from 4 sampling iterations.}
    \begin{tabular}{lrrr}
    \toprule
        Step & ExFM FID & CFM FID & OT-CFM FID \\ 
    \midrule
        0 & 447.256 $\pm$ 0.116 & 447.106 $\pm$ 0.130 & 447.091 $\pm$ 0.081 \\ 
        20000 & 281.060 $\pm$ 0.243 & 275.044 $\pm$ 0.123 & 281.499 $\pm$ 0.287 \\ 
        40000 & 52.050 $\pm$ 0.245 & 51.436 $\pm$ 0.142 & 45.976 $\pm$ 0.109 \\ 
        60000 & \g 9.125 $\pm$ 0.060 & 9.181 $\pm$ 0.035 & 10.358 $\pm$ 0.054 \\ 
        80000 & \g 6.624 $\pm$ 0.053 & 6.978 $\pm$ 0.062 & 7.492 $\pm$ 0.050 \\ 
        100000 & \g 5.641 $\pm$ 0.048 & 5.894 $\pm$ 0.045 & 6.299 $\pm$ 0.031 \\ 
        120000 & \g 5.085 $\pm$ 0.031 & 5.247 $\pm$ 0.051 & 5.558 $\pm$ 0.017 \\ 
        140000 & \g 4.766 $\pm$ 0.036 & 4.902 $\pm$ 0.053 & 5.120 $\pm$ 0.043 \\ 
        160000 & \g 4.486 $\pm$ 0.054 & 4.593 $\pm$ 0.068 & 4.828 $\pm$ 0.046 \\ 
        180000 & \g 4.294 $\pm$ 0.023 & 4.447 $\pm$ 0.045 & 4.576 $\pm$ 0.051 \\ 
        200000 & \g 4.180 $\pm$ 0.029 & 4.204 $\pm$ 0.013 & 4.434 $\pm$ 0.031 \\ 
        220000 & \g 4.022 $\pm$ 0.036 & 4.182 $\pm$ 0.024 & 4.331 $\pm$ 0.036 \\ 
        240000 & \g 3.925 $\pm$ 0.028 & 4.037 $\pm$ 0.036 & 4.227 $\pm$ 0.050 \\ 
        260000 & \g 3.852 $\pm$ 0.047 & 3.937 $\pm$ 0.018 & 4.125 $\pm$ 0.015 \\ 
        280000 & \g 3.842 $\pm$ 0.053 & 3.870 $\pm$ 0.040 & 4.056 $\pm$ 0.029 \\ 
        300000 & \g 3.758 $\pm$ 0.032 & 3.788 $\pm$ 0.024 & 4.017 $\pm$ 0.029 \\ 
        320000 & \g 3.749 $\pm$ 0.029 & 3.792 $\pm$ 0.034 & 3.937 $\pm$ 0.052 \\ 
        340000 & \g 3.724 $\pm$ 0.042 & 3.747 $\pm$ 0.033 & 3.897 $\pm$ 0.037 \\ 
        360000 & \g 3.714 $\pm$ 0.022 & 3.751 $\pm$ 0.041 & 3.875 $\pm$ 0.015 \\ 
        380000 & \g 3.707 $\pm$ 0.028 & 3.754 $\pm$ 0.020 & 3.917 $\pm$ 0.037 \\ 
        400000 & \g 3.686 $\pm$ 0.029 & 3.727 $\pm$ 0.026 & 3.843 $\pm$ 0.033 \\  

    \bottomrule
    \end{tabular}
    
    \label{tab:fid_comp_cifar10}
\end{table}

Our proposed method demonstrates competitive performance on the evaluated datasets. Notably, it consistently achieves slightly better results compared to OT-CFM. This observation aligns with the assumption that highlight the limitations of OT when dealing with high-dimensional data. Figures \ref{fig:Training_loss_CIFAR} and \ref{fig:Training_loss_MNIST} illustrate the training loss curves for CIFAR 10 and MNIST, respectively. As evident from the figures, our method exhibits a clear advantage in terms of achieving lower training losses throughout the training process. This suggests that our method converges more effectively and is potentially more stable compared to OT-CFM and CFM. Visuals of generated samples for CIFAR 10 dataset are included in Figure \ref{fig:CIFAR10} and for MNIST dataset in Figure \ref{fig:MNIST}.

\begin{figure*}[!tb]
\centering
\label{fig:Training_loss_CIFAR}
%\hspace*{-9ex}
\includegraphics[clip, width=1\linewidth]{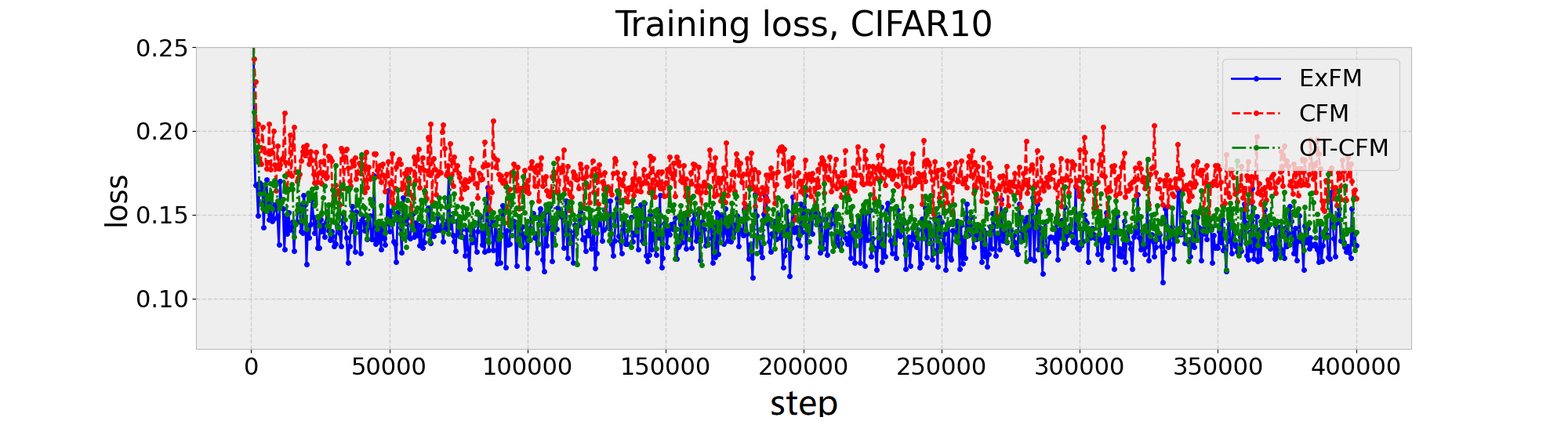}
\vspace*{-2ex}
\caption{Training loss comparison for ExFM, CFM and OT-CFM methods, CIFAR-10 dataset.}

\end{figure*}

\begin{figure*}[!tb]
\centering
\label{fig:Training_loss_MNIST}
%\hspace*{-9ex}
\includegraphics[clip, width=1\linewidth]{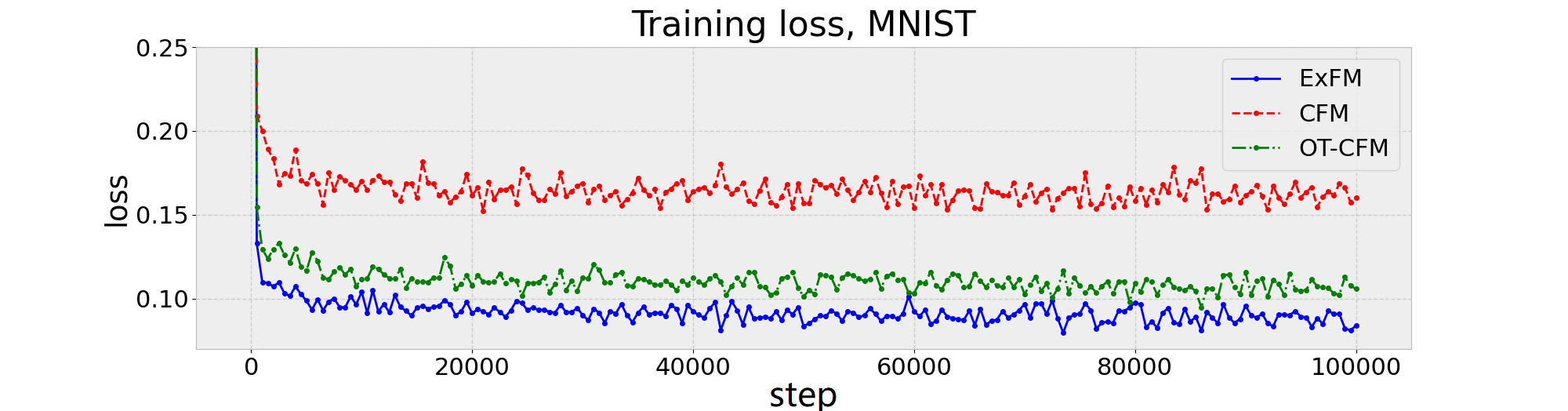}
\vspace*{-2ex}
\caption{Training loss comparison for ExFM, CFM and OT-CFM methods, MNIST dataset.}

\end{figure*}

\begin{figure*}[!tb]
\centering
\label{fig:FID}
\includegraphics[width=1\linewidth]{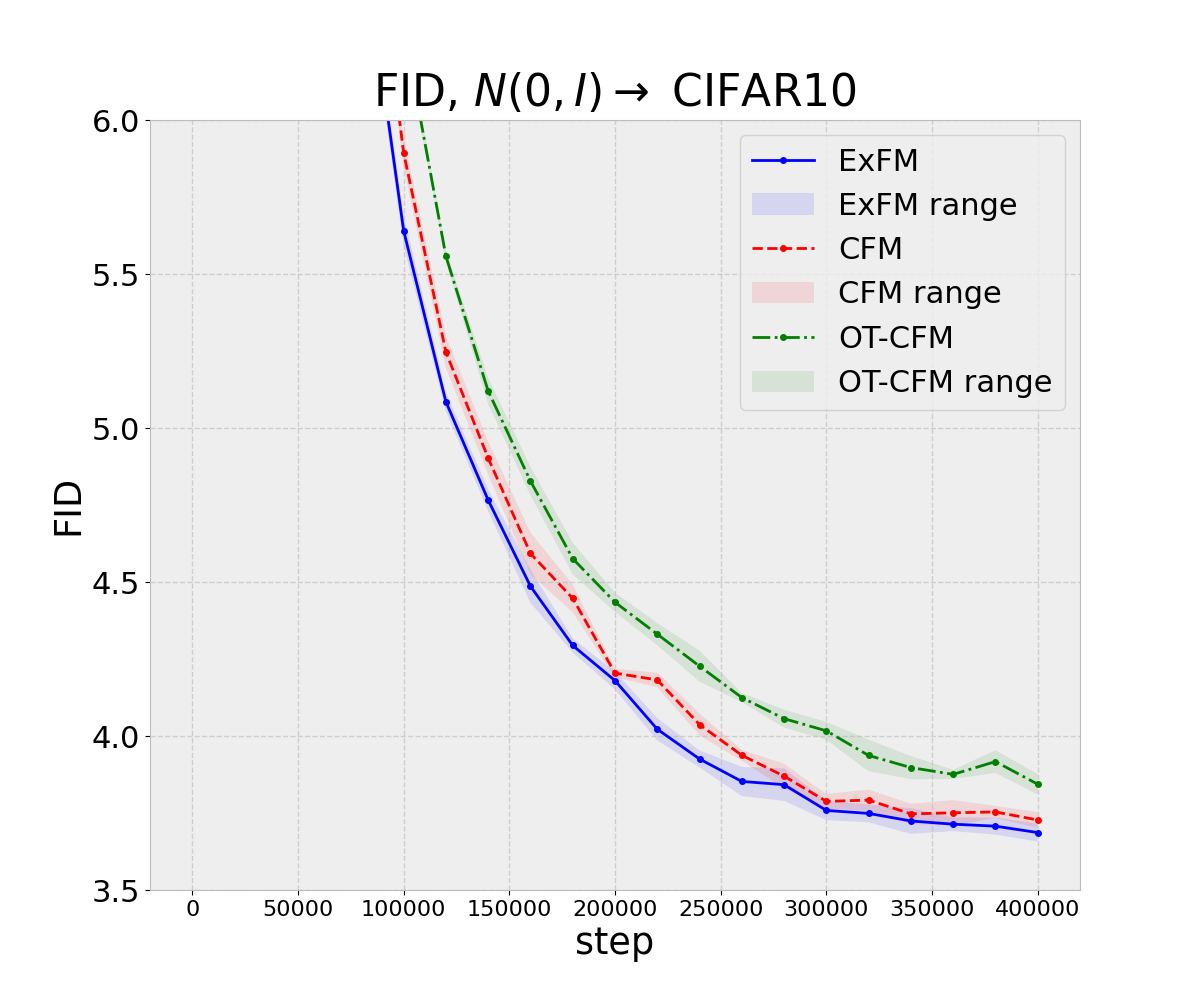}
\vspace*{-2ex}
\caption{FID comparison for ExFM, CFM and OT-CFM methods, CIFAR-10 dataset.}

\end{figure*}

\begin{figure*}[!tb]
\centering
\label{fig:CIFAR10}
\includegraphics[width=1\linewidth]{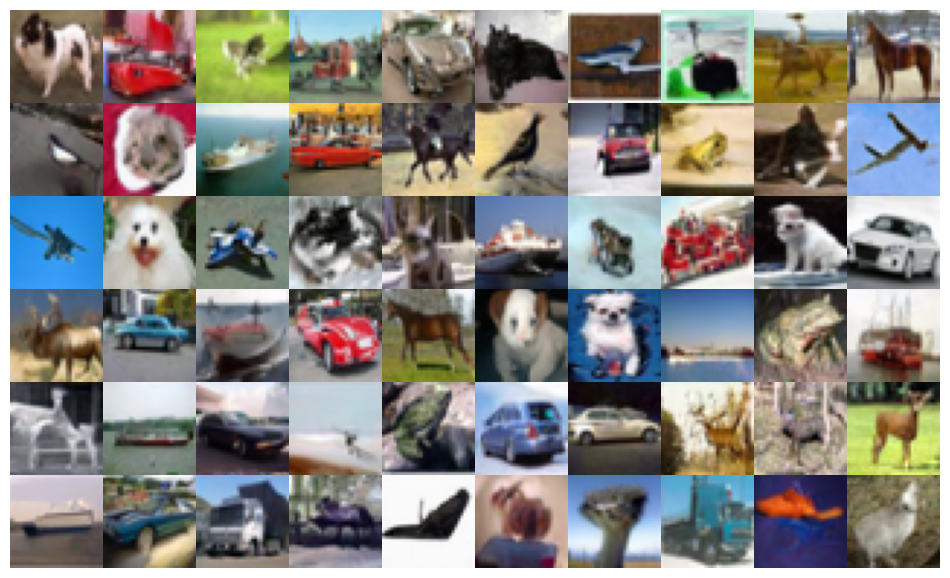}
\vspace*{-2ex}
\caption{Sampled images from ExFM method, CIFAR-10 dataset.}

\end{figure*}

\begin{figure*}[!tb]
\centering
\label{fig:MNIST}
\includegraphics[width=1\linewidth]{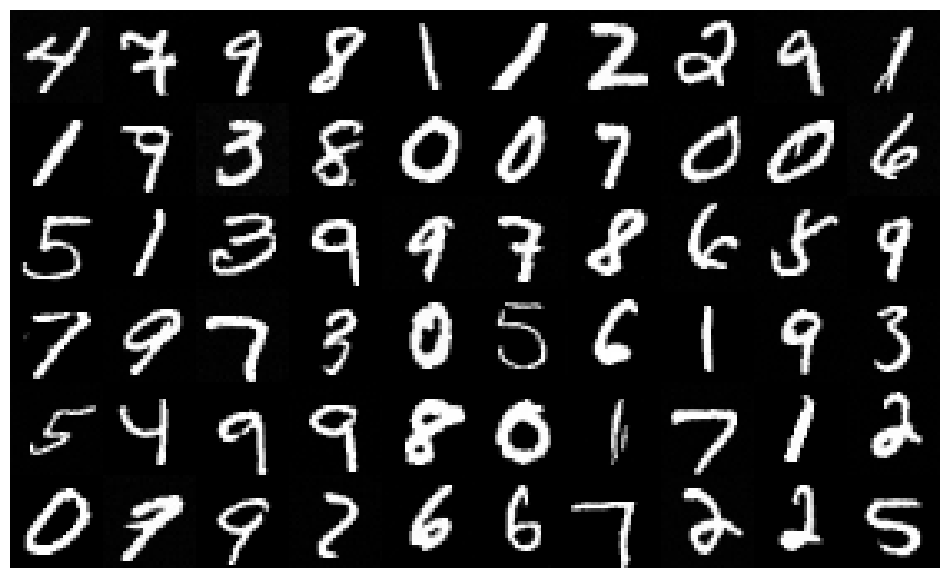}
\vspace*{-2ex}
\caption{Sampled images from ExFM method, MNIST dataset.}

\end{figure*}

\subsection{Metrics}
For evaluating 2D toy data we use Energy Distance and W2 metrices, for Tabular datasets we use Negative Log Likelihood, for CIFAR10 we took Fréchet inception distance (FID) metrics. This choice is connected with an instability and poor evaluation quality of Energy Distance metrics and $\mathcal W_2$ among high-dimensional data .

\subsubsection{Energy Distance}
We use the generalized Energy Distance~\cite{szekely2003statistics} (or E-metrics) to the metric space.

Consider the null hypothesis that two random variables, $X$ and $Y$, have the same probability distributions: $\mu = \nu$ . 

For statistical samples from $X $and $Y$:
$$ \{x_1, \dots, x_n\} \quad\text{ and }\quad \{y_1, \dots, y_m\},$$
the following arithmetic averages of distances are computed between the $X$ and the $Y$ samples:
\begin{equation*}
A= \frac{1}{nm} \sum_{i=1}^n \sum_{j=1}^m \| x_i - y_j \|,
\quad
B= \frac{1}{n^2} \sum_{i=1}^n \sum_{j=1}^n \| x_i - x_j \|,
\quad
C= \frac{1}{m^2} \sum_{i=1}^m \sum_{j=1}^m \| y_i - y_j\|
.
\end{equation*}
 
The E-statistic of the underlying null hypothesis is defined as follows:

$$E_{n,m}(X, Y) := 2A - B - C$$

\subsubsection{2-Wasserstein distance ($\mathcal W_2$)}

The 2-Wasserstein distance \cite{ramdas2017wasserstein}, also called the Earth mover’s distance or the optimal transport distance $W$ is a metric to describe the distance between two distributions, representing two different subsets $A$ and $B$. 
For continuous distributions, it is:

$$W:=W(F_A,F_B)=\left(\int_0^1\left|F_A^{-1}(u)-F_B^{-1}(u)\right|^2du\right)^{\frac12},$$

where $F_A$ and $F_B$ are the corresponding cumulative distribution functions and $F_A^{-1}$ and $F_B^{-1}$ the respective quantile functions.

\subsubsection{Negative Log Likelihood (NLL)}
To compute the NLL, we first sampled~$N=5000$ samples~$\{x^s_i\}_{i=1}^N$ from the target distribution. 
Then we solved the following inverse flow ODE:
\begin{equation*}
\left\{
\begin{aligned}
\pdv{x(t)}{t} &=  v_\theta(x(t),t), \\
    x(1) &= x_s \\
\end{aligned}
\right.
\end{equation*}
for~$t$ from $1$ to $0$.
For simplicity, changing time variable $\tau=1-t$
we solve the following ODE:
\begin{equation*}
\left\{
\begin{aligned}
\pdv{x(\tau)}{\tau} &=  -v_\theta(x(\tau),1-\tau), \\
    x(0) &= x_s \\
\end{aligned}
\right.
\end{equation*}
for $\tau$ from $0$ to $1$.
Thus we obtained~$N$ solutions~$\{x^0_i\}_{i=1}^N$
which are expected to be distributed according to the standard normal distribution~$\mathcal N(x \mid 0,I)$.
So we calculate NLL as
$$
\hbox{NLL}=-\frac1N\sum_{i=1}^N \ln \mathcal N(x^0_i \mid 0,I).
$$

\subsubsection{Fréchet inception distance (FID)}
For images evaluation we take Fréchet inception distance (FID) metrics, in particular the implementation from \cite{parmar2021cleanfid}. The main idea of FID metrics is to measure the gap between two data distributions, such as between a training set and samples from a trained model. After resizing the images, and feature extraction, the mean $(\mu,\hat{\mu})$ and covariance matrix $(\Sigma,\hat{\Sigma})$ of the corresponding features are used to compute FID:

$\mathrm{FID}=||\mu-\hat{\mu}||_{2}^{2}+\mathrm{Tr}(\Sigma+\widehat{\Sigma}-2(\Sigma\hat{\Sigma})^{1/2}),$

 where $Tr$ is the trace of the matrix.

%%%%%%%%%%%%%%%%%%%%%%%%%%%%%%%%%%%%%%%%%%%%%%%%%%%%%%%%%%%%

\end{document}